\documentclass[12pt]{article}

\usepackage[T1]{fontenc}
\usepackage[utf8]{inputenc}
\usepackage{float}
\usepackage{amsmath}
\usepackage{amssymb}
\usepackage{graphicx}
\usepackage{esint}
\usepackage{cite}
\usepackage{color}
\usepackage{subfig}
\usepackage{amsthm}
\usepackage{subfig}
\usepackage{algorithm}
\usepackage{algpseudocode}
\usepackage{algorithmicx}
\usepackage{hyperref}
\usepackage{mathtools}
\usepackage{authblk}
\usepackage{epstopdf}

\usepackage{geometry}
\usepackage{pdflscape} 
\usepackage{tikz}   
\usetikzlibrary{trees}  

\providecommand{\keywords}[1]
{
  \small	
  \textbf{\textit{Keywords---}} #1
}

\newtheorem{theorem}{Theorem}
\newtheorem{proposition}{Proposition}
\newtheorem{corollary}{Corollary}
\newtheorem{remark}{Remark}
\newtheorem{definition}{Definition}

\begin{document}





\title{{Spatiotemporal Analysis Using Riemannian Composition of Diffusion Operators}}

\author[1]{Tal Shnitzer}
\author[2]{Hau-Tieng Wu}
\author[1]{Ronen Talmon}
\affil[1]{Viterbi Faculty of Electrical and Computer Engineering, Technion - Israel Institute of Technology, Haifa, Israel}
\affil[2]{Department of Mathematics and Department of Statistical Science, Duke University, Durham, NC, USA}
\date{}
\setcounter{Maxaffil}{0}
\renewcommand\Affilfont{\itshape\small}

\maketitle

\begin{abstract}
Multivariate time-series have become abundant in recent years, as many data-acquisition systems record information through multiple sensors simultaneously. 
In this paper, we assume the variables pertain to some geometry and present an operator-based approach for spatiotemporal analysis. Our approach combines three components that are often considered separately: (i) \emph{manifold learning} for building operators representing the geometry of the variables, (ii) \emph{Riemannian geometry of symmetric positive-definite matrices} for multiscale composition of operators corresponding to different time samples, and (iii) \emph{spectral analysis} of the composite operators for extracting different dynamic modes.
We propose a method that is analogous to the classical wavelet analysis, which we term Riemannian multi-resolution analysis (RMRA). 
We provide some theoretical results on the spectral analysis of the composite operators, and we demonstrate the proposed method on simulations and on real data.
\end{abstract}\hspace{10pt}



\section{Introduction}

{
%
Multivariate time-series (temporal signals) have been studied in the statistics and signal processing societies for many years (e.g., see \cite{tsay2013multivariate,gomez2016multivariate} for a non-exhaustive literature survey), and traditional analysis methods usually highly depend on some predefined models and do not consider the unknown nonlinear structure of the variables that often exists underlying the high-dimensional time samples.
In order to accommodate contemporary data acquisitions and collections, large research activity has been devoted to developing spatiotemporal analysis that is specifically-designed to infer this underlying structure and/or take it explicitly into account.
In the last decade, perhaps the most notable attempts to handle such signals with a geometry defined by graphs are graph signal processing \cite{shuman2013emerging,sandryhaila2013discrete,sandryhaila2014discrete,ortega2018graph}, graph neural networks \cite{kipf2016semi,scarselli2008graph}, and geometric deep learning \cite{bronstein2017geometric}.
Another prominent line of work is based on an operator-theoretic approach for dynamical systems analysis \cite{schmid2010dynamic,budivsic2012applied,tu2014dynamic,kutz2016dynamic}, where the time samples have a manifold structure.
Still, despite these recent efforts, to the best of our knowledge, when there is a time-varying manifold structure underlying the time samples, only a few works are available, e.g. \cite{froyland2015dynamic,banisch2017understanding,froyland2020dynamic}.

 
In this work, we propose a new multi-resolution spatiotemporal analysis of multivariate time-series. In contrast to standard multi-resolution analysis using wavelets defined on Euclidean space \cite{daubechies1992ten,mallat1999wavelet}, we present an operator-based analysis approach combining manifold learning and Riemannian geometry, which we term {\em Riemannian multi-resolution analysis} (RMRA).
Concretely, consider a multivariate time-series $\{\mathbf{x}_t\}$. Suppose the temporal propagation of the time-series at time step $t$ can be modelled by two diffeomorphic manifolds $f_t:\mathcal{M}_t \rightarrow \mathcal{M}_{t+1}$, and suppose the corresponding pairs of time samples $(\mathbf{x}_t, \mathbf{x}_{t+1})$ are given by $\mathbf{x}_t[i] \in \mathcal{M}_t$ and $\mathbf{x}_{t+1}[i] = f_t(\mathbf{x}_t[i]) \in \mathcal{M}_{t+1}$, where $\mathbf{x}_t[i]$ is the $i$th entry of the sample $\mathbf{x}_t$ for $i=1,\ldots,N$. 
Note that the entries of the samples $\mathbf{x}_t$ lie on a manifold, and therefore, each entry is typically high-dimensional. In other words, at each time $t$, we have $N$ high-dimensional points that are distributed on the manifold $\mathcal{M}_t$.
Our RMRA consists of the following steps. First, we construct a diffusion operator for each time sample $\mathbf{x}_t$, characterizing its underlying manifold $\mathcal{M}_t$. This step is performed using a manifold learning technique, diffusion maps \cite{Coifman2006}, that facilitates a finite-dimensional matrix approximation of the Laplacian operator of the manifold based on the time sample. This approximation is informative because the Laplacian operator is known to bear the geometric information of the manifold \cite{berard1994embedding,jones2008manifold}. Then, for each pair of temporally consecutive time frames $(\mathbf{x}_t,\mathbf{x}_{t+1})$, we present two composite operators based on ``Riemannian combinations'' of the two respective diffusion operators. 
Typically, diffusion operators are not symmetric, but they are similar to symmetric positive-definite (SPD) matrices. We could thus define diffusion operators as SPD matrices, whose space is endowed with a Riemannian structure.
Therefore, taking into account this Riemannian manifold structure for the composition of the operators is natural. 
Indeed, we show, both theoretically and in practice, that one operator enhances common components that are expressed similarly in $\mathcal{M}_t$ and $\mathcal{M}_{t+1}$, while the other enhances common components that are expressed differently. These properties could be viewed as analogous to low-pass and high-pass filters in this setting, leading to a spatiotemporal decomposition of the multivariate time series into ``low frequency'' and ``high frequency'' components, by considering the common components expressed similarly (resp. differently) as the slowly (resp. rapidly) varying components.

To facilitate the multi-resolution analysis of the entire temporal sequence, the construction of the composite operators is recursively repeated at different time scales. Since the composite operators are viewed as low-pass and high-pass filters, the proposed framework can be viewed as analogous to the wavelet decomposition for time-varying manifolds in the following sense. At each iteration, the two consecutive time samples are ``fused'' using the composite operators, ``decomposing'' the multivariate time-series into two components: one that varies slowly and one that varies rapidly. The fast varying component is viewed as the ``spectral feature'' of the ``first layer'', and the slowly varying component is ``downsampled'', decomposed again in the next iteration using the composite operators into a slow component and a fast component. Again, the fast component leads to the ``spectral feature'' of the ``second layer''. By iterating this procedure, the multivariate time series is decomposed in multiple resolutions.

Broadly, the basic building block of our analysis, focusing on one time step, consists of two construction steps. First, given two consecutive time samples $(\mathbf{x}_t, \mathbf{x}_{t+1})$, we learn the underlying manifolds $\mathcal{M}_t$ and $\mathcal{M}_{t+1}$, and then, we study the (unknown) diffeomorphism $f_t$.
We posit that this building block can serve as an independent analysis module by itself.
Indeed, the setting of one time step we consider can be recast as a related multiview data analysis problem (see Section \ref{sec:related_work}). Consider two diffeomorphic manifolds $\mathcal{M}_1$ and $\mathcal{M}_2$ and the diffeomorphism $f:\mathcal{M}_1 \rightarrow \mathcal{M}_2$. Let $x \in \mathcal{M}_1$ and $y=f(x)\in \mathcal{M}_2$. The pair $(x,y)$ could be considered as two views of some object of interest, providing distinct and complementary information. Applying the proposed two-step procedure to this case first learns the manifold of each view, and then, studies the diffeomorphism representing the relationship between the two views.
In \cite{shnitzer2019recovering}, the diffeomorphism was analyzed in terms of common and unique components, which were represented by the eigenvectors of the diffusion operators.
Here, we further characterize these spectral common components. Roughly, the common components are classified into two kinds: components that are expressed similarly in the two manifolds in the sense that they have similar eigenvalues in both manifolds, and components that are expressed differently in the sense that they have different eigenvalues. Furthermore, we refine the analysis and in addition to considering
{\em strictly common components}, i.e., the same eigenvectors in both manifolds, we also consider {\em weakly common components}, i.e., similar but not the same eigenvectors. 
In contrast to the local analysis presented in \cite{shnitzer2019recovering}, we provide global and spectral analyses, showing that our method indeed extracts and identifies these different components.

%
We demonstrate the proposed RMRA on a dynamical system with a transitory double gyre configuration. We show that this framework is sensitive to the change rate of the dynamical system at different time-scales. Such a framework may be especially suitable for studying non-stationary multivariate time-series, particularly when there is a nontrivial geometric relationship among the multivariate coefficients.
In addition, for the purpose of multimodal data analysis, we demonstrate that the proposed Riemannian composite operators enhance common structures in remote sensing data captured using hyperspectral and LiDAR sensors.

The remainder of this paper is organized as follows. In Section \ref{sec:related_work}, we review related work. In Section \ref{sec:prelim}, we present preliminaries. In Section \ref{sec:rmra}, we present the proposed approach for multi-resolution spatiotemporal analysis using Riemannian composition of diffusion operators. Section \ref{sec:results} shows experimental results. In Section \ref{sec:analysis}, we present spectral analysis of the proposed composite operators. Finally, in Section \ref{sec:conc}, we conclude the paper.
}


%

\section{Related work}
\label{sec:related_work}

\subsection{Manifold learning, diffusion maps, and diffusion wavelets}

{
Manifold learning is a family of methods that consider data lying on some inaccessible manifold and provide a new low-dimensional representation of the data based on intrinsic patterns and similarities in the data \cite{tenenbaum2000global,Roweis2000,Belkin2003,Coifman2006}.
From an algorithmic viewpoint, manifold learning techniques are broadly based on two stages. The first stage is the computation of a typically positive kernel that provides a notion of similarity between the data points. The second stage is the spectral analysis of the kernel, giving rise to an embedding of the data points into a low-dimensional space.
Such a two-stage procedure results in aggregation of multiple pairwise similarities of data points, facilitating the extraction of the underlying manifold structure. This procedure was shown to be especially useful when there are limited high-dimensional data, plausibly circumventing the curse of dimensionality.

While the spectral analysis of the kernels has been the dominant approach and well investigated, recent work explores different directions as well.
One prominent direction employs an operator-based analysis, which has led to the development of several key methods.
Arguably the first and most influential is diffusion maps \cite{Coifman2006}\footnote{Laplacian eigenmaps could also be considered if the diffusion time is not taken into account \cite{Belkin2003}.}, where a transition matrix is constructed based on the kernel, forming a random walk on the dataset; such transition matrix is viewed as a diffusion operator on the data.
There has been abundant theoretical support for diffusion maps. For example, it was shown in \cite{Belkin2003,hein2006uniform,singer2006graph} that the operator associated with diffusion maps converges point-wisely to the Laplace-Beltrami operator of the underlying manifold, which embodies the geometric properties of the manifold, and its eigenfunctions form a natural basis for square integrable functions on the manifold. The spectral convergence of the eigenvalues and eigenvectors of the operator associated with diffusion maps to the eigenvalues and eigenfunctions of the Laplace-Beltrami operator was first explored in \cite{belkin2007convergence}, and recently, the $L^\infty$ spectral convergence with convergence rate was reported in \cite{dunson2019spectral}. See \cite{dunson2019spectral} and references therein for additional related work in this direction. The robustness of the diffusion maps operator was studied in \cite{el2010information,el2016graph}, and recently, its behavior under different noise levels and kernel bandwidths was explored using random matrix theory \cite{ding2021impact}.

The propagation rules associated with this diffusion operator are in turn used for defining a new distance, the so-called diffusion distance, which was shown to be useful and informative in many domains and applications \cite{Lafon2006,li2017efficient,TalmonMagazine,wu2014assess}. 
This notion of diffusion promoted the development of well-designed and controlled anisotropic diffusions for various purposes, e.g., nonlinear independent component analysis \cite{singer2008non}, intrinsic representations \cite{talmon2013empirical}, reduction of stochastic dynamical systems \cite{singer2009detecting,dsilva2016data}, and time-series forecasting \cite{zhao2016analog} and filtering \cite{shnitzer2016manifold}, to name but a few.
In another line of work, the combination and composition of diffusion operators led to the development of new manifold learning techniques for learning multiple manifolds  \cite{lederman2015alternating,shnitzer2019recovering,lindenbaum2020multi} as well as for time-series analysis \cite{froyland2015dynamic}.

A related line of work that considers multivariate time-series (high-dimensional temporal signals) introduces ways to define wavelets on graphs and manifolds, e.g., \cite{coifman2006wavelets,hammond2011wavelets,ram2011generalized}. These techniques extend the classical wavelet analysis \cite{mallat1999wavelet} from one or two dimensional Euclidean space to high-dimensional non-Euclidean spaces represented by graphs and manifolds. Specifically, diffusion wavelets \cite{coifman2006wavelets} makes use of a hierarchy of diffusion operators with multiple well-designed diffusion scales organized in a dyadic tree.
Importantly, none of these methods addresses an underlying manifold with a time-varying metric, but rather a fixed metric that exhibits different characteristics in different scales.
}

\subsection{Manifold learning for sensor fusion}

{
The basic building block of our RMRA is based on two diffeomorphic manifolds $\mathcal{M}_t$ and $\mathcal{M}_{t+1}$, which represent the temporal evolution at time $t$.
A similar setting consisting of two diffeomorphic manifolds, say $\mathcal{M}_1$ and $\mathcal{M}_2$, has recently been investigate in the context of multimodal data analysis and sensor fusion.

The sensor fusion problem typically refers to the problem of harvesting useful information from data collected from multiple, often heterogeneous, sensors. 
Sensor fusion is a gigantic field. One line of work focuses on the extraction, analysis and comparison of the components expressed by the different sensors for the purpose of gaining understanding of the underlying scene \cite{murphy2018diffusion,swatantran2011mapping,czaja2016fusion}. However, due to the complex nature of such data, finding informative representations and metrics of these components by combining the information from the different sensors is challenging.
Recently, several papers propose data fusion methods relying on manifold learning techniques and operator-based data analysis \cite{de2005spectral,eynard2015multimodal,lederman2015alternating,shnitzer2019recovering,talmon2019latent,katz2019alternating,lindenbaum2020multi}.
The basic idea is that data from different modalities or views are fused by constructing kernels that represent the data from each view and operators that combine those kernels.
Different approaches are considered for the combination of kernels. 
Perhaps the most relevant to the present work is the alternating diffusion operator, which was introduced in \cite{lederman2015alternating,talmon2019latent} and shown to recover the common latent variables from multiple modalities.
This operator is defined based on a product of two diffusion operators and then used for extracting a low dimensional representation of the common components shared by the different sensors.
Other related approaches include different combinations of graph Laplacians  \cite{de2005spectral,eynard2015multimodal}, product of kernel density estimators and their transpose for nonparametric extension of canonical correlation analysis \cite{michaeli2016nonparametric}, and various other combinations of diffusion operators \cite{katz2019alternating,shnitzer2019recovering,lindenbaum2020multi}. For a more comprehensive review of the different approaches, see \cite{shnitzer2019recovering} and references therein.

Largely, most existing sensor fusion algorithms, and particularly those based on kernel and manifold learning approaches, focus on the extraction and representation of the common components, in a broad sense. The sensor fusion framework proposed in \cite{shnitzer2019recovering} extends this scope and considers both the common components and the unique components of each sensor. 
Therefore, in the context of the present work, it could be used for the analysis of the basic building block consisting of two diffeomorphic manifolds. 
However, similarly to the other methods described above, the kernel combination in \cite{shnitzer2019recovering} is achieved through linear operations, thereby ignoring the prototypical geometry of the kernels.
Conversely, in this work, by taking the Riemannian structure of SPD matrices into account, we propose a new geometry-driven combination of kernels and a systematic analysis.
This aspect of our work could be viewed as an extension of \cite{shnitzer2019recovering} for the purpose of sensor fusion and multimodal data analysis, in addition to the new utility for multivariate time-series analysis.
}

\section{Preliminaries}\label{sec:prelim}
In this section we briefly present the required background for our method. 
For further details on the theory and derivations we refer the readers to \cite{bhatia2009positive} and \cite{Coifman2006}.

\subsection{Riemannian {Structure} of SPD Matrices\label{sub:bg_spd}}

{
In many recent studies, representing the raw data using SPD matrices and taking into account their specific Riemannian geometry have shown promising results, e.g. in computer vision \cite{bergmann2018priors,tuzel2008pedestrian}, for domain adaptation \cite{yair2019parallel}, on medical data \cite{barachant2013classification} and in recognition tasks \cite{harandi2014manifold}.
For example, Barachant et al. \cite{barachant2013classification} proposed a support-vector-machine (SVM) classifier that takes into account the Riemannian geometry of the features, which are SPD covariance matrices, representing Electroencephalogram (EEG) recordings. They showed that their ``geometry-aware'' classifier obtains significantly better results compared with a classifier that simply vectorizes the covariance matrices.

Here, we consider the space of SPD matrices endowed with the so-called affine-invariant metric \cite{pennec2006riemannian}. Using this particular Riemannian geometry results in closed-form expressions for useful properties and operations, such as the geodesic path connecting two points on the manifold \cite{bhatia2009positive} and the logarithmic map and the exponential map \cite{pennec2006riemannian}, which locally project SPD matrices onto the tangent space and back. 
While the focus is on the affine-invariant metric, which is arguably the most widely used, we remark that other geometries of SPD matrices exist, e.g., the log-Euclidean \cite{arsigny2007geometric,quang2014log}, the log-det \cite{sra2012new,chebbi2012means}, the log-Cholesky \cite{lin2019riemannian}, and the Bures-Wasserstein \cite{malago2018wasserstein,bhatia2019bures}, which could be considered as well.
In the context of this work, since diffusion operators are strictly positive in principal but in practice often have negligible eigenvalues, one particular advantage of the affine-invariant geometry is its existing extensions to symmetric positive semi-definite (SPSD) matrices (see Section \ref{subsec:SPSD}).
}

Consider the set of symmetric matrices in $\mathbb{R}^{N\times N}$, denoted by $\mathcal{S}_N$.
A symmetric matrix $\mathbf{W}\in\mathcal{S}_N$ is an SPD matrix if it has strictly positive eigenvalues.
Let $\mathcal{P}_N$ denote the set of all $N\times N$ SPD matrices.
The tangent space at any point in this set is the space of symmetric matrices $\mathcal{S}_N$. 
We denote the tangent space at $\mathbf{W}\in\mathcal{P}_N$ by $\mathcal{T}_\mathbf{W}\mathcal{P}_N$.
In this work we consider the following affine-invariant metric in the tangent space at each matrix $\mathbf{W}\in\mathcal{P}_N$, which forms a differentiable Riemannian manifold \cite{moakher2005differential}:
\begin{equation}
    \left\langle \mathbf{D}_1,\mathbf{D}_2\right\rangle_\mathbf{W} = \left\langle \mathbf{W}^{-1/2}\mathbf{D}_1\mathbf{W}^{-1/2}, \mathbf{W}^{-1/2}\mathbf{D}_2\mathbf{W}^{-1/2} \right\rangle\label{eq:riemann_spd_metric}
\end{equation}
where $\mathbf{D}_1,\mathbf{D}_2\in\mathcal{T}_\mathbf{W}\mathcal{P}_N$ denote matrices in the tangent space at $\mathbf{W}\in\mathcal{P}_N$ and $\left\langle\cdot,\cdot\right\rangle$ is given by the standard Frobenius inner product $\left\langle \mathbf{A},\mathbf{B}\right\rangle=\mathrm{Tr}\left(\mathbf{A}^T\mathbf{B}\right)$.
Using this metric, there is a unique geodesic path connecting any two matrices $\mathbf{W}_1,\mathbf{W}_2\in\mathcal{P}_N$ \cite{bhatia2009positive}, which is explicitly given by:
\begin{equation}
\gamma_{\mathbf{W}_1\rightarrow\mathbf{W}_2}(p)=\mathbf{W}_1^{1/2}\left(\mathbf{W}_1^{-1/2}\mathbf{W}_2\mathbf{W}_1^{-1/2}\right)^p\mathbf{W}_1^{1/2}, \ \ p\in[0,1],\label{eq:riemann_spd_geodesic}
\end{equation}
{The arc-length of this geodesic path defines the Riemannian distance on the manifold, 
\begin{equation}
d^2_R\left(\mathbf{W}_1,\mathbf{W}_2\right)=\left\Vert \log\left(\mathbf{W}_1^{-1/2}\mathbf{W}_2\mathbf{W}_1^{-1/2}\right)\right\Vert^2_F.  
\end{equation}
Using the Fr\'echet mean, we define the Riemannian mean of a set of matrices, {$\mathbf{W}_1,\ldots,\mathbf{W}_n$,} by $\arg\min_{\mathbf{W}\in\mathcal{P}_N}{\sum_{i=1}^n} d^2_R\left(\mathbf{W},\mathbf{W}_i\right)$. The Riemannian mean of two matrices is a special case, which coincides with the mid-point of the geodesic path connecting them and has the following closed form: }
\begin{equation}
    \gamma_{\mathbf{W}_1\rightarrow\mathbf{W}_2}(1/2)=\mathbf{W}_1^{1/2}\left(\mathbf{W}_1^{-1/2}\mathbf{W}_2\mathbf{W}_1^{-1/2}\right)^{1/2}\mathbf{W}_1^{1/2}
\end{equation}

The mapping between the Riemannian manifold of SPD matrices and its tangent space is given by the exponential map and the logarithmic map.
Each matrix $\mathbf{D}$ in the tangent space at $\mathbf{W}\in\mathcal{P}_N$ 
can be seen as the derivative of the geodesic connecting $\mathbf{W}$ and $\tilde{\mathbf{W}}=\mathrm{Exp}_\mathbf{W}(\mathbf{D})$, i.e., $\gamma_{\mathbf{W}\rightarrow\tilde{\mathbf{W}}}(p)$, at $p=0$.
The exponential map in this setting has a known closed-form given by \cite{moakher2005differential}:
\begin{equation}
    \mathrm{Exp}_{\mathbf{W}}\left(\mathbf{D}\right)=\mathbf{W}^{1/2}\exp\left(\mathbf{W}^{-1/2}\mathbf{D}\mathbf{W}^{-1/2}\right)\mathbf{W}^{1/2}\label{eq:spd_expmap},
\end{equation}
where $\mathrm{Exp}_{\mathbf{W}}(\mathbf{D})\in\mathcal{P}_N$ and $\exp(\cdot)$ is applied to the eigenvalues.
The inverse of the exponential map is the logarithmic map, which is explicitly given by:
\begin{equation}
    \mathrm{Log}_{\mathbf{W}}(\tilde{\mathbf{W}})=\mathbf{W}^{1/2}\log\left(\mathbf{W}^{-1/2}\tilde{\mathbf{W}}\mathbf{W}^{-1/2}\right)\mathbf{W}^{1/2},\label{eq:spd_logmap}
\end{equation}
where $\log(\cdot)$ is applied to the eigenvalues, $\tilde{\mathbf{W}}\in\mathcal{P}_N$, and $\mathrm{Log}_{\mathbf{W}}(\tilde{\mathbf{W}})\in\mathcal{T}_\mathbf{W}\mathcal{P}_N$.

\subsection{An extension to SPSD Matrices\label{subsec:SPSD}}

In practice, the matrices of interest are often not strictly positive, but rather symmetric positive \emph{semi}-definite (SPSD) matrices, that is symmetric matrices with non-negative eigenvalues.
Below is a summary of a Riemannian geometry introduced in \cite{bonnabel2010riemannian} that extends the affine-invariant metric. We remark that the existence of such an extension serves as an additional motivation to particularly consider the affine-invariant metric over the alternatives.

Let $\mathcal{S}^+(r,N)$ denote the set of $N\times N$ SPSD matrices of rank $r<N$, and let $\mathcal{V}_{N,r}$ denote the set of $N\times r$ matrices with orthonormal columns, i.e. $\mathbf{V}^T\mathbf{V}=\mathrm{I}_{r\times r}, \ \forall\mathbf{V}\in\mathcal{V}_{N,r}$.
Given an SPSD matrix $\mathbf{W}\in\mathcal{S}^+(r,N)$, we consider the following decomposition:
\begin{equation}
    \mathbf{W} = \mathbf{V}\mathbf{\Lambda}\mathbf{V}^T\,,
\end{equation}
where $\mathbf{V}\in\mathcal{V}_{N,r}$ and $\mathbf{\Lambda}\in\mathcal{P}_r$.
Consequently, the pair $\left(\mathbf{V},\mathbf{\Lambda}\right)\in\mathcal{V}_{N,r}\times\mathcal{P}_r$ can be considered as a representation of the matrix $\mathbf{W}$.
Based on this representation, the authors of \cite{bonnabel2010riemannian} proposed to present vectors in the tangent space of $\mathcal{S}^+\left(r,N\right)$ at $\mathbf{W}$, denoted by $\mathcal{T}_\mathbf{W}\mathcal{S}^+\left(r,N\right)$, by the infinitesimal variation $\left(\mathbf{\Delta},\mathbf{D}\right)$, where $\mathbf{\Delta}=\mathbf{V}_\perp\mathbf{B}$, $\mathbf{V}_\perp\in\mathcal{V}_{N,N-r}$, $\mathbf{V}_\perp\mathbf{V}=0$, $\mathbf{B}\in\mathbb{R}^{(N-r)\times r}$ and $\mathbf{D}\in\mathcal{T}_{\mathbf{\Lambda}}\mathcal{P}_r=\mathcal{S}_r$.
Using these tangent vectors, the metric on the tangent space of the SPSD manifold can be defined as a generalization of the metric on the manifold of SPD matrices:
\begin{equation}
    \left\langle\left(\mathbf{\Delta}_1,\mathbf{D}_1\right),\left(\mathbf{\Delta}_2,\mathbf{D}_2\right)\right\rangle_{\left(\mathbf{V},\mathbf{\Lambda}\right)} = \mathrm{Tr}\left(\mathbf{\Delta}^T_1\mathbf{\Delta}_2\right) + k\left\langle\mathbf{D}_1,\mathbf{D}_2\right\rangle_\mathbf{\Lambda}\label{eq:SPSDmetric}
\end{equation}
for some $k>0$, where $\left(\mathbf{\Delta}_1,\mathbf{D}_1\right),\left(\mathbf{\Delta}_2,\mathbf{D}_2\right)\in\mathcal{T}_\mathbf{W}\mathcal{S}^+\left(r,N\right)$ and $\left\langle\cdot,\cdot\right\rangle_\mathbf{\Lambda}$ denotes the metric defined in \eqref{eq:riemann_spd_metric} for SPD matrices.

Let $\mathbf{W}_1=\mathbf{V}_1\mathbf{\Lambda}_1\mathbf{V}_1^T$ and $\mathbf{W}_2=\mathbf{V}_2\mathbf{\Lambda}_2\mathbf{V}_2^T$ denote the decompositions of two SPSD matrices $\mathbf{W}_1,\mathbf{W}_2\in\mathcal{S}^+\left(r,N\right)$, where $\mathbf{V}_1,\mathbf{V}_2\in\mathcal{V}_{N,r}$ and $\mathbf{\Lambda}_1,\mathbf{\Lambda}_2\in\mathcal{P}_r$.
The closed-form expression of the geodesic path connecting any two such matrices in $\mathcal{S}^+\left(r,N\right)$ using the metric in \eqref{eq:SPSDmetric} is unknown. 
However, the following approximation of it was proposed in \cite{bonnabel2010riemannian}.
Denote the singular value decomposition (SVD) of $\mathbf{V}_2^T\mathbf{V}_1$ by $\mathbf{O}_2\mathbf{\Sigma}\mathbf{O}_1^T$, where $\mathbf{O}_1,\mathbf{O}_2\in\mathbb{R}^{r\times r}$ and $\mathrm{diag}(\mathbf{\Sigma})$ are the cosines of the principal angles between $\mathrm{range}(\mathbf{W}_1)$ and $\mathrm{range}(\mathbf{W}_2)$, where $\mathrm{range}(\mathbf{W})$ denotes the column space of $\mathbf{W}$.
Define $\mathbf{\Theta}=\arccos\left(\mathbf{\Sigma}\right)$, which is a diagonal matrix of size $r\times r$ with the principal angles between the two subspaces on its diagonal.
The approximation of the geodesic path connecting two points in $\mathcal{S}^+\left(r,N\right)$ is then given by:
\begin{equation}
    \tilde{\gamma}_{\mathbf{W}_1\rightarrow\mathbf{W}_2}(p)=\mathbf{U}_{\mathbf{W}_1\rightarrow\mathbf{W}_2}(p)\mathbf{R}_{\mathbf{W}_1\rightarrow\mathbf{W}_2}(p)\mathbf{U}^T_{\mathbf{W}_1\rightarrow\mathbf{W}_2}(p), \ \ p\in[0,1]\label{eq:riemann_spsd_geodesic}
\end{equation}
where $\mathbf{R}_{\mathbf{W}_1\rightarrow\mathbf{W}_2}(p)$ is the geodesic path connecting SPD matrices as defined in \eqref{eq:riemann_spd_geodesic} calculated between the matrices $\mathbf{R}_1=\mathbf{O}_1^T\mathbf{\Lambda}_1\mathbf{O}_1$ and $\mathbf{R}_2=\mathbf{O}_2^T\mathbf{\Lambda}_2\mathbf{O}_2$, i.e. $\mathbf{R}_{\mathbf{W}_1\rightarrow\mathbf{W}_2}(p)=\gamma_{\mathbf{R}_1\rightarrow\mathbf{R}_2}(p)$, and $\mathbf{U}_{\mathbf{W}_1\rightarrow\mathbf{W}_2}(p)$ is the geodesic connecting $\mathrm{range}(\mathbf{W}_1)$ and $\mathrm{range}(\mathbf{W}_2)$ on the Grassman manifold (the set of $r$ dimensional subspaces of $\mathbb{R}^N$), defined by:
\begin{equation}
    \mathbf{U}_{\mathbf{W}_1\rightarrow\mathbf{W}_2}(p) = \mathbf{U}_1\cos\left(\mathbf{\Theta}p\right)+\mathbf{X}\sin\left(\mathbf{\Theta}p\right)\label{eq:grassman_geodesic}
\end{equation}
where $\mathbf{U}_1=\mathbf{V}_1\mathbf{O}_1$, $\mathbf{U}_2=\mathbf{V}_2\mathbf{O}_2$ and $\mathbf{X}=\left(\mathrm{I}-\mathbf{U}_1\mathbf{U}_1^T\right)\mathbf{U}_2\left(\sin\left(\mathbf{\Theta}\right)\right)^{\dagger}$, with $(\cdot)^\dagger$ denoting the Moore-Penrose pseudo-inverse.


\subsection{Diffusion Operator\label{sub:meas_rep}}
{As described in Section \ref{sec:related_work}, most manifold learning methods, and particularly diffusion maps \cite{Coifman2006}, are based on positive kernel matrices. Here, we briefly present the construction of such a kernel, which we term the diffusion operator, as proposed in \cite{Coifman2006}. In the sequel, we employ this diffusion operator in our framework {to recover the geometry} underlying each modality.}

Given a set of $N$ points, $\{\mathbf{x}[i]\}_{i=1}^N$, which are sampled from some hidden manifold $\mathcal{M}$ embedded in $\mathbb{R}^n$, consider the following affinity kernel matrix $\mathbf{K}\in\mathbb{R}^{N\times N}$, whose $(i,j)$th entry is given by:
\begin{equation}
\mathrm{K}[i,j] = \exp\left(-\frac{\left\Vert\mathbf{x}[i]-\mathbf{x}[j]\right\Vert_2^2}{\sigma^2}\right),\label{eq:SPDkernK}
\end{equation}
where $\left\Vert\cdot\right\Vert_2$ denotes the $\ell_2$ norm and $\sigma$ denotes the kernel scale, typically set to the median of the Euclidean distances between the sample points multiplied by some scalar. {By Bochner's theorem, $\mathbf{K}$ is an SPD matrix.}
The kernel is normalized twice according to:
\begin{eqnarray}
\widehat{\mathbf{W}} & = & \widehat{\mathbf{D}}^{-1}\ \mathbf{K}\ \widehat{\mathbf{D}}^{-1}\nonumber\\
\mathbf{W} & = & \mathbf{D}^{-1/2}\ \widehat{\mathbf{W}}\ \mathbf{D}^{-1/2},\label{eq:SPDKern}
\end{eqnarray}
where $\widehat{\mathbf{D}}$ and $\mathbf{D}$ are diagonal matrices with $\widehat{\mathrm{D}}[i,i]=\sum_{j=1}^N \mathrm{K}[i,j]$ and $\mathrm{D}[i,i]=\sum_{j=1}^N \widehat{\mathrm{W}}[i,j]$, respectively.

The matrix $\mathbf{W}$ defined in \eqref{eq:SPDKern} is similar to the diffusion {operator considered in} \cite{Coifman2006}{, which we call the diffusion maps operator for simplicity,} with a normalization that removes the point density influence, given by $\mathbf{W}_{\texttt{DM}}=\mathbf{D}^{-1}\widehat{\mathbf{W}}$.
Due to this similarity, the matrix $\mathbf{W}$ and the diffusion maps operator $\mathbf{W}_{\texttt{DM}}$ share the same eigenvalues and their eigenvectors are related by $\psi_{\texttt{DM}}=\mathbf{D}^{-1/2}\psi$ and $\tilde{\psi}_{\texttt{DM}}=\mathbf{D}^{1/2}\psi$, where $\psi$ denotes an eigenvector of $\mathbf{W}$ and $\psi_{DM}$ and $\tilde{\psi}_{\texttt{DM}}$ denote the right and left eigenvectors of $\mathbf{W}_{\texttt{DM}}$, respectively.




\section{Riemannian Multi-resolution Analysis}\label{sec:rmra}

We are ready to present our multi-resolution framework for multivariate time-series analysis from a manifold learning perspective. In this section, we focus on the algorithmic aspect, and in Section \ref{sec:analysis} we present the theoretical justification. We start by introducing the Riemannian composition of two operators that capture the relationship between two datasets sampled from two underlying diffeomorphic manifolds. Then, we generalize the setting to a sequence of datasets in time by presenting a wavelet-like analysis using the composite operators. Finally, we conclude this section with important implementation remarks.

\subsection{Riemannian composition of two operators\label{sub:op_def}}

{Consider two datasets of $N$ points denoted by $\{\mathbf{x}_1[i]\}_{i=1}^N,\{\mathbf{x}_2[i]\}_{i=1}^N$. Suppose there is some correspondence between the datasets and that they are ordered according to this correspondence, i.e., the two points $\mathbf{x}_1[i]$ and $\mathbf{x}_2[i]$ correspond. Such a correspondence could be the result of simultaneous recording from two, possibly different, sensors.
We aim to recover the common structures in these datasets and characterize their expression in each dataset. Specifically, we consider two types of common components: common components that are expressed similarly in the two datasets and common components that are expressed differently.}

To this end, we propose a two-step method.
First, we assume that each dataset lies on some manifold and we characterize its underlying geometry using a diffusion operator constructed according to \eqref{eq:SPDKern}. 
This results in two SPD matrices denoted by $\mathbf{W}_1$ and $\mathbf{W}_2$ (see Section \ref{sub:bg_spd}).
Then, we propose to ``fuse'' the two datasets by considering compositions of $\mathbf{W}_1$ and $\mathbf{W}_2$ based on Riemannian geometry.
In contrast to previous studies that consider linear combinations involving addition, subtraction, and multiplication, e.g., \cite{lederman2015alternating,shnitzer2019recovering,lindenbaum2020multi}, which often results in non-symmetric or non-positive matrices violating the fundamental geometric structure of diffusion operators, our Riemannian compositions yield symmetric and SPD matrices. 
{{Specifically, we define} two new operators by:
\begin{eqnarray}
\mathbf{S}_{p} & = & \mathbf{W}_1\#_{p}\mathbf{W}_2=\mathbf{W}_1^{1/2}\left(\mathbf{W}_1^{-1/2}\mathbf{W}_2\mathbf{W}_1^{-1/2}\right)^{p}\mathbf{W}_1^{1/2},\label{eq:Sr}\\
\mathbf{F}_{p} & = & \mathrm{Log}_{\mathbf{S}_{p}}\left(\mathbf{W}_1\right)=\mathbf{S}_{p}^{1/2}\log\left(\mathbf{S}_{p}^{-1/2}\mathbf{W}_1\mathbf{S}_{p}^{-1/2}\right)\mathbf{S}_{p}^{1/2},\label{eq:Ar}
\end{eqnarray}
where $0 \le p \le 1$ denotes the position along the geodesic path connecting $\mathbf{W}_1$ and $\mathbf{W}_2$ on the SPD manifold, $\mathbf{W}_1\#_{p}\mathbf{W}_2$ with $p=1/2$ denotes the midpoint on this geodesic, and $\mathrm{Log}_{\mathbf{S}_{p}}\left(\mathbf{W}_1\right)$ denotes the logarithmic map, projecting the matrix $\mathbf{W}_1$ onto the tangent space of the SPD manifold at point $\mathbf{S}_{p}=\mathbf{W}_1\#_{p}\mathbf{W}_2$. Figure \ref{fig:SA_vis} presents an illustration of the definitions of the operators on the Riemannian manifold of SPD matrices and its tangent space.

Intuitively, {$\mathbf{S}$} describes the mean of the two matrices, so it enhances the components that are expressed similarly; that is, common eigenvectors with similar eigenvalues. Conversely, $\mathbf{F}_{p}$ can be seen as the difference of $\mathbf{S}_{p}$ and $\mathbf{W}_1$ along the geodesic connecting {them}, and therefore, it is related to the components expressed differently; that is, the common eigenvectors with different eigenvalues. In Section \ref{sec:analysis} we {provide a theoretical justification for the above statements}.

We remark that $\mathbf{F}_{p}$ is symmetric but not positive-definite, since it is defined as a projection of an SPD matrix onto the tangent space, and that using $\mathbf{W}_2$ instead of $\mathbf{W}_1$ in the projection leads only to a change of sign.
In addition, given $\mathbf{S}_p$ and $\mathbf{F}_p$, both SPD matrices $\mathbf{W}_1$ and $\mathbf{W}_2$ can be reconstructed using the exponential map $\mathrm{Exp}_{\mathbf{S}_p}\left(\pm\mathbf{F}_p\right)$ (as defined in \eqref{eq:spd_expmap}).}
For simplicity of notations, we focus in the following on $\mathbf{F}_p$ and $\mathbf{S}_p$ with $p=0.5$ and omit the notation of $p$. The extension to other values of $p$ is straightforward.

\begin{figure}[t]
\centering
\subfloat[]{\includegraphics[trim=30 10 30 20,clip,width=0.5\textwidth]{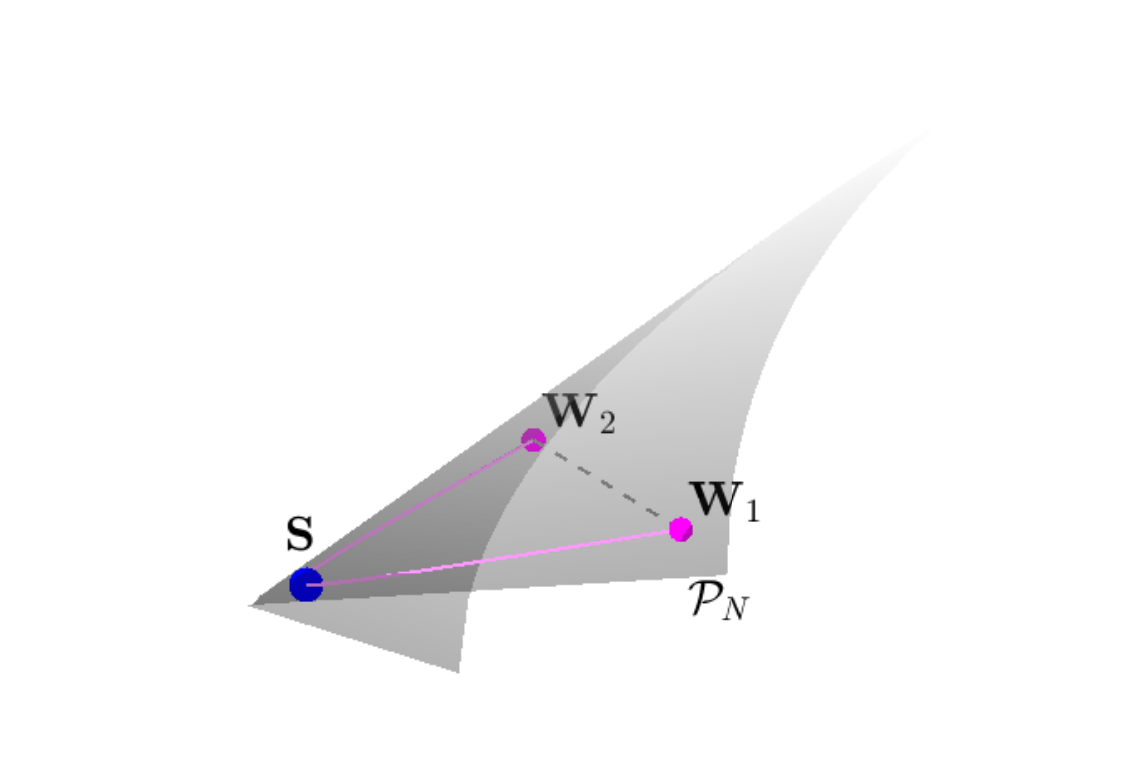}}
\subfloat[]{\includegraphics[trim=30 60 30 70,clip,width=0.5\textwidth]{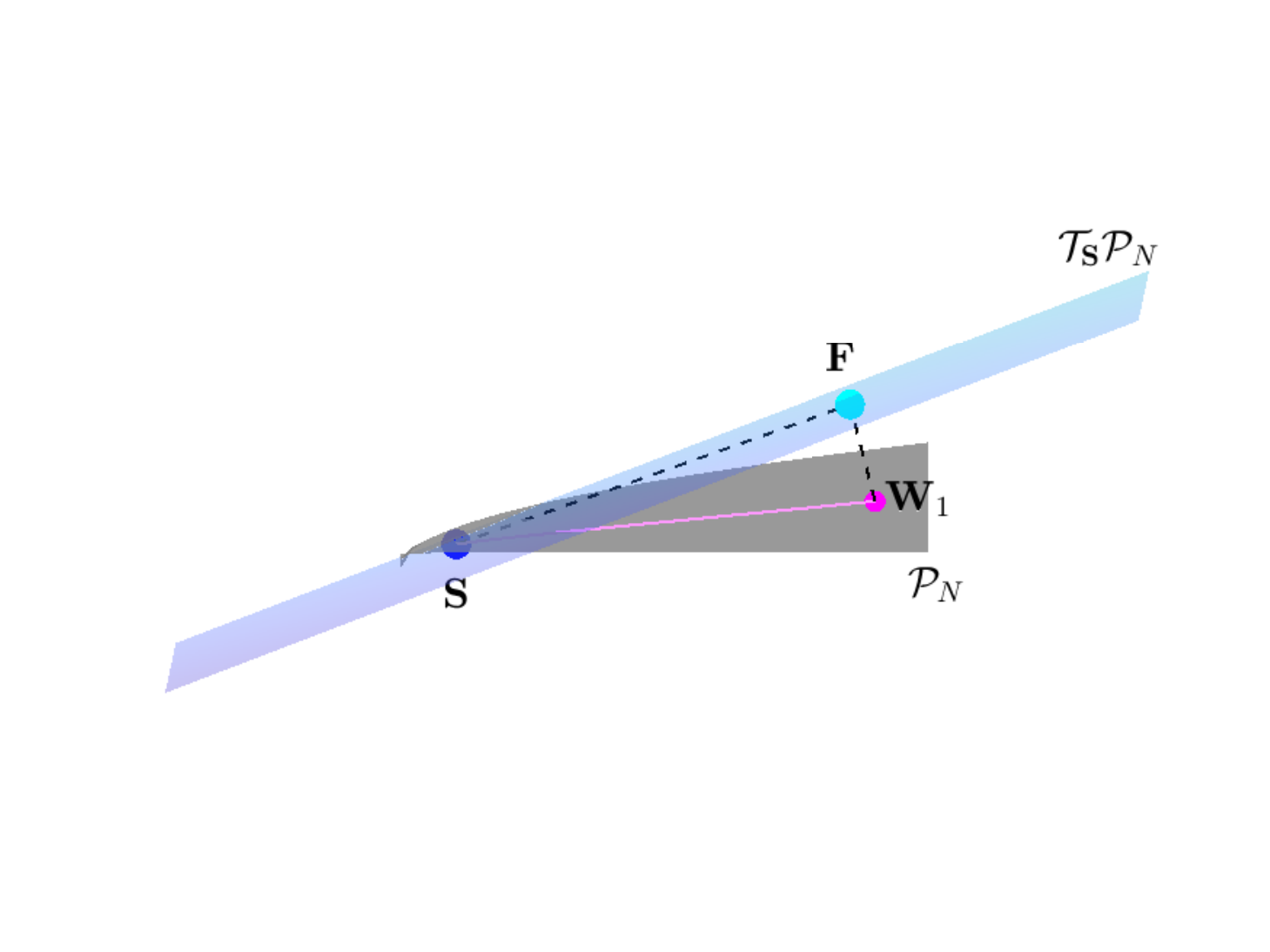}}
\caption{Illustration of the definitions of the operators $\mathbf{S}$ and $\mathbf{F}$. (a) Illustration of the operator $\mathbf{S}$ (blue point) on the geodesic path (solid line in magenta) connecting $\mathbf{W}_1$ and $\mathbf{W}_2$ (magenta points) on the manifold of SPD matrices (gray surface represents one level-set of the Riemannian manifold of SPD matrices). The dashed gray line denotes the shortest Euclidean path connecting the two matrices. (b) Illustration of the operator $\mathbf{F}$ (cyan point) on the tangent space at $\mathbf{S}$ (colored plane). Both plots present the same region in different orientations. 
\label{fig:SA_vis}}
\end{figure}


{In the second step, we} propose new embeddings of the data points, representing the common components between the two datasets based on the operators $\mathbf{S}$ and $\mathbf{F}$.
Since both operators are symmetric, their eigenvalues and eigenvectors are real, and the eigenvectors are orthogonal. 
Denote the eigenvalues and eigenvectors of the operator $\mathbf{S}$ by $\lambda_n^{(\mathbf{S})}$ and $\psi_n^{(\mathbf{S})}$, respectively, and the eigenvalues and eigenvectors of the operator $\mathbf{F}$ by $\lambda_n^{(\mathbf{F})}$ and $\psi_n^{(\mathbf{F})}$, respectively, where $n=1,\dots,N$.
The new embeddings are constructed based on the eigenvectors of $\mathbf{S}$ and $\mathbf{F}$ by taking the $M\leq N$ leading eigenvectors, i.e. eigenvectors that correspond to the $M$ largest eigenvalues (in absolute value for $\mathbf{F}$), which are organized in decreasing order $\lambda^{(\mathbf{S})}_1\geq\lambda^{(\mathbf{S})}_2\geq\dots\geq\lambda^{(\mathbf{S})}_M$ and $\lambda^{(\mathbf{F})}_1\geq\lambda^{(\mathbf{F})}_2\geq\dots\geq\lambda^{(\mathbf{F})}_M$. 
The new embeddings are defined by:
\begin{eqnarray}
(\mathbf{x}_1[i],\mathbf{x}_2[i]) & \mapsto & \Psi^{(\mathbf{S})}[i]=\left\lbrace\psi^{(\mathbf{S})}_1[i],\dots,\psi^{(\mathbf{S})}_M[i]\right\rbrace\\
(\mathbf{x}_1[i],\mathbf{x}_2[i]) & \mapsto & \Psi^{(\mathbf{F})}[i]=\left\lbrace\psi^{(\mathbf{F})}_1[i],\dots,\psi^{(\mathbf{F})}_M[i]\right\rbrace.
\end{eqnarray}

{Algorithm \ref{alg:twomod} summarizes the above two-step operator and embedding construction.}
In Section \ref{subsub:ills_exmp}, we demonstrate the properties of the operators $\mathbf{S}$ and $\mathbf{F}$ and the proposed embeddings on an illustrative toy example. In Section \ref{sec:analysis}, we present some analysis.

{As a final remark, we note that} other SPD kernels and matrices {may be considered} instead of the proposed diffusion operators, e.g. covariance or correlation matrices, which can simply substitute $\mathbf{W}_1$ and $\mathbf{W}_2$ in the above definitions.

\begin{algorithm}[hbt!]
\caption{{Operator composition and spectral embedding based on Riemannian geometry}}
\label{alg:twomod}
\textbf{Input:} Two datasets $\{\mathbf{x}_1[i]\}_{i=1}^N,\{\mathbf{x}_2[i]\}_{i=1}^N$; {the embedding dimension $M$}\\
\textbf{Output:} Operators $\mathbf{S}$ and $\mathbf{F}$ and new embeddings
$\mathbf{\Psi}^{(\mathbf{S})}$ and $\mathbf{\Psi}^{(\mathbf{F})}$
\begin{algorithmic}[1]
\Statex
\State{Construct a diffusion operator for each dataset, $\mathbf{W}_\ell\in\mathbb{R}^{N\times N}$, $\ell=1,2$, according to \eqref{eq:SPDkernK} and \eqref{eq:SPDKern}}
\Statex
\State{Build operators $\mathbf{S}$ and $\mathbf{F}$:}
\State{$\ \ \ \ \ \mathbf{S}=\mathbf{W}_1^{1/2}\left(\mathbf{W}_1^{-1/2}\mathbf{W}_2\mathbf{W}_1^{-1/2}\right)^{1/2}\mathbf{W}_1^{1/2}$}
\State{$\ \ \ \ \ \mathbf{F}=\mathbf{S}^{1/2}\log\left(\mathbf{S}^{-1/2}\mathbf{W}_1\mathbf{S}^{-1/2}\right)\mathbf{S}^{1/2}$}
\Statex
\State{Compute the eigenvalue decomposition of the operators $\mathbf{S}$ and $\mathbf{F}$}
\Statex
\State{Take the $M$ largest eigenvalues (in absolute value) and order them such that $\lambda^{(\mathbf{S})}_1\geq\lambda^{(\mathbf{S})}_2\geq\dots\geq\lambda^{(\mathbf{S})}_M$ and $\lambda^{(\mathbf{F})}_1\geq\lambda^{(\mathbf{F})}_2\geq\dots\geq\lambda^{(\mathbf{F})}_M$}
\Statex
\State{Take the corresponding $M$ eigenvectors of $\mathbf{S}$ and define:\\
$\mathbf{\Psi}^{(\mathbf{S})}=\left\lbrace\psi^{(\mathbf{S})}_1,\dots,\psi^{(\mathbf{S})}_M\right\rbrace\in\mathbb{R}^{N\times M}$}
\Comment{{Capture} similarly expressed common components}
\Statex
\State{Take the corresponding $M$ eigenvectors of $\mathbf{F}$ and define:\\
$\mathbf{\Psi}^{(\mathbf{F})}=\left\lbrace\psi^{(\mathbf{F})}_1,\dots,\psi^{(\mathbf{F})}_M\right\rbrace\in\mathbb{R}^{N\times M}$}
\Comment{{Capture} differently expressed common components}
\end{algorithmic}
\end{algorithm}



\subsection{Operator-based analysis of a sequence of datasets\label{sub:rmra}}

{Let $\{\mathrm{x}_t[i]\}_{i=1}^{N}$ denote a temporal sequence of datasets, where $t=1,\dots,T=2^m$, $m\in\mathbb{N}$, denotes time, and $\mathrm{x}_t[i]\in\mathbb{R}^d$ is the $i$-th point sampled at time $t$.
Considering first just two consecutive datasets $\{\mathrm{x}_t[i]\}_{i=1}^{N}$ and $\{\mathrm{x}_{t+1}[i]\}_{i=1}^{N}$ is analogous to the setting presented in Section \ref{sub:op_def}. Applying the same analysis gives rise to the operators $\mathbf{S}$ and $\mathbf{F}$ corresponding to 
$\{\mathrm{x}_t[i]\}_{i=1}^{N}$ and $\{\mathrm{x}_{t+1}[i]\}_{i=1}^{N}$, which facilitate the extraction of the two types of underlying common components. Unlike the general setting in Section \ref{sub:op_def}, the temporal order of the two datasets considered here allows us to view the common components that are expressed similarly and extracted by $\mathbf{S}$ as the slowly changing components. Similarly, the common components that are expressed differently and extracted by $\mathbf{F}$ are considered as rapidly changing components.}

%
%

The above description constitutes the basic building block of our analysis. 
With that in mind, we proceed to the construction of the proposed multi-resolution analysis of the entire sequence.
At the first step, we build a diffusion operator according to \eqref{eq:SPDKern} for the dataset $\{\mathrm{x}_t[i]\}_{i=1}^{N}$ at each time point $t$, resulting in $T$ kernels $\mathbf{W}_t\in\mathbb{R}^{N\times N}$, $t=1,\dots,T$.
Then, for every pair of consecutive time-points, $(2t-1,2t)$, $t=1,\dots,T/2$, we construct the two operators\footnote{Note that the operator underscore notation now denotes the time index rather than the geodesic curve parameter $p$ as in Section \ref{sub:op_def}.} 
$\mathbf{S}_t^{(1)}$ and $\mathbf{F}_t^{(1)}$ according to \eqref{eq:Sr} and \eqref{eq:Ar} with $p=0.5$. 
These $2\times T/2$ operators represent the fine level, denoted $\ell=1$, of the multi-resolution framework and recover components which are common to consecutive time-frames.
At coarser (higher) levels, i.e. for $\ell>1$, the operators are constructed according to \eqref{eq:Sr} and \eqref{eq:Ar} (with $p=0.5$) using the operators from the previous level as input. Specifically, at level $\ell>1$, the operators are given by
\begin{eqnarray}
\mathbf{S}^{(\ell)}_t&=&\mathbf{S}_{2t-1}^{(\ell-1)}\#\mathbf{S}_{2t}^{(\ell-1)}\label{eq:Swav}\\
\mathbf{F}_t^{(\ell)}&=&\mathrm{Log}_{\mathbf{S}_t^{(\ell)}}\left(\mathbf{S}_{2t-1}^{(\ell-1)}\right),\label{eq:Awav}
\end{eqnarray}
where $t=1,\dots,T/2^\ell$ and $\ell=2,\dots,\log_2T$.
At each level, only the operator $\mathbf{S}$ is used to construct the operators of the next level. The reason for this construction choice is that $\mathbf{S}$ enhances similarly expressed common components. 
In the present setting, the similarly expressed common components of consecutive time frames are in fact components that change slowly in time.
Therefore, using the operators $\mathbf{S}^{(\ell-1)}_t$, $t=1,\dots,T/2^{\ell-1}$, to construct the operators of the coarser level, $\ell$, has a smoothing effect.
This is analogous to the construction of the ordinary wavelet decomposition, where the outputs of the low-pass filters at each level are used as inputs to the coarser level.
In this analogy, the operator $\mathbf{F}^{(\ell)}_t$ can be viewed as a high-pass filter, since it enhances components that are expressed significantly different in consecutive time frames, i.e. rapidly changing components.

Similarly to the embedding defined based on $\mathbf{S}$ and $\mathbf{F}$ in Subsection \ref{sub:op_def}, we define an embedding based on the eigenvectors of the operators $\mathbf{S}^{(\ell)}_t$ and $\mathbf{F}^{(\ell)}_t$ at different levels and time-frames.
Denote the eigenvalues and eigenvectors of the operator $\mathbf{S}^{(\ell)}_t$ by $\lambda_n^{(\mathbf{S}^{(\ell)}_t)}$ and $\psi_n^{(\mathbf{S}^{(\ell)}_t)}$, respectively, and the eigenvalues and eigenvectors of the operator $\mathbf{F}^{(\ell)}_t$ by $\lambda_n^{(\mathbf{F}^{(\ell)}_t)}$ and $\psi_n^{(\mathbf{F}^{(\ell)}_t)}$, respectively, where $n=1,\dots,N$ and the eigenvalues are ordered in decreasing magnitude.
The embedding at each level and each time frame is then defined by taking the $M\leq N$ leading eigenvectors of the operators as follows:
\begin{eqnarray}
\mathbf{x}_t[i] & \rightarrow & \Psi^{(\mathbf{S}^{(\ell)}_t)}=\left\lbrace\psi^{(\mathbf{S}^{(\ell)}_t)}_1[i],\dots,\psi^{(\mathbf{S}^{(\ell)}_t)}_M[i]\right\rbrace\\
\mathbf{x}_t[i] & \rightarrow & \Psi^{(\mathbf{F}^{(\ell)}_t)}=\left\lbrace\psi^{(\mathbf{F}^{(\ell)}_t)}_1[i],\dots,\psi^{(\mathbf{F}^{(\ell)}_t)}_M[i]\right\rbrace.
\end{eqnarray}
where $t=1,\dots,T/2^\ell$ and $\ell=2,\dots,\log_2T$.
These embedding coordinates capture the slowly varying components ($\Psi^{(\mathbf{S}^{(\ell)}_t)}$) and fast varying components ($\Psi^{(\mathbf{F}^{(\ell)}_t)}$) at each time frame $t$ and each level $\ell$.

The proposed algorithm is summarized in Algorithm \ref{alg:wavelet}.

\begin{algorithm}[hbt!]
\caption{Riemannian multi-resolution analysis algorithm}
\label{alg:wavelet}
\textbf{Input:} A time-varying dataset $\{\mathrm{x}_t[i]\}_{i=1}^N$, $\mathrm{x}_t[i]\in\mathbb{R}^d$, $t=1,\dots,T$.\\
\textbf{Output:} Operators $\{\mathbf{S}_t^{(\ell)},\mathbf{F}_t^{(\ell)}\}_{t=1}^{T/2^\ell}$ and new representations for each level\\
$\{\mathbf{\Psi}^{(\mathbf{S}_t^{(\ell)})},\mathbf{\Psi}^{(\mathbf{F}_t^{(\ell)})}\}_{t=1}^{T/2^\ell}$, where $\ell=1,\dots,\log_2T$
\begin{algorithmic}[1]
\Statex
\State{Construct an SPD kernel representing each time point $t$, denoted by $\{\mathbf{W}_t\}_{t=1}^T$, according to \eqref{eq:SPDKern}.}
\Statex
\For{$t=1:T/2$} \Comment{Construct the operators for level $1$}
	\State{$\mathbf{S}_t^{(1)}=\mathbf{W}_{2t-1}\#_{0.5}\mathbf{W}_{2t}$}
	\State{$\mathbf{F}_t^{(1)}=\mathrm{Log}_{\mathbf{S}_t^{(1)}}\left(\mathbf{W}_{2t-1}\right)$}
\EndFor
\For{$\ell=2:\log_2T$} \Comment{Construct the operators for level $\ell$}
\For{$t=1:T/2^\ell$}
	\State{$\mathbf{S}_t^{(\ell)}=\mathbf{S}_{2t-1}^{(\ell-1)}\#_{0.5}\mathbf{S}_{2t}^{(\ell-1)}$}
	\State{$\mathbf{F}_t^{(\ell)}=\mathrm{Log}_{\mathbf{S}_t^{(\ell)}}\left(\mathbf{S}_{2t-1}^{(\ell-1)}\right)$}
\EndFor
\EndFor
\Statex
\For{$\ell=1:\log_2T$} \Comment{Construct the new representations}
\For{$t=1:T/2^\ell$}
	\State{$\Psi^{(\mathbf{S}_t^{(\ell)})}=\left[\psi^{(\mathbf{S}_t^{(\ell)})}_1,\dots,\psi^{(\mathbf{S}_t^{(\ell)})}_M\right]$}
	\State{$\Psi^{(\mathbf{F}_t^{(\ell)})}=\left[\psi^{(\mathbf{F}_t^{(\ell)})}_1,\dots,\psi^{(\mathbf{F}_t^{(\ell)})}_M\right]$}
	\State{where $\psi^{(\mathbf{S}_t^{(\ell)})}$ and $\psi^{(\mathbf{F}_t^{(\ell)})}$ are the eigenvectors of $\mathbf{S}_t^{(\ell)}$ and $\mathbf{F}_t^{(\ell)}$.}
\EndFor
\EndFor
\end{algorithmic}
\end{algorithm}


\subsection{Implementation Remarks\label{sec:implement}}

{The numerical implementation of the proposed algorithm, particularly the Riemannian composition of operators, needs some elaboration. While the diffusion operators we consider in \eqref{eq:SPDKern} are  SPD matrices by definition,} in practice, some of their eigenvalues could be {close to zero numerically,} forming in effect SPSD matrices instead.
In order to address this issue, we propose an equivalent definition of the operators $\mathbf{S}$ and $\mathbf{F}$ for SPSD matrices of fixed rank, based on the Riemannian metric and mean that were introduced in \cite{bonnabel2010riemannian}.

Based on the approximated geodesic path in \eqref{eq:riemann_spsd_geodesic}, we define the operators $\mathbf{S}$ and $\mathbf{F}$ for SPSD matrices as follows. First, define
\begin{eqnarray}
    \mathbf{S} & = & \tilde{\gamma}_{\mathbf{W}_1\rightarrow\mathbf{W}_2}(0.5)=\mathbf{U}_{\mathbf{W}_1\rightarrow\mathbf{W}_2}(0.5)\left(\mathbf{R}_1\#_{0.5}\mathbf{R}_2\right)\mathbf{U}^T_{\mathbf{W}_1\rightarrow\mathbf{W}_2}(0.5)\,.
\end{eqnarray}
{Next, evaluate $\mathbf{V}_1^T\mathbf{V}_\mathbf{S}=\mathbf{O}_\mathbf{S}\tilde{\mathbf{\Sigma}}\tilde{\mathbf{O}}^T_1$ by SVD, where $\mathbf{\Lambda}_\mathbf{S}$ and $\mathbf{V}_\mathbf{S}$ denote the eigenvalues and eigenvectors of $\mathbf{S}$, respectively. Also, define $\mathbf{R}_\mathbf{S}:=\mathbf{O}_\mathbf{S}^T\mathbf{\Lambda}_\mathbf{S}\mathbf{O}_\mathbf{S}$.
Then, define}
\begin{eqnarray}
    \mathbf{F} & = & \mathbf{U}_{\mathbf{S}\rightarrow\mathbf{W}_1}(1)\mathrm{Log}_{\mathbf{R}_\mathbf{S}}\left(\tilde{\mathbf{O}}_1^T\mathbf{\Lambda}_1\tilde{\mathbf{O}}_1\right)\mathbf{U}^T_{\mathbf{S}\rightarrow\mathbf{W}_1}(1)
\end{eqnarray}
where $\mathbf{U}_{\mathbf{S}\rightarrow\mathbf{W}_1}(p)$ is defined for matrices $\mathbf{S}$ and $\mathbf{W}_1$ correspondingly to \eqref{eq:grassman_geodesic} and the derivations leading to it.
Intuitively, this can be viewed as applying the operators $\mathbf{S}$ and $\mathbf{F}$ to matrices expressed in the bases associated with the non-trivial SPD matrices of rank $r$, and then projecting the resulting operators back to the original space of SPSD matrices by applying $\mathbf{U}_{\mathbf{W}_1\rightarrow\mathbf{W}_2}(p)$.

A summary of this derivation is presented in Algorithm \ref{alg:SA_implementation}.
A demonstration of the properties of these new operators for SPSD matrices using a simple simulation of $4\times 4$ matrices, similar to the one presented in Subsection \ref{subsub:ills_exmp} but with matrices of rank $3$, could be found in Section \ref{subsub:mat4x4_exmp_fxd_rnk}.

\begin{algorithm}
\caption{Implementation of the operators for SPSD matrices}
\textbf{Input:} Two datasets with point correspondence $\{\mathbf{x}_1[i],\mathbf{x}_2[i]\}_{i=1}^N$\\
\textbf{Output:} Operators $\mathbf{S}$ and $\mathbf{F}$ and their eigenvectors:  $\mathbf{\Psi}^{(\mathbf{S})},\mathbf{\Psi}^{(\mathbf{F})}$\\
\begin{algorithmic}[1]
\Statex
\Function{SPSD-Geodesics}{$\mathbf{G}_1,\mathbf{G}_2,p$}
\Comment{As defined in \cite{bonnabel2010riemannian}}
\State{Set $r=\min\left\lbrace\textrm{rank}\left(\mathbf{G}_1\right),\textrm{rank}\left(\mathbf{G}_2\right)\right\rbrace$.}
\For{$i\in[1,2]$}
\State{Set $\mathbf{G}_i=\mathbf{V}_i\mathbf{\Lambda}_i\mathbf{V}_i^T$} \Comment{Eigenvalue decomposition}
\State{Set $\tilde{\mathbf{V}}_i=\mathbf{V}_i(:,1:r)$}
\State{Set $\tilde{\mathbf{\Lambda}}_i=\mathbf{\Lambda}_i(1:r,1:r)$}
\EndFor
\State{Set $[\mathbf{O}_1,\mathbf{\Sigma},\mathbf{O}_2]=\mathrm{SVD}\left(\tilde{\mathbf{V}}_2^T\tilde{\mathbf{V}}_1\right)$}
\State{Set $\mathbf{\Theta}=\arccos(\mathbf{\Sigma})$}
\For{$i\in[1,2]$}
\State{Set $\mathbf{U}_i=\tilde{\mathbf{V}}_i\mathbf{O}_i$}
\State{Set $\mathbf{R}_{\mathbf{G}_i}=\mathbf{O}_i^T\tilde{\mathbf{\Lambda}}_i\mathbf{O}_i$}
\EndFor
\State{Compute $\mathbf{U}_{\mathbf{G}_1\rightarrow\mathbf{G}_2}(p)$} \Comment{According to \eqref{eq:grassman_geodesic}}
\State{Compute $\mathbf{R}_{\mathbf{G}_1\rightarrow\mathbf{G}_2}(p)=\mathbf{R}_1^{1/2}\left(\mathbf{R}_1^{-1/2}\mathbf{R}_2\mathbf{R}_1^{-1/2}\right)^{p}\mathbf{R}_1^{1/2}$}
\State{\textbf{return} $\mathbf{U}_{\mathbf{G}_1\rightarrow\mathbf{G}_2}(p)$, $\mathbf{R}_{\mathbf{G}_1\rightarrow\mathbf{G}_2}(p)$, $\mathbf{R}_{\mathbf{G}_1}$, $\mathbf{R}_{\mathbf{G}_2}$}
\EndFunction
\Statex
\Function{Main}{}
\State{Construct SPSD matrices for the two datasets $\mathbf{W}_1$ and $\mathbf{W}_2$} \Comment{According to \eqref{eq:SPDKern}}
\Statex
\State{\Call{SPSD-Geodesics}{$\mathbf{W}_1,\mathbf{W}_2,0.5$}}
\State{$\mathbf{S}=\mathbf{U}_{\mathbf{W}_1\rightarrow\mathbf{W}_2}(0.5)\mathbf{R}_{\mathbf{W}_1\rightarrow\mathbf{W}_2}(0.5)\mathbf{U}^T_{\mathbf{W}_1\rightarrow\mathbf{W}_2}(0.5)$}
\Statex
\State{\Call{SPSD-Geodesics}{$\mathbf{S},\mathbf{W}_1,1$}}
\State{$\mathbf{F}=\mathbf{U}_{\mathbf{S}\rightarrow\mathbf{W}_1}(1)\mathrm{Log}_{\mathbf{R}_{\mathbf{S}}}\left(\mathbf{R}_{\mathbf{W}_1}\right)\mathbf{U}^T_{\mathbf{S}\rightarrow\mathbf{W}_1}(1)$} \Comment{$\mathrm{Log}_\cdot(\cdot)$ is defined as in \eqref{eq:spd_logmap}}
\EndFunction
\end{algorithmic}
\label{alg:SA_implementation}
\end{algorithm}



\section{Experimental Results\label{sec:results}}

\subsection{Illustrative Toy Example: SPD Case\label{subsub:ills_exmp}}
We demonstrate the properties of the composite operators $\mathbf{S}$ and $\mathbf{F}$, constructed in Algorithm \ref{alg:twomod}, using a simple simulation of $4\times4$ matrices.
Define two matrices, $\mathbf{M}_1=\mathbf{\Psi}\mathbf{\Lambda}^{(1)}\mathbf{\Psi}^T$ and $\mathbf{M}_2=\mathbf{\Psi}\mathbf{\Lambda}^{(2)}\mathbf{\Psi}^T$, with the following common eigenvectors:
\begin{eqnarray}\label{toy example eigenvector matrix}
\mathbf{\Psi} = \left[
\begin{matrix}
\psi_1, \psi_2, \psi_3, \psi_4
\end{matrix}
\right] = \frac{1}{2}\left[
\begin{matrix}
1,\ \ 1,\ \ 1,\ \ 1\\
1,\ \ 1, -1, -1\\
1, -1, -1,\ \ 1\\
1, -1,\ \ 1, -1
\end{matrix}
\right]
\end{eqnarray}
and the following eigenvalues:
\begin{eqnarray}
\mathbf{\Lambda}^{(1)} = \mathrm{diag}\left(\left[\lambda^{(1)}_1,\lambda^{(1)}_2,\lambda^{(1)}_3,\lambda^{(1)}_4\right]\right) = \mathrm{diag}(\left[ 0.5,\ \ 1,\ 0.01,\ 0.2 \right])\\
\mathbf{\Lambda}^{(2)} = \mathrm{diag}\left(\left[\lambda^{(2)}_1,\lambda^{(2)}_2,\lambda^{(2)}_3,\lambda^{(2)}_4\right]\right) = \mathrm{diag}(\left[ 0.01, \ 1,\ 0.5,\ \ 0.2\right])
\end{eqnarray}
In this example, $\psi_1$ is a common eigenvector that is dominant in $\mathbf{M}_1$ and weak in $\mathbf{M}_2$, $\psi_3$ is a common eigenvector that is dominant in $\mathbf{M}_2$ and weak in $\mathbf{M}_1$ and $\psi_2$ and $\psi_4$ are common eigenvectors that are similarly expressed in both $\mathbf{M}_1$ and $\mathbf{M}_2$.

We construct the operators $\mathbf{S}=\mathbf{M}_1\#\mathbf{M}_2$ and $\mathbf{F}=\mathrm{Log}_{\mathbf{S}}\left(\mathbf{M}_1\right)$ and compute their eigenvalues and eigenvectors. 
Figure \ref{fig:4x4mat_full} presents the $4$ eigenvalues of $\mathbf{M}_1$, $\mathbf{M}_2$, $\mathbf{S}$ and $\mathbf{F}$, denoted by $\{\lambda^{(1)}_n\}_{n=1}^4$, $\{\lambda^{(2)}_n\}_{n=1}^4$, $\{\lambda^{(\mathbf{S})}_n\}_{n=1}^4$ and $\{\lambda^{(\mathbf{F})}_n\}_{n=1}^4$, respectively, in the left plots, and the corresponding eigenvectors in the right plots. 
This figure depicts that the two matrices $\mathbf{M}_1$ and $\mathbf{M}_2$ share the same $4$ eigenvectors (as defined) and that the resulting eigenvectors of $\mathbf{S}$ and $\mathbf{F}$ are similar to these $4$ eigenvectors.
Note that eigenvectors $2$ and $4$ of operator $\mathbf{F}$ are not identical to the eigenvectors of $\mathbf{M}_1$ and $\mathbf{M}_2$ due to numerical issues, which arise since these eigenvectors in $\mathbf{F}$ correspond to negligible eigenvalues.
The left plots show that the eigenvalues of $\mathbf{S}$ and $\mathbf{F}$ capture the similarities and differences in the expression of the spectral components of $\mathbf{M}_1$ and $\mathbf{M}_2$.
Specifically, since $\lambda^{(1)}_2=\lambda^{(2)}_2$ and $\lambda^{(1)}_4=\lambda^{(2)}_4$, the corresponding eigenvalues of $\mathbf{S}$ assume the same magnitude.
In contrast, due to this equality, these eigenvalues correspond to negligible eigenvalues of $\mathbf{F}$.
The two other eigenvectors, $\psi_1$ and $\psi_3$, correspond to eigenvalues that differ by an order of magnitude in the two matrices and are therefore the most dominant components in $\mathbf{F}$.
In addition, note the opposite sign of eigenvalues $\lambda^{(\mathbf{F})}_1$ and $\lambda^{(\mathbf{F})}_3$, which indicates the source of the more dominant component, i.e. whether $\lambda^{(1)}_n>\lambda^{(2)}_n$ or $\lambda^{(1)}_n{<}\lambda^{(2)}_n$.
These properties are proved and explained in more detail in Section \ref{sec:analysis}.

\begin{figure}[bhtp!]
\centering
\includegraphics[width=1\textwidth]{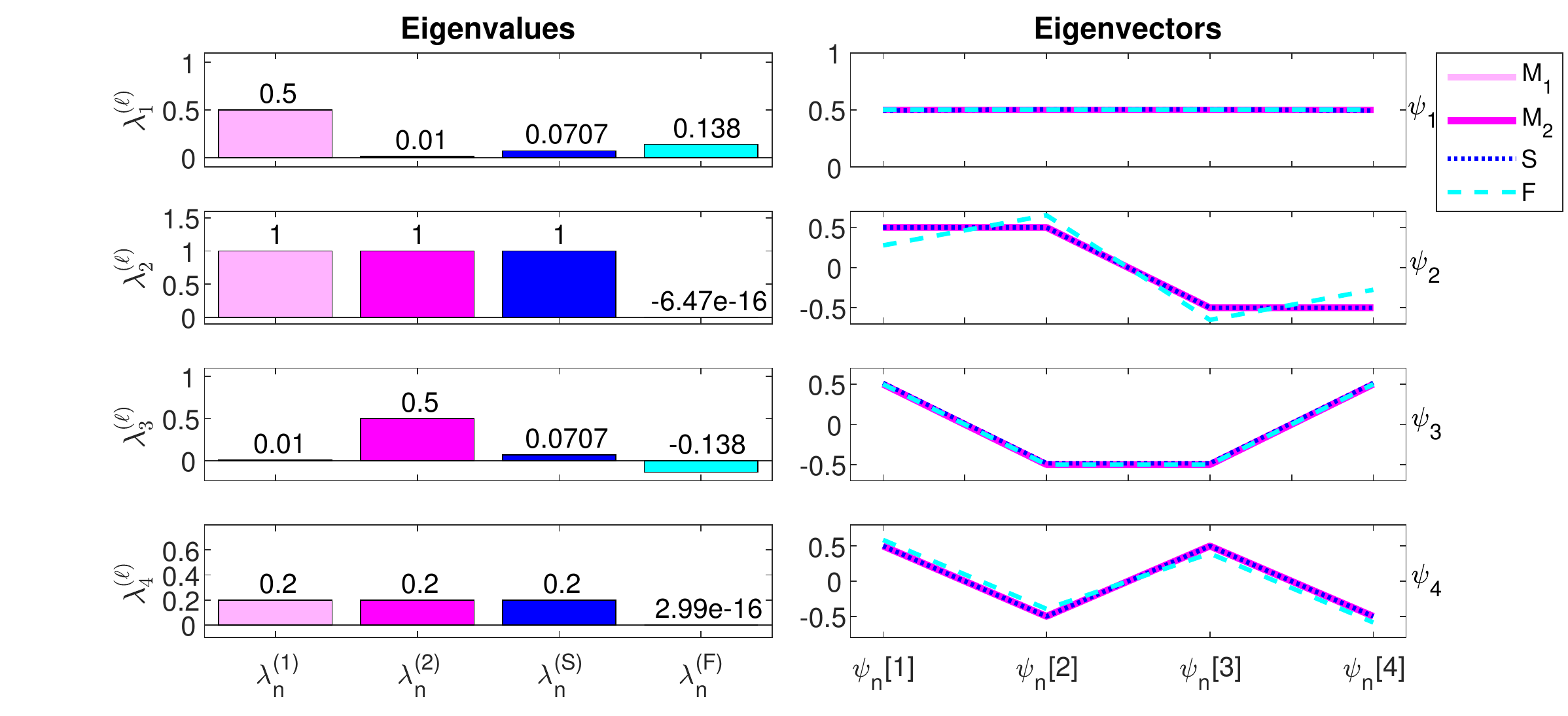}
\caption{Application of $\mathbf{S}$ and $\mathbf{F}$ to two $4\times 4$ matrices with identical eigenvectors.\label{fig:4x4mat_full}}
\end{figure}

\subsection{Illustrative Toy Example: SPSD case\label{subsub:mat4x4_exmp_fxd_rnk}}

Consider two matrices, $\mathbf{M}_1=\mathbf{\Psi}\mathbf{\Lambda}^{(1)}\mathbf{\Psi}^T$ and $\mathbf{M}_2=\mathbf{\Psi}\mathbf{\Lambda}^{(2)}\mathbf{\Psi}^T$, {with $\mathbf{\Psi}$ defined in \eqref{toy example eigenvector matrix}}
and the following eigenvalues:
\begin{eqnarray}
\mathbf{\Lambda}^{(1)} = \mathrm{diag}\left(\left[\lambda^{(1)}_1,\lambda^{(1)}_2,\lambda^{(1)}_3,\lambda^{(1)}_4\right]\right) = \mathrm{diag}(\left[ 0.5,\ \ 1,\ 0.01,\ 0 \right])\\
\mathbf{\Lambda}^{(2)} = \mathrm{diag}\left(\left[\lambda^{(2)}_1,\lambda^{(2)}_2,\lambda^{(2)}_3,\lambda^{(2)}_4\right]\right) = \mathrm{diag}(\left[ 0.01, \ 1,\ 0.5,\ \ 0\right])
\end{eqnarray}
Note that the $4$th eigenvalue is zero in both matrices resulting in SPSD matrices of rank $3$.

We construct the operators $\mathbf{S}$ and $\mathbf{F}$ according to Algorithm \ref{alg:SA_implementation} and compute their eigenvalues and eigenvectors. The results are presented in Figure \ref{fig:4x4mat_rank3}. Same as Figure \ref{fig:4x4mat_full},
Figure \ref{fig:4x4mat_rank3} presents in the left plots the $4$ eigenvalues of $\mathbf{M}_1$, $\mathbf{M}_2$, $\mathbf{S}$ and $\mathbf{F}$, denoted by $\{\lambda^{(1)}_n\}_{n=1}^4$, $\{\lambda^{(2)}_n\}_{n=1}^4$, $\{\lambda^{(\mathbf{S})}_n\}_{n=1}^4$ and $\{\lambda^{(\mathbf{F})}_n\}_{n=1}^4$, respectively, and the corresponding eigenvectors in the right plots.
Both matrices, $\mathbf{M}_1$ and $\mathbf{M}_2$, share the same $4$ eigenvectors as depicted in the right plots, and the resulting eigenvectors of $\mathbf{S}$ and $\mathbf{F}$ are similar to these $4$ eigenvectors.
In this example, $\psi_2$ is a dominant component in both $\mathbf{M}_1$ and $\mathbf{M}_2$ with the same large eigenvalue. 
Therefore, similarly to the SPD case, the eigenvalue of $\mathbf{S}$ associated with this eigenvector remains large, whereas the eigenvalue of $\mathbf{F}$ associated with this eigenvector is negligible, as expected.
In contrast, $\psi_1$ and $\psi_3$ are eigenvectors that are differently expressed in the two matrices (corresponding to eigenvalues $0.5$ and $0.01$), and therefore, they correspond to dominant eigenvalues in $\mathbf{F}$.
The left plot demonstrates that this behavior is indeed captured by the operators $\mathbf{S}$ and $\mathbf{F}$ for SPSD matrices.
Moreover, the eigenvalues of the operators $\mathbf{S}$ and $\mathbf{F}$ for SPSD matrices are equal to the eigenvalues that were obtained by the operators for SPD matrices in a corresponding toy example, presented in Figure \ref{fig:4x4mat_full}.
Note that all eigenvalues that correspond to $\psi_4$ are very close to zero, as expected due to the definition of the matrices $\mathbf{M}_1$ and $\mathbf{M}_2$.

\begin{figure}[bhtp!]
\centering
\includegraphics[width=1\textwidth]{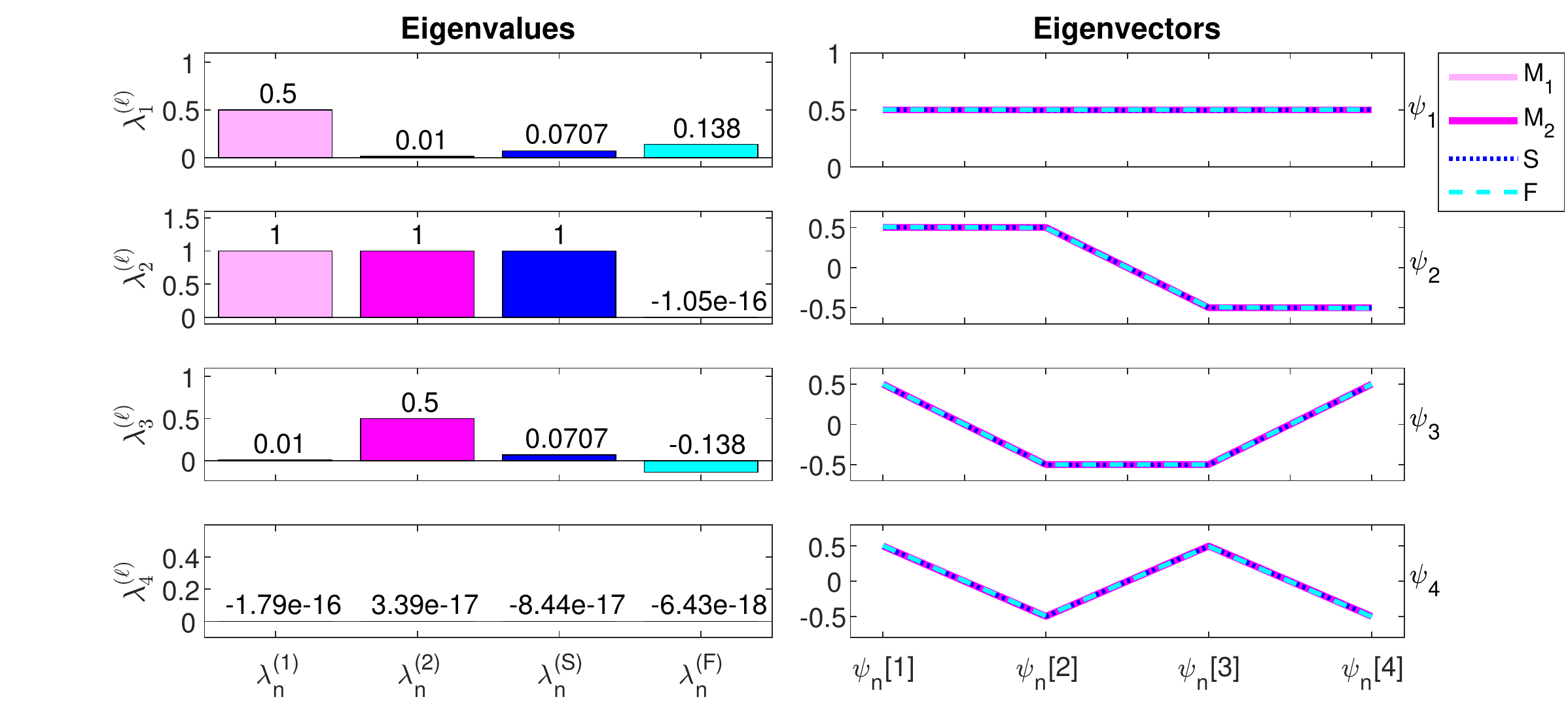}
\caption{Application of the operators $\mathbf{S}$ and $\mathbf{F}$ for SPSD matrices to two $4\times 4$ matrices of rank $3$ with identical eigenvectors.\label{fig:4x4mat_rank3}}
\end{figure}


\subsection{Transitory Double Gyre Flow\label{sub:2gyre}}

To demonstrate the proposed Riemannian multi-resolution analysis described in Section \ref{sec:rmra}, we consider a variation of the transitory double gyre flow presented in \cite{mosovsky2011transport,froyland2014almost}.

We simulate a 2D dynamical system with coordinates $(x_t,y_t)$ using the following equations:
\begin{eqnarray}
\dot{x}_t & = & -\frac{\partial}{\partial y_t}H\left(x_t,y_t,t\right)\\
\dot{y}_t & = & \frac{\partial}{\partial x_t}H\left(x_t,y_t,t\right)
\end{eqnarray}
with the function:
\begin{eqnarray}
H\left(x_t,y_t,t\right) & = & (1-g(t))H_1\left(x_t,y_t\right) + g(t)H_2\left(x_t,y_t\right)\\
H_1\left(x_t,y_t\right) & = & c_1\sin(2\pi x_t)\sin(\pi y_t)\\
H_2\left(x_t,y_t\right) & = & c_2\sin(\pi x_t)\sin(2\pi y_t)\\
g(t) & = & t^2(3-2t),
\end{eqnarray}
where $c_1=2$, $c_2=10$, $i=1,...,N$ and $t\in[0,1]$.

These equations describe a double gyre pattern, which is horizontal at $t=0$ and transitions into a vertical double gyre pattern at $t=1$.
Note that in our simulations we add the parameters $c_1$ and $c_2$ to the dynamics, which lead to a change in rate between the (slower) horizontal and (faster) vertical double gyre patterns. 
These parameters are added in order to demonstrate the time-varying multi-resolution properties of our analysis.

We generate $N=2500$ trajectories with initial values uniformly distributed in $\left(x_0,y_0\right)\in [0,1]\times[0,1]$, where each trajectory has $T=256$ time points on a discrete uniform time-grid with a step size of $\Delta t =1/256$.
We denote each of these trajectories with an index $i=1,\ldots, N$ by a matrix $\mathbf{x}[i] \in \mathbb{R}^{2 \times T}$, whose columns are the pair of time samples $(x_t[i],y_t[i])^T$.
A short GIF file demonstrating the resulting trajectories is available on \href{https://github.com/shnitzer/Manifold-based-temporal-analysis/blob/main/Transitory_Double_Gyre_Flow_Data.gif}{GitHub}, where each point is colored according to its initial location along the x-axis to illustrate the point movement in time.
The point movement demonstrated in this GIF exhibits two main structures: (i) points that rotate in two circular structures (transitioning from a horizontal setting into a vertical setting), which can be described as almost-invariant (coherent) sets as defined and captured by \cite{froyland2014almost}, and (ii) points that are located on the boundary of these almost-invariant sets and their movement changes significantly over time.
Our goal in this example is to analyze these two movement types and to recover their different trajectories over time.

For this purpose, we construct an SPD kernel for each time frame $t$, denoted by $\mathbf{W}_t\in\mathbb{R}^{N\times N}$, according to \eqref{eq:SPDKern} based on the distances between the points in that time frame, i.e., $\left\Vert(x_t[i]-x_t[j],y_t[i]-y_t[j])\right\Vert_2^2$, $i,j=1,\dots,N$, with $\sigma$ set to $0.5$ times the median of these distances.
We then apply Algorithm \ref{alg:wavelet} and obtain the multi-resolution representation and $\ell=\log_2(T)=8$ levels of operators.
We denote the operators of different levels and different time frames by $\mathbf{S}_r^{(\ell)}\in\mathbb{R}^{N\times N}$ and $\mathbf{F}_r^{(\ell)}\in\mathbb{R}^{N\times N}$, where $\ell=1,\dots,8$ and $r=1,\dots,T/2^\ell$. Note that $r$ is associated with the time-frame indices, e.g., at level $\ell=6$, $r=\lceil t/2^\ell\rceil=3$ corresponds to time points $t=129,\dots,192$.

In the following, we focus on the second eigenvector of $\mathbf{S}_r^{(\ell)}$ and show that it indeed captures the common components and common trajectory behavior at the different operator levels.
Figure \ref{fig:DG_S} presents the data-points colored according to the second eigenvector of $\mathbf{S}_r^{(\ell)}$, denoted by $\psi_2^{(\mathbf{S}_r^{(\ell)})}$ at levels $\ell=8$, $\ell=4$ and $\ell=3$ and with $r$ values corresponding to different time frames.
More specifically, the locations of all $N=2500$ points are presented at $8$ different time-instances along the trajectory. 
Each point is colored according to its value in: (a) $\psi^{(\mathbf{S}_1^{(8)})}_2$, (b) $\psi^{(\mathbf{S}_4^{(4)})}_2$, (c) $\psi^{(\mathbf{S}_{10}^{(4)})}_2$, (d) $\psi_2^{(\mathbf{S}_7^{(3)})}$ and (e) $\psi_2^{(\mathbf{S}_8^{(3)})}$.
Note that Figure \ref{fig:DG_S} (d) and Figure \ref{fig:DG_S} (e) present the preceding time frames of Figure \ref{fig:DG_S} (b).
We maintained a consistent color coding in all time-instances, i.e., each point has the same color throughout its trajectory in time, and the most significant values (largest in absolute value) are colored in either yellow or blue.

\begin{landscape}
\begin{figure}[bhtp!]
\centering
\includegraphics[height=0.8\textheight]{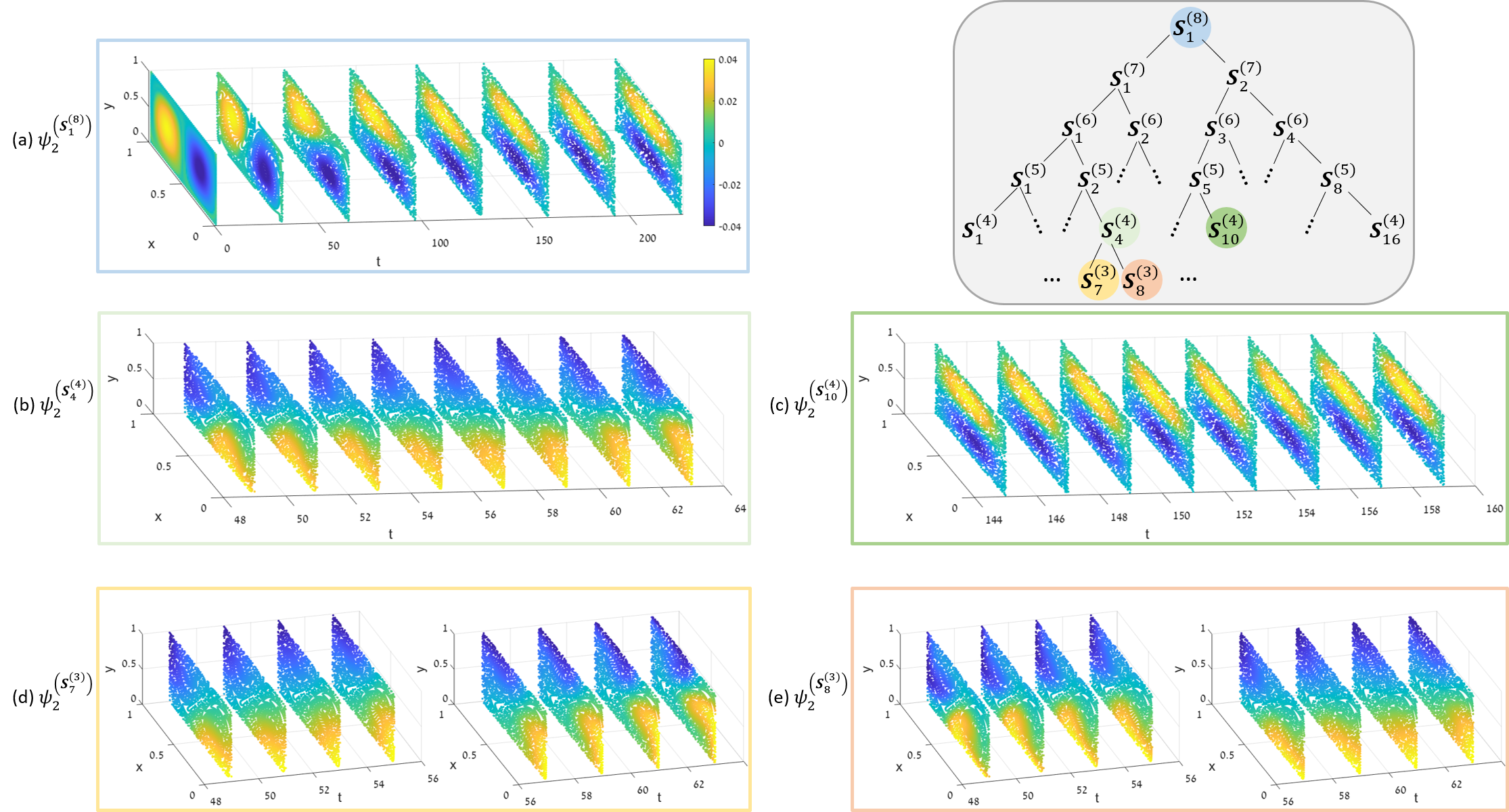}
\caption{Data points colored according to the second eigenvector of the operator $\mathbf{S}_r^{(\ell)}$ at different levels and time frames: (a) $\mathbf{S}_1^{(8)}$, (b) $\mathbf{S}_4^{(4)}$, (c) $\mathbf{S}_{10}^{(4)}$, (d) $\mathbf{S}_7^{(3)}$ and (e) $\mathbf{S}_8^{(3)}$. Plots (d) and (e) present the same $8$ time points as in plot (b), where the points are colored according to a different eigenvector in each plot.
    \label{fig:DG_S}}
\end{figure}
\end{landscape}


In the figure we see the multi-resolution properties of the proposed framework.
At the highest level, $\ell=8$ in plot (a), the eigenvector of $\mathbf{S}_1^{(8)}$ captures the coherent circular structures, i.e. the almost-invariant sets, which change from a horizontal orientation at the beginning of the trajectory to a vertical orientation at the end. These structures are consistent with the ones described by \cite{froyland2014almost}. 
In contrast, in plots (b)-(c) (level $\ell=4$), the effect of the velocity change over time is apparent, demonstrating that our framework is capable of detecting such properties. 
These plots present two equal-length sub-segments of the trajectory: $t\in\{49,\ldots,64\}$ in plot (b) and $t\in\{145,\ldots,160\}$ in plot (c).
Plot (c), which corresponds to the faster regime closer to the end of the trajectory, depicts that the circular structures are captured by the eigenvector, whereas in plot (b), which corresponds to the slower regime, these structures are not visible. 
Due to the increase in point movement velocity over time, the components that are similarly expressed over time in the sub-segment that is closer to the end of the trajectory (plot (c)) are mainly the almost-invariant sets, as captured by the operator from the highest level in plot (a). 
Conversely, in the slower regime, there are other components that are similarly expressed over short sub-segments in time, as captured by the eigenvector presented in plot (b).
Plots (d) and (e) correspond to the two sub-segments $t\in\{49,\ldots,56\}$ and $t\in\{57,\ldots,64\}$, respectively, whose union is the sub-segment presented in plot (b). 
Note the similarity between the captured point dynamics in the two sub-segments of plots (d) and (e). This similarity explains the component emphasized by the eigenvector in plot (b), which is constructed based on these two sub-segments.

Note that the leading eigenvector of the operator $\mathbf{S}_r^{(\ell)}$ was omitted throughout this example since it mostly captures the common point distribution at the different time-frames.  
The point distribution is of less interest in this example since it provides a general geometric description of the problem setting rather than the common trajectory properties.

In the following we present the eigenvectors of the operators $\mathbf{F}$ and show that they indeed capture the time-varying trajectory behavior in consecutive time-frames.

Figure \ref{fig:DG_A} presents the data-points colored according to the leading eigenvectors of the respective operators $\mathbf{F}$ corresponding to the largest positive and negative (in absolute value) eigenvalues. 
In this setting, the eigenvectors corresponding to negative eigenvalues describe the components that are significantly more dominant in the first half of the time segment and the eigenvectors corresponding to positive eigenvalues describe components that are significantly more dominant in the second half of the segment, as we will show in Section \ref{sec:analysis} (Theorem \ref{theo:A_eigs}).

Figure \ref{fig:DG_A} (a) and Figure \ref{fig:DG_A} (b) present the eigenvectors of the operator $\mathbf{F}_1^{(8)}$ corresponding to the smallest negative eigenvalue in plot (a) and to the largest positive eigenvalue in plot (b).
These plots depict that $\mathbf{F}_1^{(8)}$ captures the differences between the slower point movement in $t\in\{1,\ldots,128\}$ and the faster point movement in $t\in\{129,\ldots,256\}$. 
Due to the change in point movement velocity over time, the component describing the circular structures (the almost-invariant sets) is significantly more dominant in the sub-segment from the faster regime ($t\in\{129,\ldots,256\}$) than the slower regime ($t\in\{1,\ldots,128\}$), as captured by the eigenvector in plot (b). In contrast, in the slower regime other components are dominant (as demonstrated also by Figure \ref{fig:DG_S} (b)), leading to different structures being emphasized in Figure \ref{fig:DG_A} (a), which mostly captures the boundary points.
Figure \ref{fig:DG_A} (c) and Figure \ref{fig:DG_A} (d) correspond to the time segment $t\in\{129,\ldots,256\}$.
Plot (c) presents the leading eigenvector of the operator $\mathbf{F}_2^{(7)}$ with a negative eigenvalue, describing the components that are more dominant in the slower regime ($t\in\{129,\ldots,192\}$) and plot (d) presents the leading eigenvector with a positive eigenvalue, describing the components that are more dominant in the faster regime ($t\in\{193,\ldots,256\}$).
Note that both plots (c) and (d) emphasize circular structures, however, the structures in plot (d) are smaller than the ones in plot (b) and are approximately complemented by the structures in plot (c).
This behavior implies that our framework decomposes the almost-invariant sets into smaller components in short sub-segments (at lower operator-tree levels), and therefore, indicates that the proposed method indeed captures meaningful dynamical information in different time-scales.
Figure \ref{fig:DG_A} (e) presents the eigenvector of $\mathbf{F}_1^{(6)}$ (describing $t\in[\{1,\ldots,64\}$) with the largest negative eigenvalue.
This plot depicts that in the slower regime (at the beginning of the trajectory) the operator $\mathbf{F}$ highlights high-resolution fine components of the point movement dynamics.

\begin{landscape}
\begin{figure}[bhtp!]
\centering
\includegraphics[height=0.7\textheight]{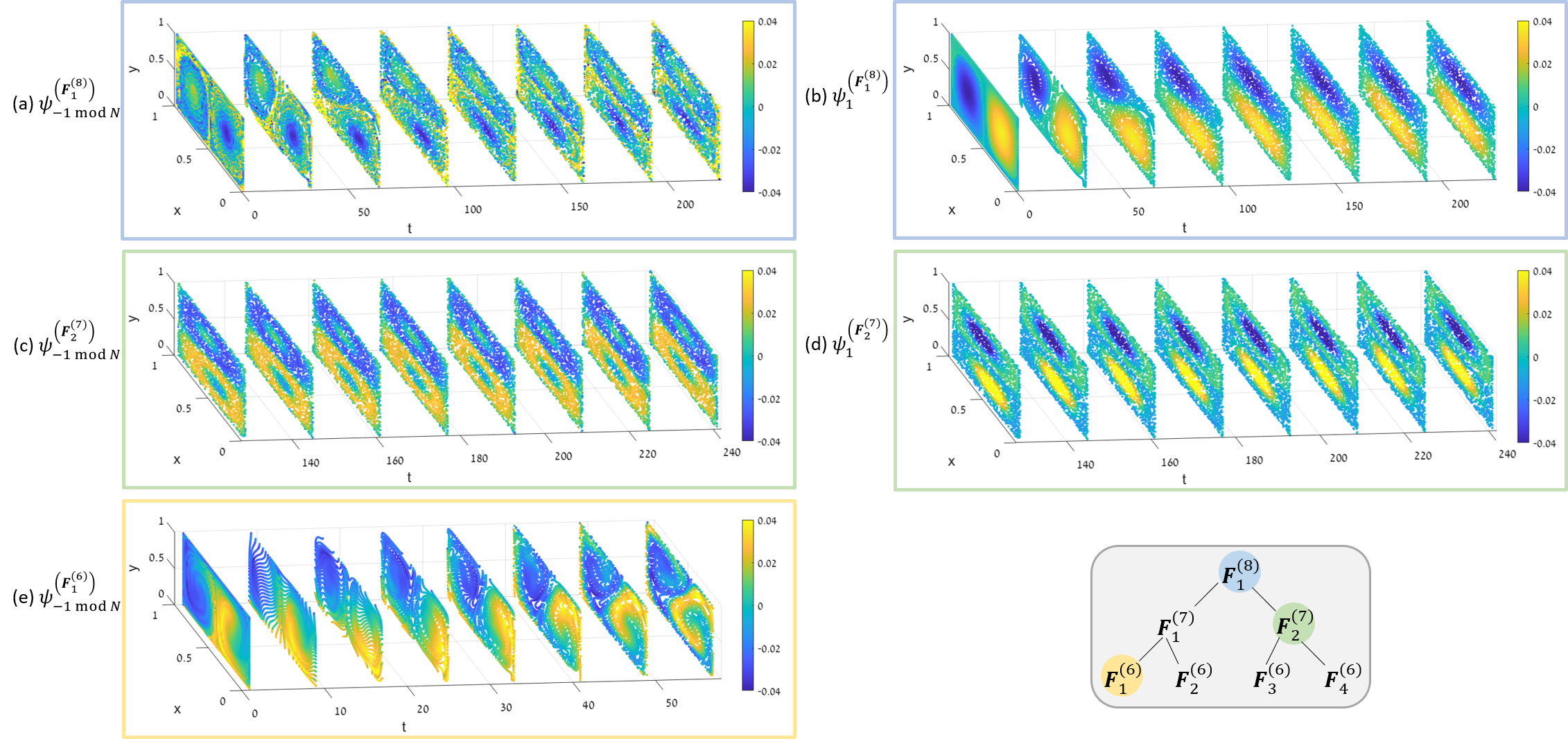}
\caption{Data points colored according to the eigenvector of (a-b) $\mathbf{F}_1^{(8)}$ with the smallest negative and largest positive eigenvalues, (c-d) $\mathbf{F}_2^{(7)}$ with the smallest negative and largest positive eigenvalues, and (e) $\mathbf{F}_1^{(6)}$ with the smallest negative eigenvalue.
\label{fig:DG_A}}
\end{figure}

\end{landscape}


We remark that different choices of the kernel scale in \eqref{eq:SPDKern} lead to different resolutions. 
For example, taking a smaller kernel scale leads to a slower ``convergence'' to the almost-invariant sets of the representations obtained at the different levels of the operator $\mathbf{S}_r^{(\ell)}$, as well as an enhancement of finer structures captured by the operator $\mathbf{F}_r^{(\ell)}$.

In order to evaluate our framework with respect to previous work, we compare the operators $\mathbf{S}$ and $\mathbf{F}$, which serve as the building blocks of our algorithm, with related operators: (i) the dynamic Laplacian \cite{froyland2015dynamic}, which was shown to recover coherent sets from multiple time-frames of dynamical systems, and (ii) symmetric and anti-symmetric diffusion operators that were shown to recover similar and different components in multimodal data \cite{shnitzer2019recovering}. 
In \cite{froyland2015dynamic}, with a slight abuse of notation and in analogy to \eqref{eq:Sr}, the dynamic Laplacian is defined by $\mathbf{L}^\top\mathbf{L}$, where $\mathbf{L}=\mathbf{W}_{1}\mathbf{W}_{2}$. The common and difference operators in \cite{shnitzer2019recovering} are defined by $\mathbf{\hat{S}}=\mathbf{W}_{1}\mathbf{W}_{2}^\top+\mathbf{W}_{2}\mathbf{W}_{1}^\top$ and $\mathbf{\hat{A}}=\mathbf{W}_{1}\mathbf{W}_{2}^\top-\mathbf{W}_{2}\mathbf{W}_{1}^\top$, respectively, which are analogous to the operators $\mathbf{S}$ and $\mathbf{F}$ in \eqref{eq:Sr} and \eqref{eq:Ar}.

Figure \ref{fig:op_compare_S} presents point clustering using k-means applied to the second eigenvector of the $3$ operators used for recovering similarities: the proposed operator $\mathbf{S}$ in plot (a), the operator $\mathbf{\hat{S}}$ from \cite{shnitzer2019recovering} in plot (b) and the operator $\mathbf{L}^\top\mathbf{L}$ from \cite{froyland2015dynamic} in plot (c).
All $3$ operators were constructed from time frames $t=250$ and $t=256$.
We see in this figure that the proposed formulation of the operator is significantly better at capturing the almost-invariant sets (the circular structures). 
Note that the results presented here for $\mathbf{L}^\top\mathbf{L}$ are different than those in \cite{froyland2015dynamic}, since we take into account only two close time frames, whereas in \cite{froyland2015dynamic} the operator is constructed using all the points along the trajectory.

\begin{figure}[t]
    \centering
    \subfloat[$\psi^{(\mathbf{S})}_2$]{\includegraphics[width=0.33\textwidth]{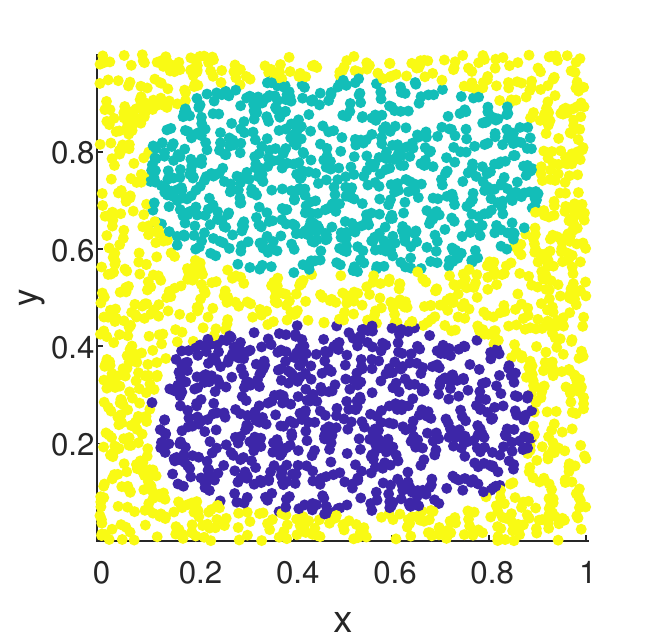}}
    \subfloat[$\psi^{(\mathbf{\hat{S}})}_2$ \cite{shnitzer2019recovering}]{\includegraphics[width=0.33\textwidth]{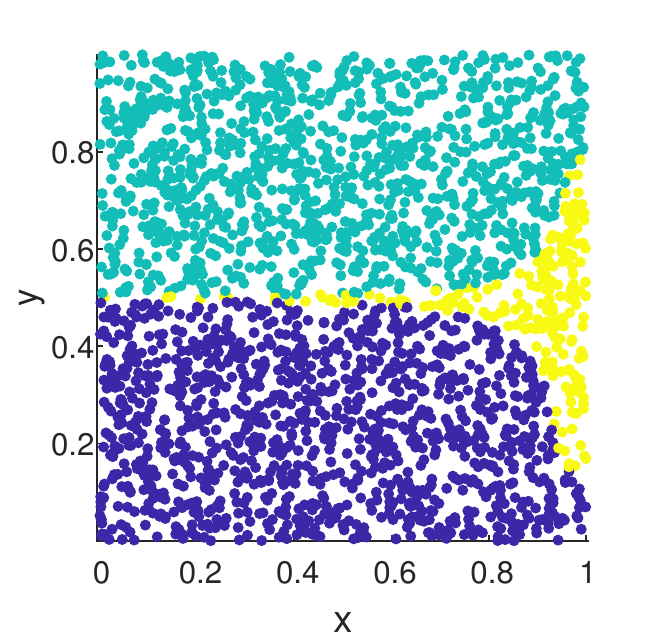}}
    \subfloat[$\psi^{(\mathbf{L}^T\mathbf{L})}_2$ \cite{froyland2015dynamic}]{\includegraphics[width=0.33\textwidth]{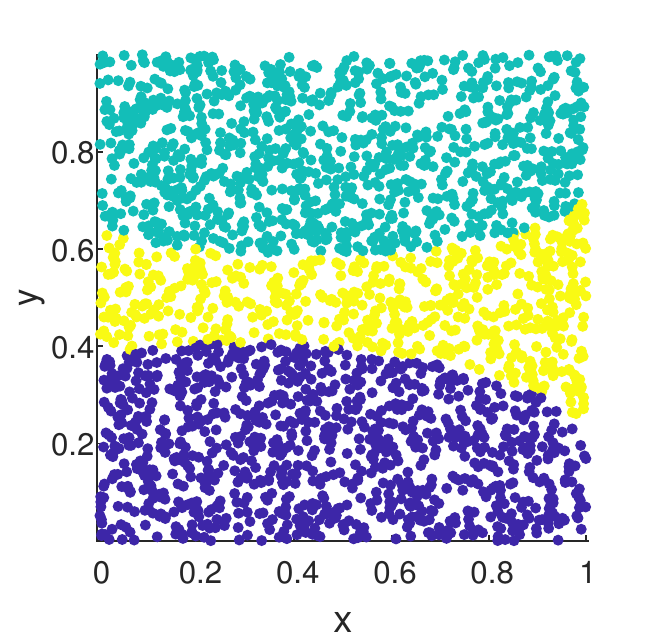}}
    \caption{Clustering of trajectory points based on eigenvectors of the proposed operator $\mathbf{S}$ in plot (a), and the operators from \cite{shnitzer2019recovering} in plot (b) and from \cite{froyland2015dynamic} in plot (c). All operators are constructed by combining two time frames at $t=250$ and $t=256$.}
    \label{fig:op_compare_S}
\end{figure}

In Figure \ref{fig:op_compare_A}, we present the point clustering obtained by applying k-means to the second eigenvector of the $2$ operators used for recovering differences: the proposed operator $\mathbf{F}$ in plot (a) and the operator $\mathbf{\hat{A}}$ from \cite{shnitzer2019recovering} in plot (b). In contrast to plot (b), in plot (a) we clearly see the swirl of the flow from and to the invariant sets (outward and inward).

\begin{figure}[bhtp!]
    \centering
    \subfloat[$\psi^{(\mathbf{F})}_2$]{\includegraphics[width=0.33\textwidth]{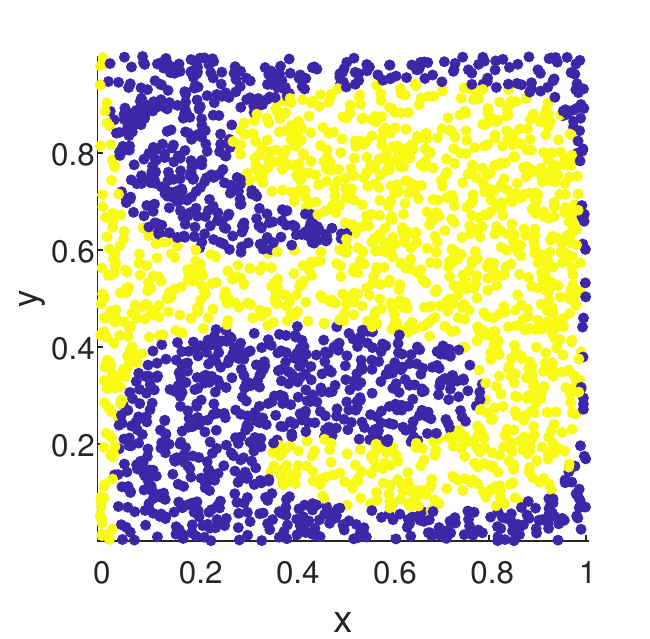}}
    \subfloat[$\psi^{(\mathbf{\hat{A}})}_2$ \cite{shnitzer2019recovering}]{\includegraphics[width=0.33\textwidth]{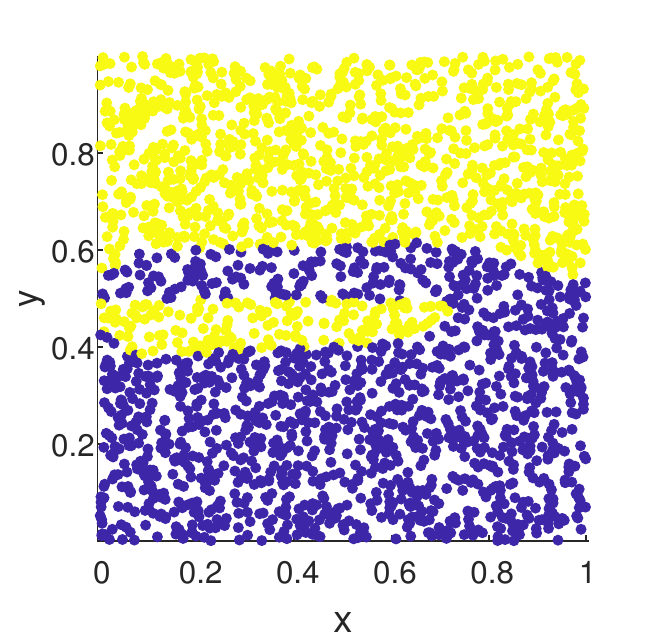}}
    
    \caption{Clustering of trajectory points based on eigenvectors of the proposed operator $\mathbf{F}$ in plot (a) and operator $\mathbf{\hat{A}}$ from \cite{shnitzer2019recovering} in plot (b).}
    \label{fig:op_compare_A}
\end{figure}

In addition to the differences demonstrated in Figure \ref{fig:op_compare_S} and Figure \ref{fig:op_compare_A}, another crucial advantage of our formulation relates to the construction of the multi-resolution framework and its theoretical justification presented in Section \ref{sec:analysis}.
The operators in \cite{shnitzer2019recovering} and \cite{froyland2015dynamic} do not have any theoretical guarantees in such an operator-tree construction and may not be suitable to such a setting. 
Indeed, we report that a similar operator-tree constructed using the operators from \cite{shnitzer2019recovering} did not exhibit the expected behavior and no meaningful representations were obtained.

We conclude by noting that such a multi-resolution analysis of the dynamics may be especially useful in applications where the parameters of interest are inaccessible, e.g., for oceanic current analysis based on ocean drifters data \cite{froyland2015rough,banisch2017understanding}, since the data is represented using non-linear kernels.


\subsection{Hyperspectral and LiDAR Imagery\label{sec:HSI_LIDAR}}

In Section \ref{sub:op_def}, we consider two datasets and present the two Riemannian composite operators $\mathbf{S}$ and $\mathbf{F}$.
Later, in Section \ref{sub:rmra}, these two datasets are considered as two consecutive sets in a temporal sequence of datasets, and the two Riemannian composite operators are used as a basic building block for our Riemannian multi-resolution analysis. 
Alternatively, similarly to the setting in \cite{lederman2015alternating,shnitzer2019recovering}, the two datasets could arise from simultaneous observations from two views or modalities. 
Here, we demonstrate the properties of the Riemannian composite operators $\mathbf{S}$ and $\mathbf{F}$ on real remote sensing data obtained by two different modalities. 

We consider data from the 2013 IEEE GRSS data fusion contest\footnote{\url{http://www.grss-ieee.org/community/technical-committees/data-fusion/2013-ieee-grss-data-fusion-contest/}}, which includes a hyperspectral image (HSI) with 144 spectral bands ($380-1050$nm range) and a LiDAR Digital Surface Model of the University of Houston campus and its neighboring urban area in Houston, Texas. 
The data from both modalities have the same spatial resolution of $2.5$m.
This data was previously considered in the context of manifold learning in \cite{murphy2018diffusion}.

We focus on two $60\times 90$ image patches from the full dataset, in order to reduce computation time of the operators and their eigenvalue decomposition.
We first preprocess the $60\times 90$ LiDAR image and each image in the $60\times 90\times 144$ HSI data.
The preprocessing stage includes dividing each $60\times 90$ image by its standard deviation and removing outliers: for the LiDAR image, pixel values larger than the $99$th percentile were removed, and for the HSI data, in each image, pixel values larger than the $95$th percentile or smaller than the $5$th percentile were removed. In both modalities the outliers were replaced by their nearest non-outlier values.
Figure \ref{fig:HSI_LDR_1_all} (a) and Figure \ref{fig:HSI_LDR_2_allEV} (a) present the two LiDAR image patches after preprocessing, and Figure \ref{fig:HSI_LDR_1_all} (b) and Figure \ref{fig:HSI_LDR_2_allEV} (b) present the two average HSI image patches after preprocessing.

We apply the operators $\mathbf{S}$ and $\mathbf{F}$ to this data in order to analyze the scene properties captured by both the LiDAR and the HSI sensors and extract the similarities and differences between them.
This can be viewed as a manifold-driven component analysis.

We construct the operators according to Algorithm \ref{alg:SA_implementation}, where the LiDAR image and the HSI images are defined as two datasets with point correspondence between them given by the pixel location. 
We reshape both datasets such that $\mathbf{x}_1\in\mathbb{R}^{N\times 1}$ is the reshaped LiDAR image and $\mathbf{x}_2\in\mathbb{R}^{N\times 144}$ is the reshaped HSI images, where $N=5400$. 
The resulting kernels, $\mathbf{W}_1$ and $\mathbf{W}_2$, and operators, $\mathbf{S}$ and $\mathbf{F}$, are matrices of size $N\times N$.

We begin with an analysis of the first chosen image patch, shown in Figure \ref{fig:HSI_LDR_1_all} (a) and (b).
To depict the advantages of applying the proposed operators, we visually compare the eigenvectors of the kernels, $\mathbf{W}_1$ and $\mathbf{W}_2$, with the eigenvectors of the operators $\mathbf{S}$ and $\mathbf{F}$.

Figure \ref{fig:HSI_LDR_1_all} (c-k) presents the absolute values of the leading eigenvectors of $\mathbf{S}$ in (c), $\mathbf{W}_1$ in (d-e), $\mathbf{W}_2$ in (h-i), and of $\mathbf{F}$ that correspond to the largest positive eigenvalues in (f-g) and largest negative (in absolute value) eigenvalues in (j-k).
All eigenvectors are reshaped into images of size $60\times 90$.
The absolute value of the eigenvectors is presented in order to emphasize the dominant structures in the images and the differences between the leading eigenvectors of the two kernels and the leading eigenvectors of the operators $\mathbf{S}$ and $\mathbf{F}$.

Figure \ref{fig:HSI_LDR_1_all} (c) presents the absolute values of the leading eigenvector of $\mathbf{S}$ and depicts that the operator $\mathbf{S}$ indeed recovers common structures strongly expressed in both images.
Specifically, this figure mostly highlights an `L'-shaped building at the top of the image, which is the most dominant structure (represented by the high pixel values) in both modalities.
Figure \ref{fig:HSI_LDR_1_all} (d-k) depicts that the eigenvectors of the operator $\mathbf{F}$ capture and enhance differently expressed common structures.
Consider for example the most dominant structures (with highest absolute values) in the LiDAR image presented in Figure \ref{fig:HSI_LDR_1_all} (a).
These structures include the `L'-shaped building at the top of the image and trees at the bottom.
Both structures are represented by high values in the eigenvectors of $\mathbf{W}_1$ in Figure \ref{fig:HSI_LDR_1_all} (d-e).
However, in Figure \ref{fig:HSI_LDR_1_all} (f-g), which presents the leading eigenvectors of $\mathbf{F}$ with positive eigenvalues, only the trees are significantly highlighted, whereas the `L'-shaped building is significantly attenuated.
This is due to the differences between the two modalities, since the HSI image highlights this `L' shaped building but not the trees.
Other structures exhibiting such properties are marked by black arrows in Figure \ref{fig:HSI_LDR_1_all} (f-g).
In addition, Figure \ref{fig:HSI_LDR_1_all} (h-k) depicts that the structures which are dominant only in the HSI images are emphasized by the eigenvectors of $\mathbf{F}$ corresponding to negative eigenvalues, whereas structures that are dominant in both modalities are significantly attenuated.
Examples for such structures are marked by black arrows in Figure \ref{fig:HSI_LDR_1_all} (j-k).

\begin{figure}[bhtp!]
\centering
\subfloat[]{\includegraphics[width=0.29\textwidth]{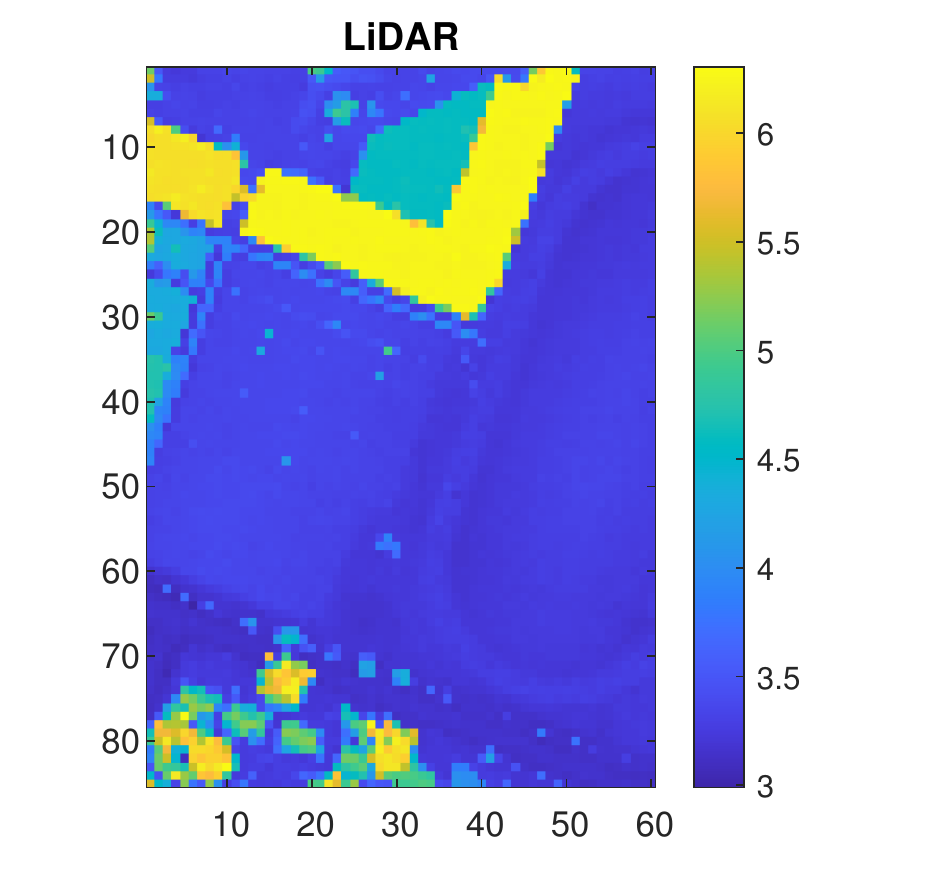}}
\subfloat[]{\includegraphics[width=0.29\textwidth]{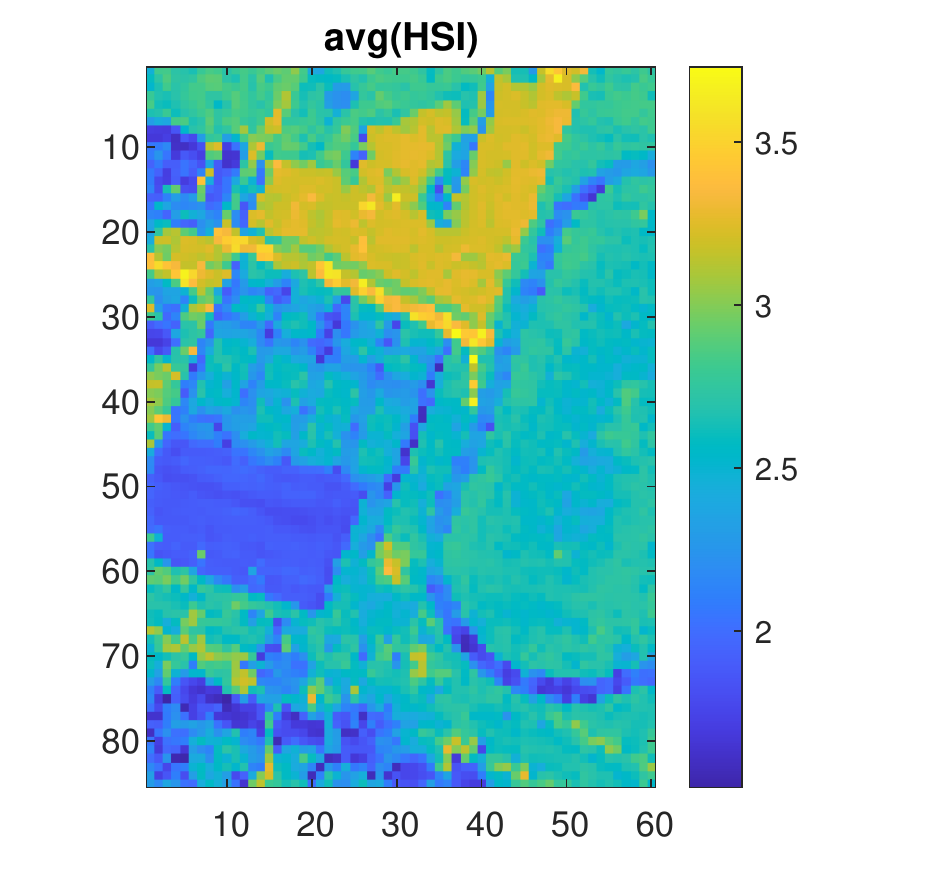}}
\subfloat[]{\includegraphics[width=0.25\textwidth]{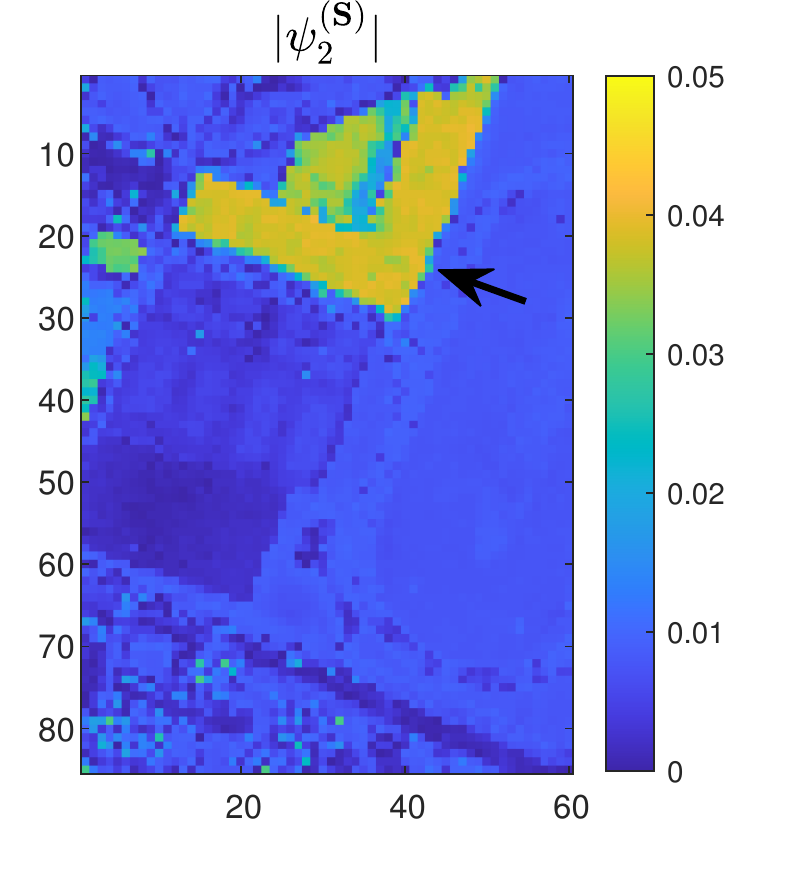}}

\subfloat[]{\includegraphics[width=0.25\textwidth]{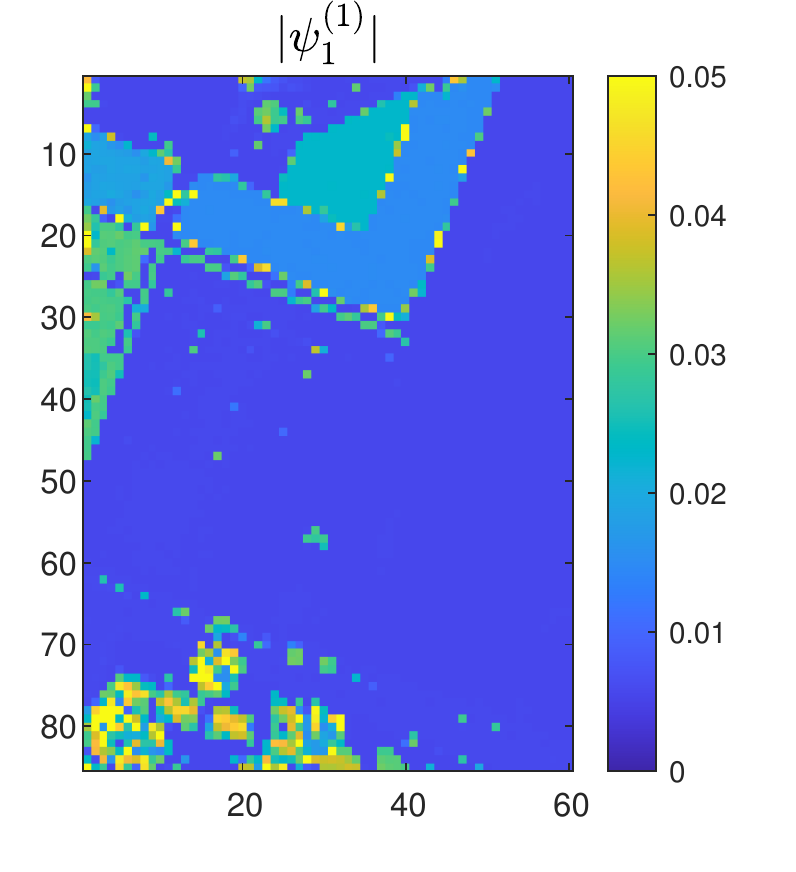}}
\subfloat[]{\includegraphics[width=0.25\textwidth]{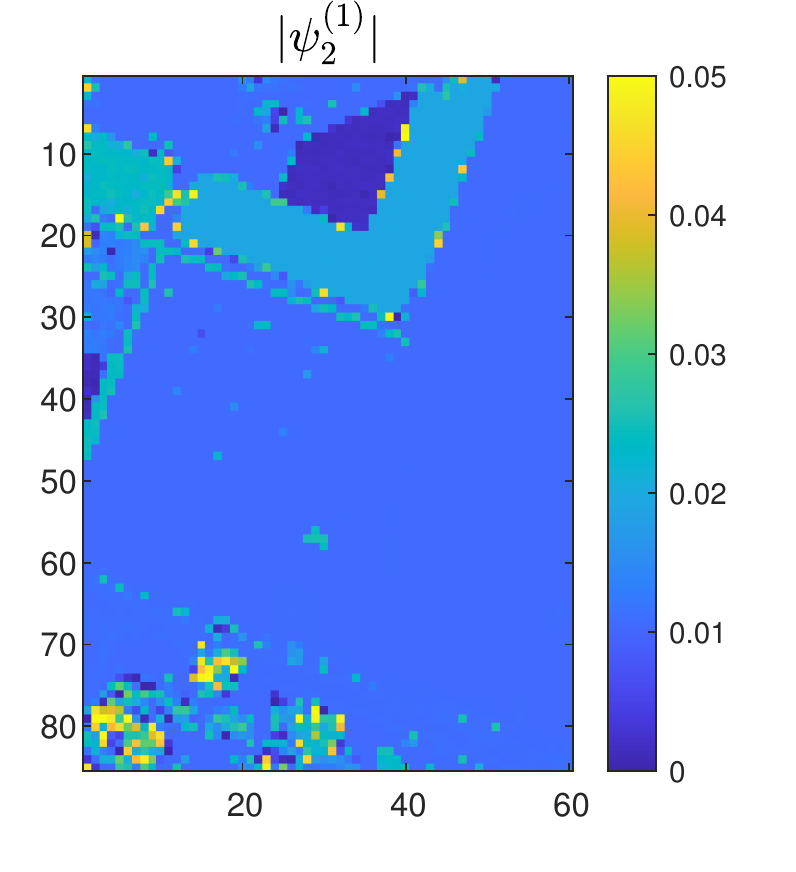}}
\subfloat[]{\includegraphics[width=0.25\textwidth]{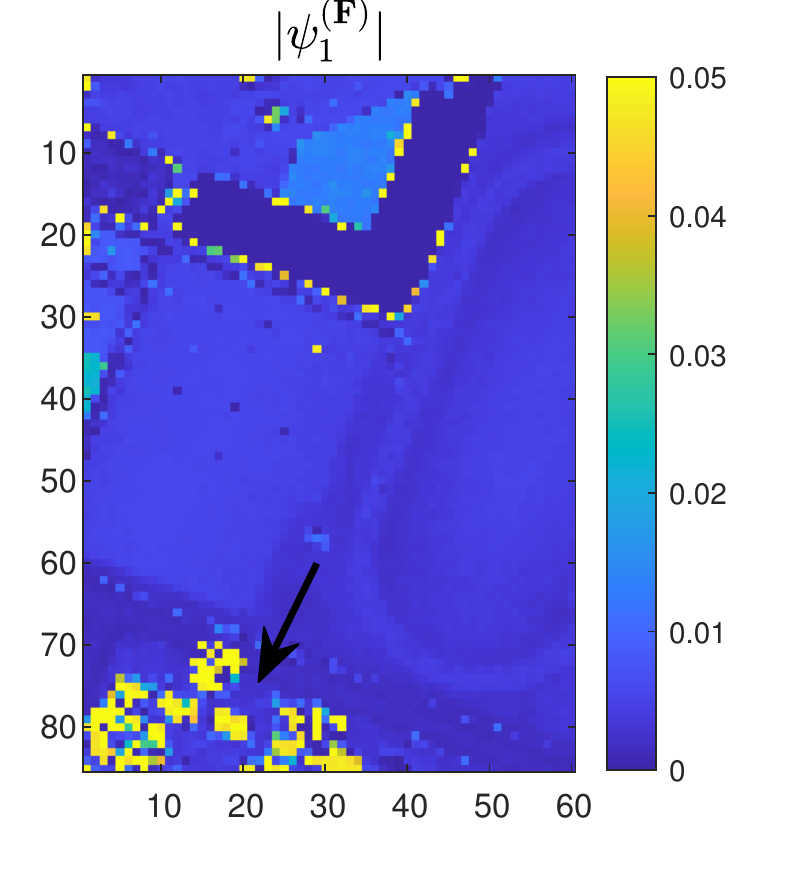}}
\subfloat[]{\includegraphics[width=0.25\textwidth]{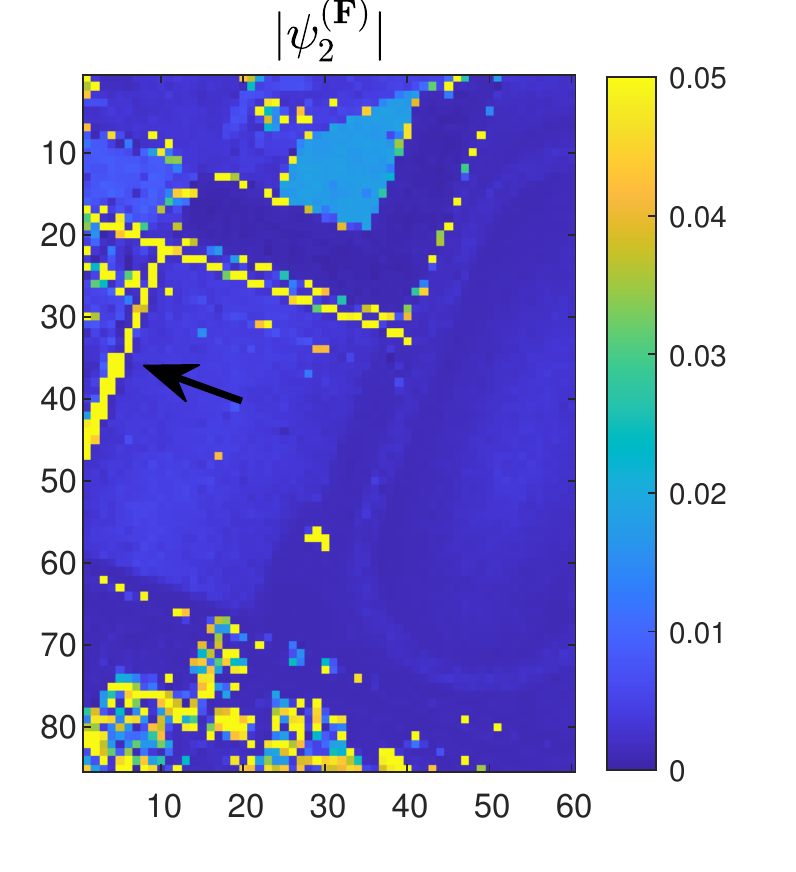}}

\subfloat[]{\includegraphics[width=0.25\textwidth]{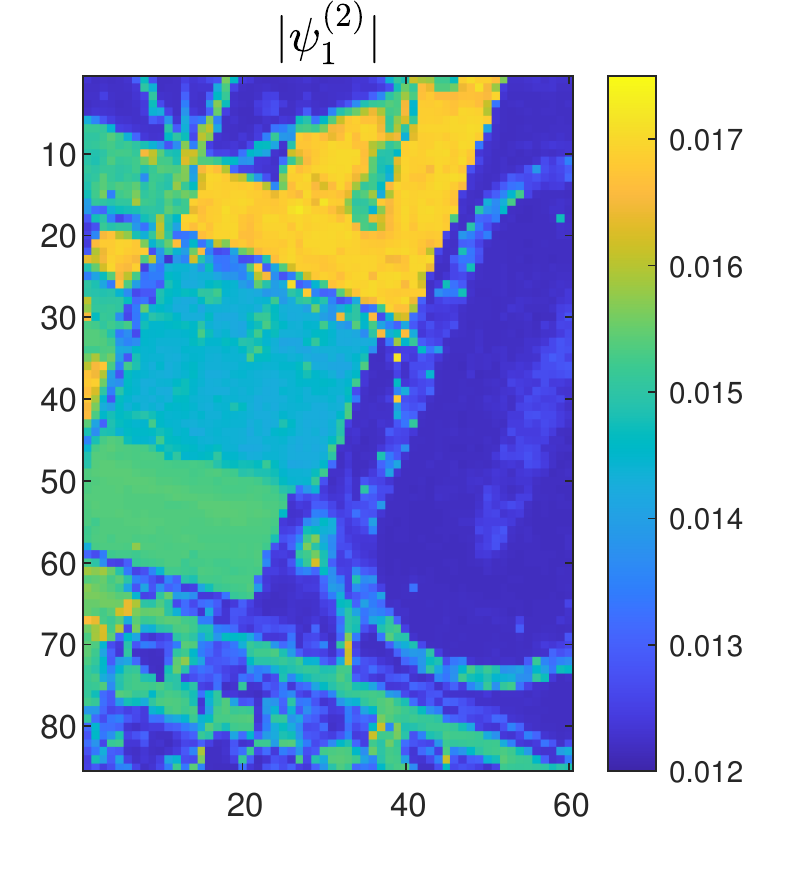}}
\subfloat[]{\includegraphics[width=0.25\textwidth]{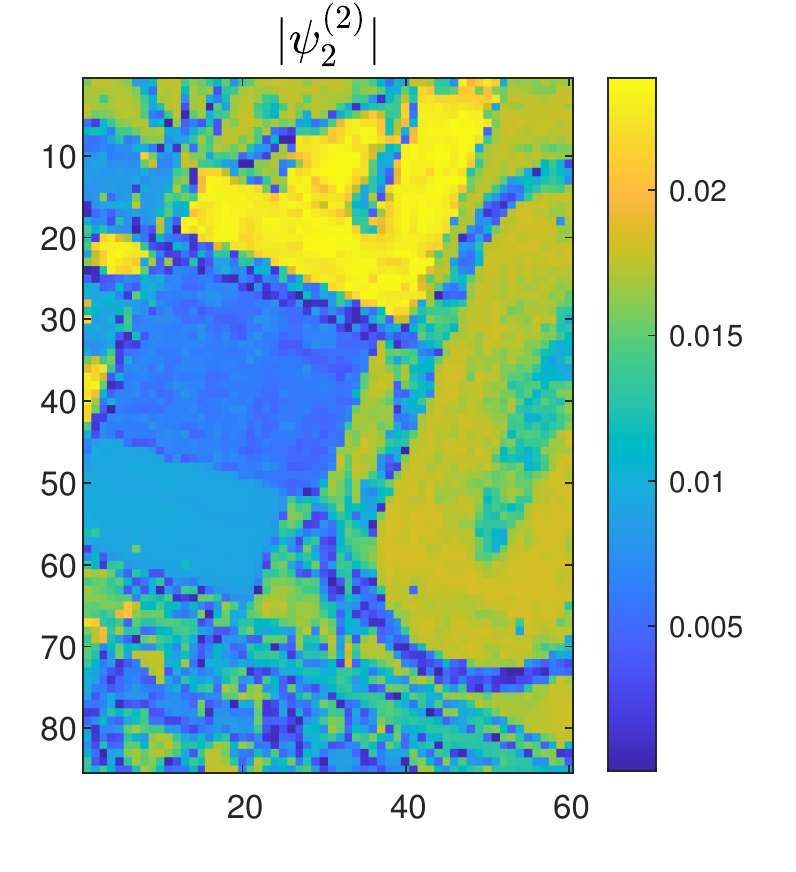}}
\subfloat[]{\includegraphics[width=0.25\textwidth]{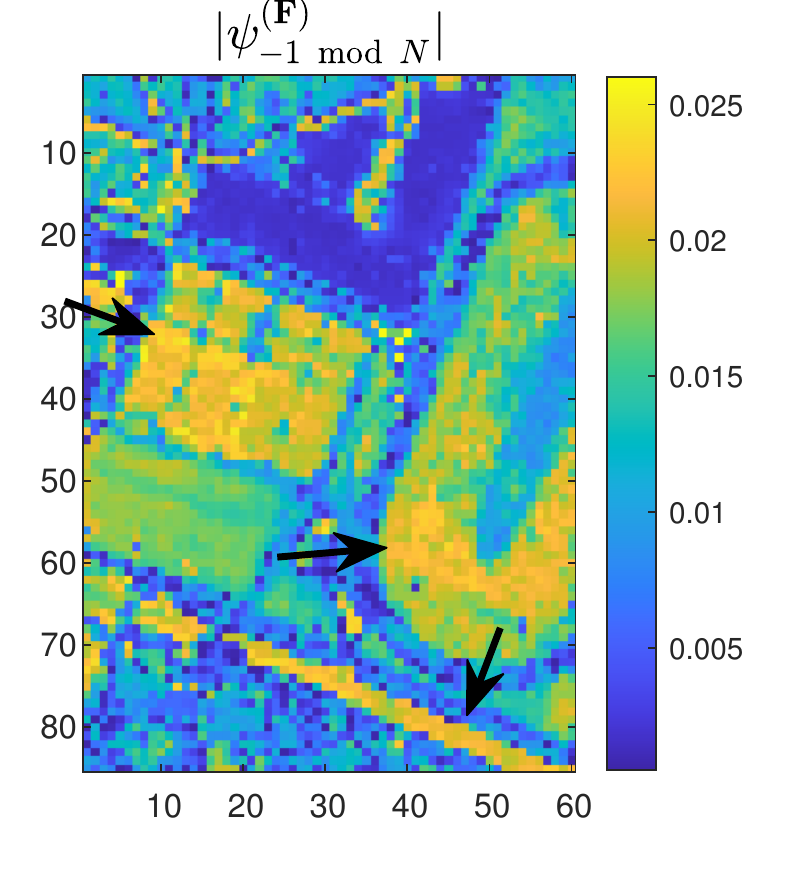}}
\subfloat[]{\includegraphics[width=0.25\textwidth]{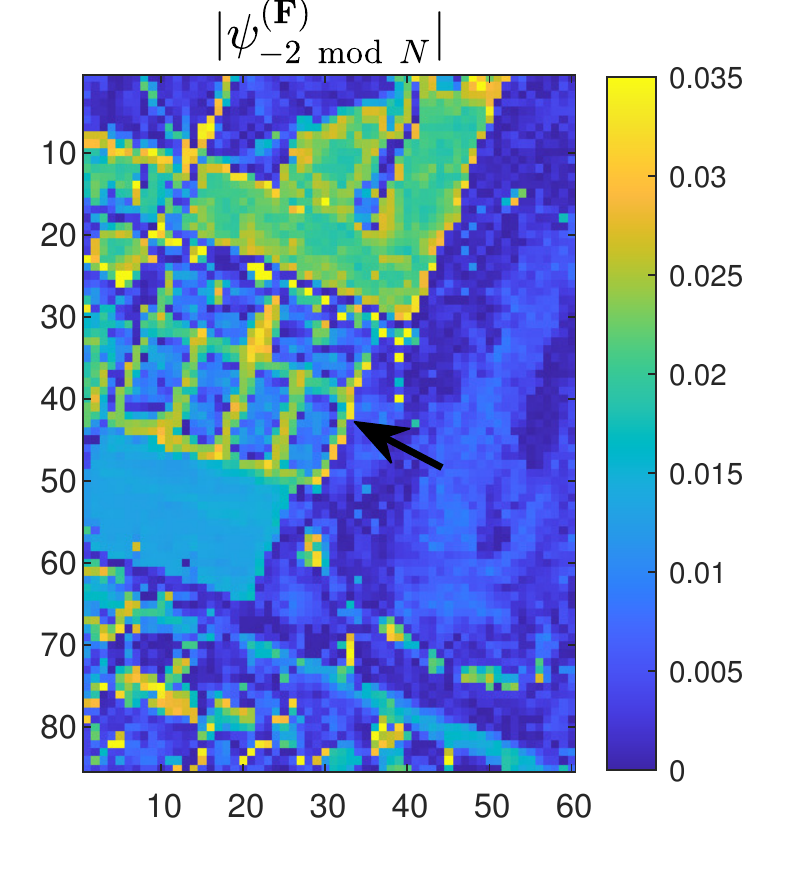}}
\caption{The two chosen image patches of (a) the LiDAR image and (b) the HSI image after preprocessing, along with the leading eigenvectors of (c) $\mathbf{S}$, (d-e) $\mathbf{W}_1$ (LiDAR), (f-g) $\mathbf{F}$ corresponding to its $2$ largest positive eigenvalues,
(h-i) $\mathbf{W}_2$ (HSI), (j-k) $\mathbf{F}$ corresponding to its two smallest negative eigenvalues.
\label{fig:HSI_LDR_1_all}}
\end{figure}

We repeat the presentation for the second image patch, shown in Figure \ref{fig:HSI_LDR_2_allEV} (a) and (b).
Figure \ref{fig:HSI_LDR_2_allEV} (c-g) presents the absolute value of the leading eigenvector of $\mathbf{W}_1$ in plot (c), of $\mathbf{W}_2$ in plot (d), of $\mathbf{F}$ with a positive eigenvalue in plot (e), of $\mathbf{F}$ with a negative eigenvalue in plot (f) and of $\mathbf{S}$ in plot (g).
\begin{figure}[bhtp!]
\centering
\subfloat[]{\includegraphics[width=0.33\textwidth]{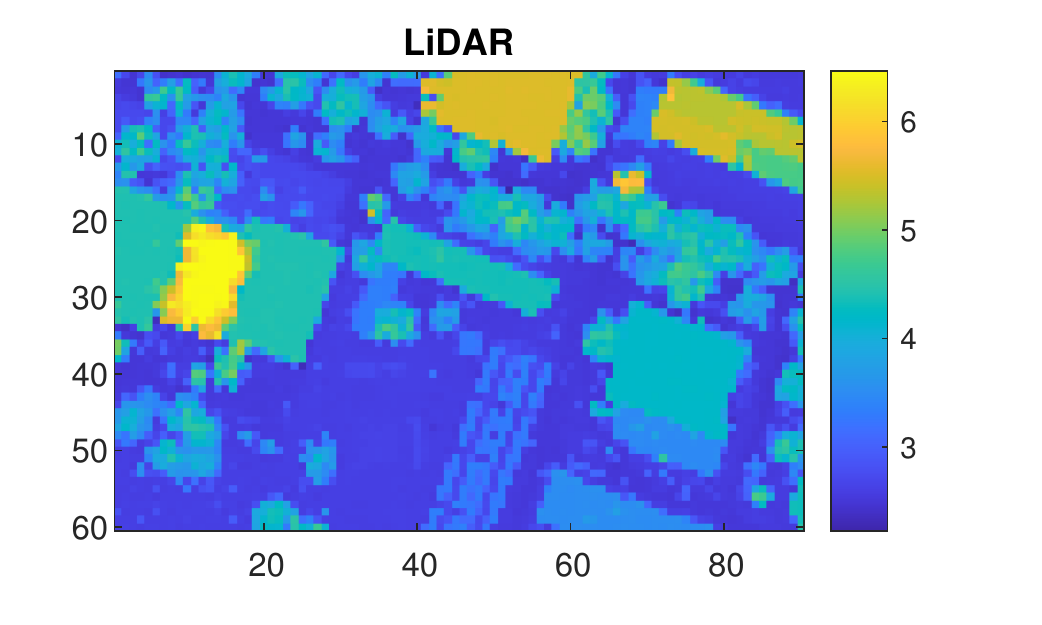}}
\subfloat[]{\includegraphics[width=0.33\textwidth]{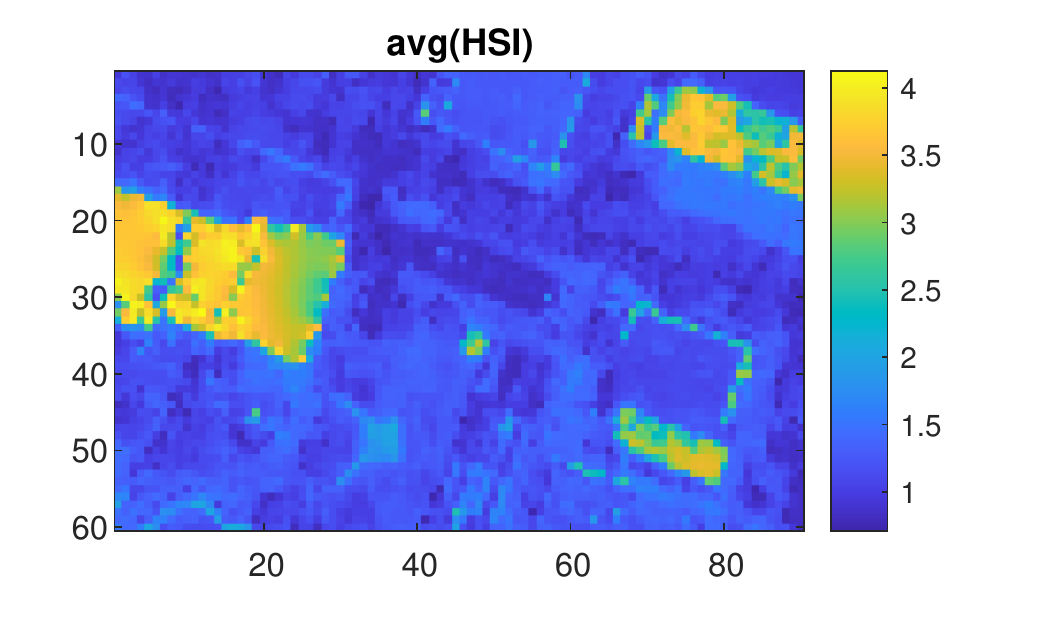}}

\subfloat[]{\includegraphics[width=0.33\textwidth]{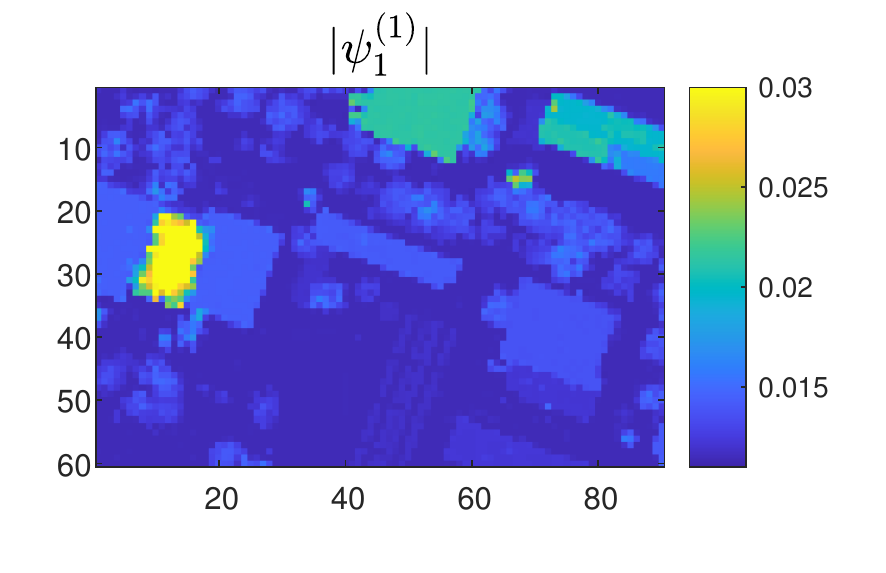}}
\subfloat[]{\includegraphics[width=0.33\textwidth]{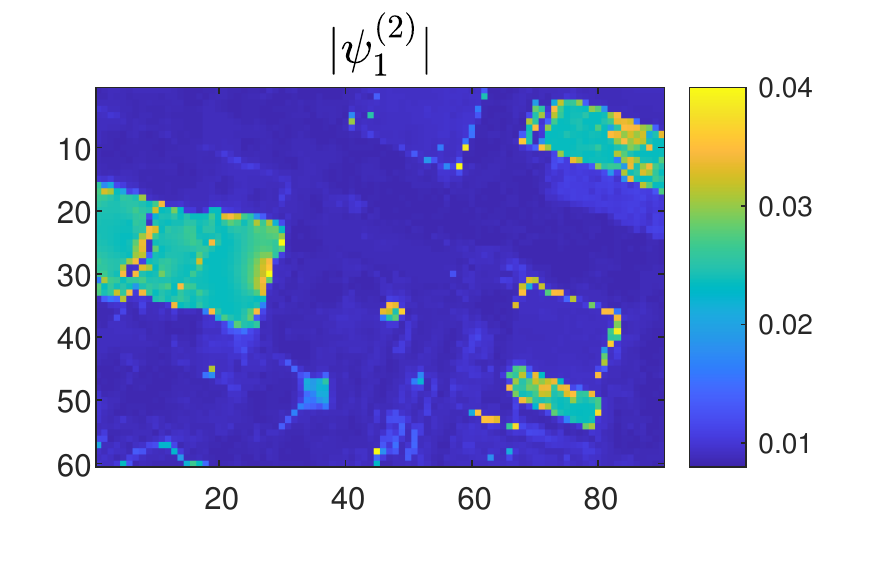}}

\subfloat[]{\includegraphics[width=0.33\textwidth]{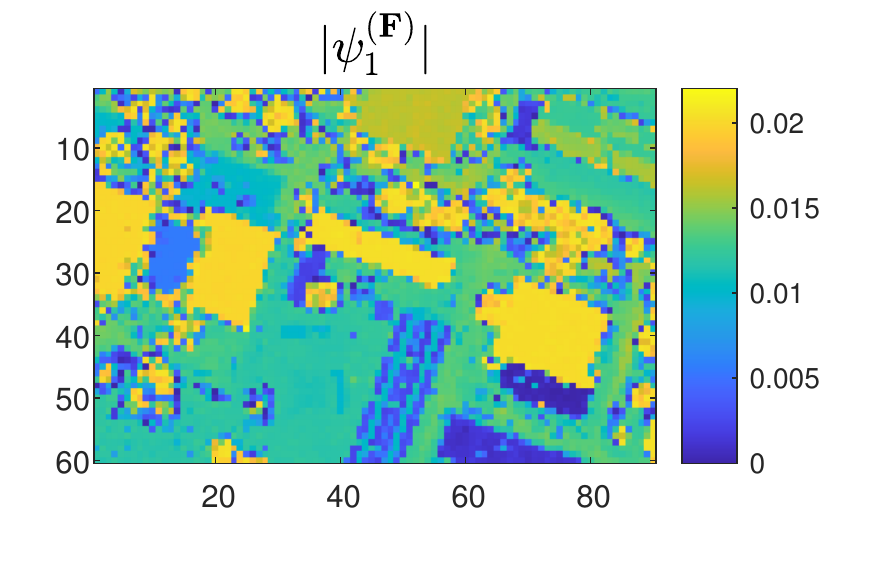}}
\subfloat[]{\includegraphics[width=0.33\textwidth]{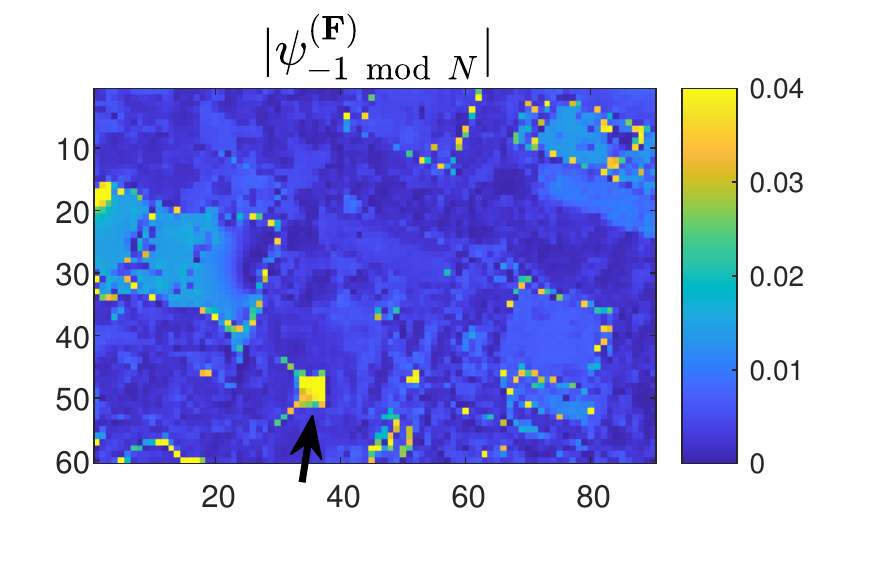}}
\subfloat[]{\includegraphics[width=0.33\textwidth]{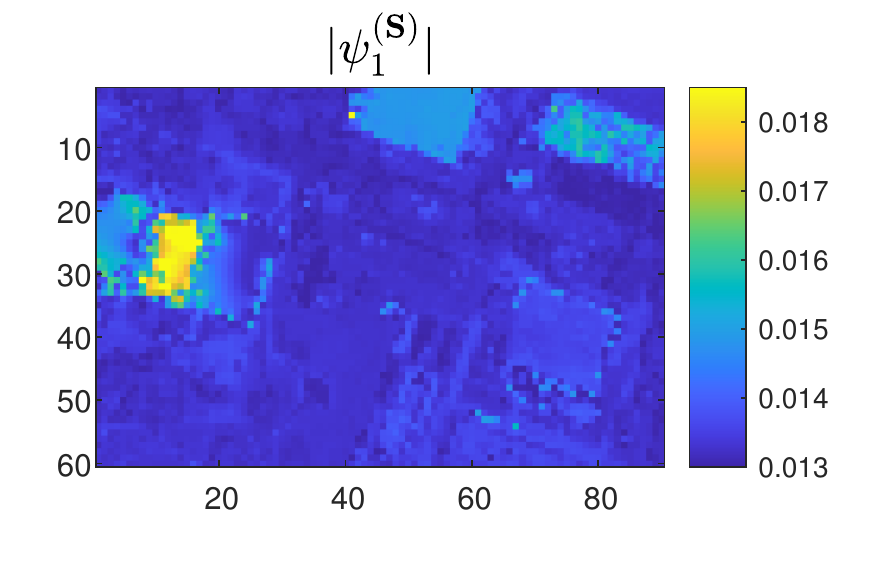}}
\caption{The two chosen image patches of (a) the LiDAR image and (b) the HSI image after preprocessing along with the leading eigenvectors of (c) $\mathbf{W}_1$ (LiDAR data), (d)  $\mathbf{W}_2$ (HSI data), (e) $\mathbf{F}$ corresponding to the largest positive eigenvalue, (f) $\mathbf{F}$ corresponding to the smallest negative eigenvalue and (g) $\mathbf{S}$.
\label{fig:HSI_LDR_2_allEV}}
\end{figure}

Note the dominant structures (with high absolute values) in the leading eigenvectors of the two modalities in Figure \ref{fig:HSI_LDR_2_allEV} (c) and (d). 
The dominant structures of the eigenvector representing the LiDAR image in plot (c) include buildings and trees, whereas in plot (d), which relates to the HSI image, only some of the building appear with high intensities and the trees are not clearly visible. 
This corresponds to the data presented in Figure \ref{fig:HSI_LDR_2_allEV} (a) and (b).
The leading eigenvector of $\mathbf{F}$ with a positive eigenvalue, presented in Figure \ref{fig:HSI_LDR_2_allEV} (e), captures buildings and trees that are expressed more dominantly in the LiDAR image compared with the HSI image. 
In addition, the structures that are dominant in both modalities appear to be less dominant in this plot. 
For example, the ``attenuated'' structures in plot (e) include the small rectangular roof part around pixels $(x,y)=(10,25)$ (where $x$ denotes the horizontal axis and $y$ denotes the vertical axis) and the building around pixels $(x,y)=(80,10)$.
Conversely, the leading eigenvector of $\mathbf{F}$ with a negative eigenvalue, presented in Figure \ref{fig:HSI_LDR_2_allEV} (f), significantly enhances a specific location, marked by a black arrow in this plot, that is clearly visible in the HSI image presented in Figure \ref{fig:HSI_LDR_2_allEV} (b) but barely visible in Figure \ref{fig:HSI_LDR_2_allEV} plot (a).
Note that this building is not represented by high pixel values in the raw HSI average image and therefore a simple subtraction between the two images will not lead to a similar emphasis of the building.

The leading eigenvector of $\mathbf{S}$, presented in Figure \ref{fig:HSI_LDR_2_allEV} (g), captures some combination of the structures that are dominant in both modalities.

To summarize this example, we showed that the operator $\mathbf{F}$ captures common components that are expressed strongly only by one of the modalities and that the sign of the eigenvalues of $\mathbf{F}$ indicates in which modality the component is stronger. 
In addition, we showed that the operator $\mathbf{S}$ captures some combination of the dominant components in both modalities.


\section{Spectral Analysis\label{sec:analysis}}

To provide theoretical justification to the proposed RMRA framework for spatiotemporal analysis presented in Section \ref{sec:rmra}, we analyze the operators $\mathbf{S}$ and $\mathbf{F}$ defined in \eqref{eq:Sr} and \eqref{eq:Ar} and show that they admit the desired properties. 
Specifically, we show that the operator $\mathbf{S}$ enhances common eigenvectors that are expressed similarly in two consecutive time frames in the sense that they have similar eigenvalues. In addition, we show that the operator $\mathbf{F}$ enhances common eigenvectors that are expressed differently in two consecutive time frames in the sense that they have different eigenvalues.





In the following theoretical analysis we focus on two cases: (i) $\mathbf{W}_1$ and $\mathbf{W}_2$ have \emph{strictly} common components, i.e., some of the eigenvectors of the two matrices are identical, and (ii) $\mathbf{W}_1$ and $\mathbf{W}_2$ have \emph{weakly} common components, i.e. some of the eigenvectors of the two matrices differ by a small perturbation.

\subsection{Strictly Common Components\label{sub:ident_eig}}

Given that some of the eigenvectors of matrices $\mathbf{W}_1$ and $\mathbf{W}_2$ are identical, we show in the following that for these identical eigenvectors, the operator $\mathbf{S}$ enhances the eigenvectors that have similar dominant eigenvalues and the operator $\mathbf{F}$ enhances the eigenvectors that have significantly different eigenvalues.

We begin by reiterating a theorem from \cite{ok19} along with its proof, which shows that the eigenvectors that are similarly expressed in both matrices, $\mathbf{W}_1$ and $\mathbf{W}_2$, correspond to the largest eigenvalues of $\mathbf{S}$.
\begin{theorem}\label{theo:S_eigs}
Consider a vector $\psi$, which is an eigenvector of both $\mathbf{W}_1$ and $\mathbf{W}_2$ with possibly different eigenvalues: $\mathbf{W}_1\psi=\lambda^{(1)}\psi$ and $\mathbf{W}_2\psi=\lambda^{(2)}\psi$.
Then, $\psi$ is also an eigenvector of $\mathbf{S}$ with the corresponding eigenvalue:
\begin{equation}
\lambda^{(\mathbf{S})}=\sqrt{\lambda^{(1)}\lambda^{(2)}}    
\end{equation}
\end{theorem}
\begin{proof}
From \eqref{eq:Sr} we have:
\begin{align*}
\mathbf{S}\psi=&\,\mathbf{W}_1^{1/2}\left(\mathbf{W}_1^{-1/2}\mathbf{W}_2\mathbf{W}_1^{-1/2}\right)^{1/2}\mathbf{W}_1^{1/2}\psi`\\
=&\,\mathbf{W}_1^{1/2}\left(\mathbf{W}_1^{-1/2}\mathbf{W}_2\mathbf{W}_1^{-1/2}\right)^{1/2}\sqrt{\lambda^{(1)}}\psi\\
=&\,\mathbf{W}_1^{1/2}\sqrt{\lambda^{(2)}/\lambda^{(1)}}\sqrt{\lambda^{(1)}}\psi\\
=&\,\sqrt{\lambda^{(1)}\lambda^{(2)}}\psi
\end{align*}
where the transition before last is due to $\mathbf{W}_1^{-1/2}\mathbf{W}^{(2)}\mathbf{W}_1^{-1/2}\psi=(\lambda^{(2)}/\lambda^{(1)})\psi$.
\end{proof}
This result implies that strictly common components that are dominant and similarly expressed in both datasets (with similar large eigenvalues) are dominant in $\mathbf{S}$ (have a large eigenvalue $\lambda^{(\mathbf{S})}$), i.e. if $\lambda^{(1)}\approx\lambda^{(2)}$ then $\lambda^{(\mathbf{S})}\approx\lambda^{(1)},\lambda^{(2)}$.

We derive a similar theoretical analysis for the operator $\mathbf{F}$.
\begin{theorem}\label{theo:A_eigs}
Consider a vector $\psi$, which is an eigenvector of both $\mathbf{W}_1$ and $\mathbf{W}_2$ with possibly different eigenvalues: $\mathbf{W}_1\psi=\lambda^{(1)}\psi$ and $\mathbf{W}_2\psi=\lambda^{(2)}\psi$.
Then $\psi$ is also an eigenvector of $\mathbf{F}$ with the corresponding eigenvalue:
\begin{equation}
    \lambda^{(\mathbf{F})}=\frac{1}{2}\sqrt{\lambda^{(1)}\lambda^{(2)}}(\log(\lambda^{(1)})-\log(\lambda^{(2)}))
\end{equation}
\end{theorem}
\begin{proof}
From \eqref{eq:Ar} we have:
\begin{align*}
\mathbf{F}\psi=&\,\mathbf{S}^{1/2}\log\left(\mathbf{S}^{-1/2}\mathbf{W}_1\mathbf{S}^{-1/2}\right)\mathbf{S}^{1/2}\psi\\
=&\,\mathbf{S}^{1/2}\log\left(\mathbf{S}^{-1/2}\mathbf{W}_1\mathbf{S}^{-1/2}\right)(\lambda^{(1)}\lambda^{(2)})^{0.25}\psi\\
=&\,\mathbf{S}^{1/2}(0.5\log(\lambda^{(1)})-0.5\log(\lambda^{(2)}))(\lambda^{(1)}\lambda^{(2)})^{0.25}\psi\\
=&\,\frac{1}{2}\sqrt{\lambda^{(1)}\lambda^{(2)}}(\log(\lambda^{(1)})-\log(\lambda^{(2)}))\psi
\end{align*}
where the transition before last is due to $\mathbf{S}^{-1/2}\mathbf{W}_1\mathbf{S}^{-1/2}\psi=\sqrt{\lambda^{(1)}/\lambda^{(2)}}\psi$ and the application of $\log$ to the matrix multiplication, which is equivalent to applying $\log$ to its eigenvalues, leading to $\log(\sqrt{\lambda^{(1)}/\lambda^{(2)}})=0.5\log(\lambda^{(1)})-0.5\log(\lambda^{(2)})$. 
\end{proof}
This result indicates that the strictly common components of the two datasets are also expressed by $\mathbf{F}$ and that their order is determined by their relative expression in each dataset.
For example, if $\psi$ is an eigenvector of both $\mathbf{W}_1$ and $\mathbf{W}_2$ and corresponds to equal eigenvalues, $\lambda^{(1)}=\lambda^{(2)}$, then this component is part of the null space of $\mathbf{F}$;
if $\lambda^{(1)}\neq\lambda^{(2)}$, then $\psi$ corresponds to a nonzero eigenvalue in $\mathbf{F}$. 
Note that $\lambda^{(\mathbf{F})}$ depends also on the multiplication by $\sqrt{\lambda^{(1)}\lambda^{(2)}}$.

Another notable result of Theorem \ref{theo:A_eigs} is that the sign of the eigenvalues of $\mathbf{F}$ indicates in which dataset their corresponding eigenvector is more dominant. 
For example, if $\psi$ is an eigenvector of both $\mathbf{W}_1$ and $\mathbf{W}_2$ that has a large corresponding eigenvalue in $\mathbf{W}_1$ but a small eigenvalue in $\mathbf{W}_2$ ($\lambda^{(1)}\gg\lambda^{(2)}$), then the corresponding eigenvalue in $\mathbf{F}$ is large and positive. 
Conversely, if $\psi$ is more dominant in $\mathbf{W}_2$, then its corresponding eigenvalue in $\mathbf{F}$ is large (in absolute value) and negative.

An example depicting these properties is presented in Subsection \ref{subsub:ills_exmp}, which demonstrates that the eigenvalues of the operators $\mathbf{S}$ and $\mathbf{F}$ are indeed equal to the expected values based on Theorem \ref{theo:S_eigs} and Theorem \ref{theo:A_eigs}.


\subsection{Weakly Common Components}

{
To further demonstrate the power of the operators $\mathbf{S}$ and $\mathbf{F}$, we provide stability analysis by investigating how a small variation of the common eigenvector affects the results. For this purpose, we make use of the concept of a pseudo-spectrum \cite{trefethen2005spectra}. While  pseudo-spectra is typically used to provide an analytic framework for investigating non-normal matrices and operators, here we apply it to symmetric matrices for the purpose of analysis of nearly but non-common eigenvectors.
We begin by recalling three equivalent definitions of the $\epsilon$-pseudo-spectrum as presented in \cite{trefethen2005spectra}.

\begin{definition}[Pseudo-spectrum]\label{def:pseudo}
Given a matrix $\mathbf{M}\in\mathbb{R}^{N\times N}$, the following definitions of the $\epsilon$-pseudo-spectrum are equivalent 
for a small $\epsilon>0$:
\begin{enumerate}
    \item $\sigma_\epsilon(\mathbf{M}) = \left\lbrace\lambda\in\mathbb{R}:\left\Vert(\lambda\mathrm{I}-\mathbf{M})^{-1}\right\Vert\geq\epsilon^{-1}\right\rbrace$
    \item $\sigma_\epsilon(\mathbf{M}) = \left\lbrace\lambda\in\mathbb{R}:\lambda\in\sigma(\mathbf{M}+\mathbf{E})\text{ for a }\mathbf{E}\text{ with }\left\Vert\mathbf{E}\right\Vert\leq\epsilon\right\rbrace$
    \item $\sigma_\epsilon(\mathbf{M}) = \left\lbrace\lambda\in\mathbb{R}:\exists v\in\mathbb{R}^N\text{ with }\left\Vert v\right\Vert_2=1\text{ s.t. }\left\Vert(\mathbf{M}-\lambda\mathrm{I})v\right\Vert_2\leq\epsilon\right\rbrace$
\end{enumerate}
where $\sigma(\mathbf{M})$ denotes the set of eigenvalues of $\mathbf{M}$, $\mathrm{I}$ denotes the identity matrix, $\left\Vert\cdot\right\Vert_2$ denotes the $\ell_2$ norm, and $\left\Vert\cdot\right\Vert$ denotes the induced operator norm. Moreover, we term a vector $v$ that adheres to definition 3 an $\epsilon$-pseudo-eigenvector.
\end{definition}

The following theorem is the counterpart of Theorem \ref{theo:S_eigs} for the case where the eigenvector is slightly perturbed.

\begin{theorem}\label{prop:pseudo_sa}

Suppose there exists an eigenpair $\lambda^{(k)}$ and $\psi^{(k)}$ of $\mathbf{W}_k$ for $k=1,2$ so that $\psi^{(1)}=\psi^{(2)}+\psi_{\epsilon_{\mathbf{S}}}$, where $\left\Vert \psi_{\epsilon_{\mathbf{S}}}\right\Vert_2\leq \frac{\sqrt{\lambda^{(2)}}}{\tilde{\lambda}^{(2)}_{max}\sqrt{\lambda^{(1)}}}\epsilon_{\mathbf{S}}$ for a small $\epsilon_{\mathbf{S}}>0$, where $\tilde{\lambda}^{(2)}_{max} = \Vert\mathbf{W}_2-\lambda^{(2)}\mathrm{I}\Vert$. Then, we have
\begin{equation}
\sqrt{\lambda^{(1)}\lambda^{(2)}}\in \sigma_{\epsilon_{\mathbf{S}}}(\mathbf{S})\,.\label{pseudospectrum of S}
\end{equation}
Specifically, we have 
\[
\left\Vert(\mathbf{S}-\sqrt{\lambda^{(1)}\lambda^{(2)}}\mathrm{I})\psi^{(1)}\right\Vert_2\leq\epsilon_{\mathbf{S}}\,.
\]
and $\psi^{(1)}$ is a corresponding $\epsilon_{\mathbf{S}}$-pseudo-eigenvector of $\mathbf{S}$.
\end{theorem}

Informally, this theorem implies that when $\psi^{(1)}$ is a slight perturbation of $\psi^{(2)}$, then $\psi^{(1)}$ is ``almost'' an eigenvector of $\mathbf{S}$ with a corresponding eigenvalue $\sqrt{\lambda^{(1)}\lambda^{(2)}}$.
Equivalently, $\psi^{(2)}$ can also be shown to be an $\epsilon_{\mathbf{S}}$-pseudo-eigenvector of $\mathbf{S}$ with the same corresponding eigenvalue, which suggests that the operator $\mathbf{S}$ is ``stable'' to finite perturbations.
We remark that since $\|\mathbf{W}_2\|=1$, we have that $\tilde{\lambda}^{(2)}_{max}=\max(1-\lambda^{(2)}, \lambda^{(2)})\in[0.5,1)$, guaranteeing that the perturbation of the eigenvector is small.
We also remark that our numerical study shows that the bound for the $\epsilon$-pseudo-eigenvalues and $\epsilon$-pseudo-eigenvectors of $\mathbf{S}$ is tight.

\begin{proof}
By Proposition \ref{prop_app:sa_equiv_forms} (see Appendix \ref{app:add_state}), we have
\begin{equation}
\mathbf{S}=\mathbf{W}_1^{1/2}\left(\mathbf{W}_1^{-1/2}\mathbf{W}_2\mathbf{W}_1^{-1/2}\right)^{1/2}\mathbf{W}_1^{1/2}=(\mathbf{W}_2\mathbf{W}_1^{-1})^{1/2}\mathbf{W}_1.\label{eq:equivexp}
\end{equation}
Since $\psi^{(1)}$ is an eigenvector of $\mathbf{W}_1$ with an eigenvalue $\lambda^{(1)}$, we have 
\begin{align}
    \mathbf{S}\psi^{(1)} =  \left(\mathbf{W}_2\mathbf{W}_1^{-1}\right)^{1/2}\mathbf{W}_1\psi^{(1)}
    = \lambda^{(1)}\left(\mathbf{W}_2\mathbf{W}_1^{-1}\right)^{1/2}\psi^{(1)}\label{eq:pseudo_s_psi}.
\end{align}
Therefore, it is sufficient to show that $\psi^{(1)}$ is an $\epsilon$-pseudo-eigenvector of $\left(\mathbf{W}_2\mathbf{W}_1^{-1}\right)^{1/2}$. 
By a direct expansion, we have
\begin{align}
\mathbf{W}_2\mathbf{W}_1^{-1}\psi^{(1)}  = & \, \frac{1}{\lambda^{(1)}}\mathbf{W}_2 \psi^{(1)}
 = \frac{1}{\lambda^{(1)}}\mathbf{W}_2 (\psi^{(2)}+\psi_{\epsilon_{\mathbf{S}}})\nonumber\\
 = & \, \frac{\lambda^{(2)}}{\lambda^{(1)}}\psi^{(2)}+\frac{1}{\lambda^{(1)}}\mathbf{W}_2 \psi_{\epsilon_{\mathbf{S}}} =  \frac{\lambda^{(2)}}{\lambda^{(1)}}\psi^{(1)}+\frac{1}{\lambda^{(1)}}(\mathbf{W}_2-\lambda^{(2)}\mathrm{I}) \psi_{\epsilon_{\mathbf{S}}}\,,\label{eq:pseudo_s_1}
\end{align}
where the last transition is obtained by replacing $\psi^{(2)}$ with $\psi^{(1)}-\psi_{\epsilon_{\mathbf{S}}}$.
By reorganizing the elements in \eqref{eq:pseudo_s_1}, we have
\begin{align}
&\left(\mathbf{W}_2\mathbf{W}_1^{-1}-\frac{\lambda^{(2)}}{\lambda^{(1)}}\mathrm{I}\right)\psi^{(1)} = \frac{1}{\lambda^{(1)}}(\mathbf{W}_2-\lambda^{(2)}\mathrm{I}) \psi_{\epsilon_{\mathbf{S}}}\,,
\end{align}
and applying the $\ell_2$ norm leads to:
\begin{align}
&\left\Vert\left(\mathbf{W}_2\mathbf{W}_1^{-1}-\frac{\lambda^{(2)}}{\lambda^{(1)}}\mathrm{I}\right)\psi^{(1)}\right\Vert_2  =  \left\Vert\frac{1}{\lambda^{(1)}}(\mathbf{W}_2-\lambda^{(2)}\mathrm{I}) \psi_{\epsilon_{\mathbf{S}}}\right\Vert_2
 \leq  \frac{1}{\lambda^{(1)}}\left\Vert \mathbf{W}_2-\lambda^{(2)}\mathrm{I}\right\Vert\left\Vert \psi_{\epsilon_{\mathbf{S}}}\right\Vert_2\label{eq:pseudo_s_2}\\
&\qquad\qquad \leq  \,\frac{1}{\lambda^{(1)}}\left\Vert\mathbf{W}_2-\lambda^{(2)}\mathrm{I}\right\Vert\frac{\sqrt{\lambda^{(2)}}}{\tilde{\lambda}^{(2)}_{max}\sqrt{\lambda^{(1)}}}\epsilon_{\mathbf{S}} =  \frac{\tilde{\lambda}^{(2)}_{max}}{\lambda^{(1)}}\frac{\sqrt{\lambda^{(2)}}}{\tilde{\lambda}^{(2)}_{max}\sqrt{\lambda^{(1)}}}\epsilon_{\mathbf{S}}=\frac{\sqrt{\lambda^{(2)}}}{\lambda^{(1)}\sqrt{\lambda^{(1)}}}\epsilon_{\mathbf{S}}\,.\nonumber
\end{align}
This derivation shows that $\psi^{(1)}$ is a pseudo-eigenvector of $\mathbf{W}_2\mathbf{W}_1^{-1}$ with a pseudo-eigenvalue $\frac{\lambda^{(2)}}{\lambda^{(1)}}$, i.e. $\left\Vert\left(\mathbf{W}_2\mathbf{W}_1^{-1}-\frac{\lambda^{(2)}}{\lambda^{(1)}}\mathrm{I}\right)\psi^{(1)}\right\Vert_2\leq\frac{\sqrt{\lambda^{(2)}}}{\lambda^{(1)}\sqrt{\lambda^{(1)}}}\epsilon_{\mathbf{S}}$.
{Thus, by definition (see Proposition \ref{prop_app:pseudo_explicit} in Appendix \ref{app:add_state})}, there exists a matrix $\mathbf{E}$ such that $(\mathbf{W}_2\mathbf{W}_1^{-1}+\mathbf{E})\psi^{(1)}=\frac{\lambda^{(2)}}{\lambda^{(1)}}\psi^{(1)}$ and $\left\Vert\mathbf{E}\right\Vert\leq\frac{\sqrt{\lambda^{(2)}}}{\lambda^{(1)}\sqrt{\lambda^{(1)}}}\epsilon_{\mathbf{S}}$. 
Therefore, we can write the following:
\begin{align}
\mathbf{E}\psi^{(1)}  = &\, -\left(\mathbf{W}_2\mathbf{W}_1^{-1}-\frac{\lambda^{(2)}}{\lambda^{(1)}}\mathrm{I}\right)\psi^{(1)}\label{eq:pseudo_s_3}\\
 = &\, -\left((\mathbf{W}_2\mathbf{W}_1^{-1})^{1/2}+\sqrt{\frac{\lambda^{(2)}}{\lambda^{(1)}}}\mathrm{I}\right)\left((\mathbf{W}_2\mathbf{W}_1^{-1})^{1/2}-\sqrt{\frac{\lambda^{(2)}}{\lambda^{(1)}}}\mathrm{I}\right)\psi^{(1)}.\nonumber
\end{align}
Since $(\mathbf{W}_2\mathbf{W}_1^{-1})^{1/2}+\sqrt{\lambda^{(2)}/\lambda^{(1)}}\mathrm{I}$ is positive definite, \eqref{eq:pseudo_s_3} leads to
\begin{equation}
\left((\mathbf{W}_2\mathbf{W}_1^{-1})^{1/2}-\sqrt{\frac{\lambda^{(2)}}{\lambda^{(1)}}}\mathrm{I}\right)\psi^{(1)}= -    \left((\mathbf{W}_2\mathbf{W}_1^{-1})^{1/2}+\sqrt{\frac{\lambda^{(2)}}{\lambda^{(1)}}}\mathrm{I}\right)^{-1}\mathbf{E}\psi^{(1)}. 
\label{eq:pseudo_s_sqrt_eq}
\end{equation}
With the above preparation, we have 
\begin{align}
\left(\mathbf{S}-\sqrt{\lambda^{(1)}\lambda^{(2)}}\mathrm{I}\right)\psi^{(1)} \underset{\eqref{eq:pseudo_s_psi}}{=} & \lambda^{(1)}(\mathbf{W}_2\mathbf{W}_1^{-1})^{1/2}\psi^{(1)}-\sqrt{\lambda^{(1)}\lambda^{(2)}}\psi^{(1)} \nonumber \\ 
= \,& \lambda^{(1)}\left((\mathbf{W}_2\mathbf{W}_1^{-1})^{1/2}-\sqrt{\frac{\lambda^{(2)}}{\lambda^{(1)}}}\mathrm{I}\right)\psi^{(1)} \nonumber \\
\underset{\eqref{eq:pseudo_s_sqrt_eq}}{=} & -\lambda^{(1)}\left((\mathbf{W}_2\mathbf{W}_1^{-1})^{1/2}+\sqrt{\frac{\lambda^{(2)}}{\lambda^{(1)}}}\mathrm{I}\right)^{-1}\mathbf{E}\psi^{(1)}.\label{eq:pseudo_s_eq}
\end{align}
Taking the norm gives
\begin{align}
\left\Vert \left(\mathbf{S}-\sqrt{\lambda^{(1)}\lambda^{(2)}}\mathrm{I}\right)\psi^{(1)} \right\Vert_2 \le \, &\lambda^{(1)}\left\Vert\left((\mathbf{W}_2\mathbf{W}_1^{-1})^{1/2}+\sqrt{\frac{\lambda^{(2)}}{\lambda^{(1)}}}\mathrm{I}\right)^{-1}\right\Vert\left\Vert\mathbf{E}\right\Vert\\
\leq \,&  \frac{\sqrt{\lambda^{(2)}/\lambda^{(1)}}}{\sigma_{\min}\left((\mathbf{W}_2\mathbf{W}_1^{-1})^{1/2}\right)+\sqrt{\lambda^{(2)}/\lambda^{(1)}}}\epsilon_{\mathbf{S}} \nonumber \\
\leq \,& \frac{\sqrt{\lambda^{(2)}/\lambda^{(1)}}}{\sqrt{\lambda^{(2)}/\lambda^{(1)}}}\epsilon_{\mathbf{S}}=\epsilon_{\mathbf{S}} \nonumber, \label{eq:pseudo_s_sqrt_norm}
\end{align}
where $\sigma_{\min}$ denotes the minimum eigenvalue. 
Thus, $\sqrt{\lambda^{(1)}\lambda^{(2)}}$ is an $\epsilon_{\mathbf{S}}$-pseudo-eigenvalue of $\mathbf{S}$, where $\psi^{(1)}$ is a corresponding $\epsilon_{\mathbf{S}}$-pseudo-eigenvector.
\end{proof}

Note that the above proof shows that $\psi^{(1)}$ is a pseudo-eigenvector of $\mathbf{S}$, when $\mathbf{S}$ is defined as the midpoint of the geodesic curve connecting $\mathbf{W}_1$ and $\mathbf{W}_2$ (by setting $p=0.5$).
However, due to the decomposition in \eqref{eq:pseudo_s_3}, the proof is not compatible with definitions of $\mathbf{S}$ at other points $p \in (0,1)$ along the geodesic path.
For such cases, a different proof is required, specifically, without using the algebraic relationship in \eqref{eq:pseudo_s_3} that leads to \eqref{eq:pseudo_s_sqrt_eq}. In the following statement, which is the counterpart of Theorem \ref{theo:A_eigs} for the case where the eigenvector is not strictly common, we control the ``pseudo'' part by a straightforward perturbation argument.

\begin{theorem}\label{prop:pseudo_apart}
For $k=1,2$, consider the eigendecomposition $\mathbf{W}_k=\mathbf{U}_k \mathbf{L}_k \mathbf{U}_k^\top\in \mathbb{R}^{N\times N}$, where $\mathbf{L}_k:=\text{diag}(\lambda_1^{(k)},\ldots,\lambda_N^{(k)})$ so that $\lambda_1^{(1)}\geq\ldots\geq\lambda_N^{(1)}$ and $\mathbf{U}_k:=\begin{bmatrix}\psi_1^{(k)}&\ldots&\psi_N^{(k)}\end{bmatrix}\in \mathcal{V}_{N,N}$. 
Assume the above eigendecomposition satisfies $\mathbf{U}_2=\mathbf{U}_1+\epsilon \mathbf{A}$, where $\|\mathbf{A}\|=1$, $\epsilon>0$ is a small constant, and $c^{-1}\leq \ell_i:= \lambda_i^{(2)}/\lambda_i^{(1)}\leq c$ for some constant $c\geq1$ for all $i=1,\ldots,N$. For any $i=1,\ldots,N$, denote the spectral gap $\gamma_i:=\min_{k,\,\ell_k\neq \ell_i}|\ell_i-\ell_k|$.

Fix $j$. Then, for the $j$-th eigenpair of $\mathbf{W}_1$, when $\epsilon$ is sufficiently small, we have
\[
\left\Vert\left(\mathbf{F}-0.5\sqrt{\lambda_j^{(1)}\lambda_j^{(2)}}\log\left(\frac{\lambda_j^{(1)}}{\lambda_j^{(2)}}\right)\mathrm{I}\right)\psi_j^{(1)}\right\Vert_2=O(\epsilon)\,,
\]
where the implied constant depends on 
$\frac{\sqrt{c}\ln c}{\min_i\left(\gamma_i\sqrt{\lambda_i^{(1)}}\right)}$.

\end{theorem}

\begin{proof}  

By Proposition \ref{prop_app:sa_equiv_forms} (see Appendix \ref{app:add_state}), we rewrite the operator $\mathbf{F}=\mathbf{S}^{1/2}\log\left(\mathbf{S}^{-1/2}\mathbf{W}_1\mathbf{S}^{-1/2}\right)\mathbf{S}^{1/2}$ as $\mathbf{F} = \log\left(\mathbf{W}_1\mathbf{S}^{-1}\right)\mathbf{S}$. 
Denote 
\[
\hat{\lambda}_{\mathbf{F}}=0.5\sqrt{\lambda_j^{(1)}\lambda_j^{(2)}}\log\left(\frac{\lambda_j^{(1)}}{\lambda_j^{(2)}}\right)\,.
\]
{By Theorem \ref{prop:pseudo_sa}, there exists $\mathbf{E}_{\mathbf{S}}$ with a sufficiently small norm, such that
\[
    \mathbf{S}\psi_j^{(1)} = \left(\sqrt{\lambda_j^{(1)}\lambda_j^{(2)}}\mathrm{I}-\mathbf{E}_{\mathbf{S}}\right)\psi_j^{(1)}, \
\]
and therefore, we have}
\begin{align}
\left(\mathbf{F}-\hat{\lambda}_{\mathbf{F}}\mathrm{I}\right)\psi_j^{(1)} = &\, \left(\log\left(\mathbf{W}_1\mathbf{S}^{-1}\right)\mathbf{S}-\hat{\lambda}_{\mathbf{F}}\mathrm{I}\right)\psi_j^{(1)}
\nonumber\\
=&\,\sqrt{\lambda_j^{(1)}\lambda_j^{(2)}}\left(\log\left(\mathbf{W}_1\mathbf{S}^{-1}\right)-\frac{\hat{\lambda}_{\mathbf{F}}}{\sqrt{\lambda_j^{(1)}\lambda_j^{(2)}}}\mathrm{I}\right)\psi_j^{(1)}-\log\left(\mathbf{W}_1\mathbf{S}^{-1}\right)\mathbf{E}_{\mathbf{S}}\psi_j^{(1)}.
\label{eq:pseudo_a_1}
\end{align}

Since both $\mathbf{W}_1$ and $\mathbf{W}_2$ are positive definite, they are invertible and also $\left(\mathbf{W}_2\mathbf{W}_1^{-1}\right)^{1/2}$ is invertible. 
Therefore, by \eqref{eq:equivexp}, $\mathbf{W}_1\mathbf{S}^{-1}=\mathbf{W}_1\mathbf{W}_1^{-1}\left(\mathbf{W}_2\mathbf{W}_1^{-1}\right)^{-1/2}=\left(\mathbf{W}_2\mathbf{W}_1^{-1}\right)^{-1/2}$. 
Substituting this relationship into \eqref{eq:pseudo_a_1} yields
\begin{align}
\left(\mathbf{F}-\hat{\lambda}_{\mathbf{F}}\mathrm{I}\right)\psi^{(1)}_j
= - \frac{\sqrt{\lambda_j^{(1)}\lambda_j^{(2)}}}{2}
\left(\log\left(\mathbf{W}_2\mathbf{W}_1^{-1}\right)-\log\left(\frac{\lambda_j^{(2)}}{\lambda_j^{(1)}}\right)\mathrm{I}\right)\psi_j^{(1)}-\frac{1}{2}\log\left(\mathbf{W}_2\mathbf{W}_1^{-1}\right)\mathbf{E}_{\mathbf{S}}\psi_j^{(1)}\,. \label{eq:pseudo_a_2}
\end{align}
Now we control the right hand side term by term. 
Recall that for any analytic function $f$ over an open set in $\mathbb{R}$ that contains the spectrum of $\mathbf{W}_2\mathbf{W}_1^{-1}$, we can define $f(\mathbf{W}_2\mathbf{W}_1^{-1})$.
Since $\mathbf{W}_2\mathbf{W}_1^{-1}$ and $\mathbf{W}_1^{-1/2}\mathbf{W}_2 \mathbf{W}_1^{-1/2}$ are similar, we have
\begin{equation}
f(\mathbf{W}_2\mathbf{W}_1^{-1})=\mathbf{W}_1^{1/2}f\left(\mathbf{W}_1^{-1/2}\mathbf{W}_2 \mathbf{W}_1^{-1/2}\right)\mathbf{W}_1^{-1/2}\,,\label{eq:pseudo_a_6}
\end{equation}
and hence
\begin{align}
f\left(\mathbf{W}_1^{-1/2}\mathbf{W}_2\mathbf{W}_1^{-1/2}\right)\psi_j^{(1)}= \mathbf{W}_1^{-1/2}f(\mathbf{W}_2\mathbf{W}_1^{-1})\mathbf{W}_1^{1/2}\psi_j^{(1)}=\sqrt{\lambda_j^{(1)}}\mathbf{W}_1^{-1/2}f(\mathbf{W}_2\mathbf{W}_1^{-1})\psi_j^{(1)}.\label{eq:pseudo_a_6q}
\end{align}

Let $\mu_i$ and $v_i$ denote the eigenvalues and eigenvectors of the matrix $\mathbf{W}_1^{-1/2}\mathbf{W}_2\mathbf{W}_1^{-1/2}$, respectively, for $i=1,\dots,N$.  %
Since $\{\psi_j^{(1)}\}_{j=1}^N$ and $\{v_j\}_{j=1}^N$ are both orthonormal bases of $\mathbb{R}^N$, we have $\psi_j^{(1)}=\sum_i \alpha_{ji} v_i$, where $\alpha_{ji}\in\mathbb{R}$ and $\sum_i \alpha_{ji}^2=1$ for all $j$.
By \eqref{eq:pseudo_a_6q}, we have 

\begin{align}
&\frac{\sqrt{\lambda_j^{(1)}\lambda_j^{(2)}}}{2}\left(\log(\mathbf{W}_2\mathbf{W}_1^{-1})-\log\left(\frac{\lambda_j^{(2)}}{\lambda_j^{(1)}}\right)\mathrm{I}\right)\psi_j^{(1)}\nonumber\\
=&\,\frac{\sqrt{\lambda_j^{(2)}}}{2}\mathbf{W}_1^{1/2}\left(\log(\mathbf{W}_1^{-1/2}\mathbf{W}_2\mathbf{W}_1^{-1/2})-\log\left(\frac{\lambda_j^{(2)}}{\lambda_j^{(1)}}\right)\mathrm{I}\right)\psi_j^{(1)}\nonumber\\
=&\,\frac{\sqrt{\lambda_j^{(2)}}}{2}\mathbf{W}_1^{1/2}\sum_{i=1}^N\left(\log(\mathbf{W}_1^{-1/2}\mathbf{W}_2\mathbf{W}_1^{-1/2})-\log\left(\frac{\lambda_j^{(2)}}{\lambda_j^{(1)}}\right)\mathrm{I}\right)\alpha_{ji}v_i.
\end{align}
Using the fact that 
\begin{eqnarray}
 \left(\log\left(\mathbf{W}_1^{-1/2}\mathbf{W}_2\mathbf{W}_1^{-1/2}\right)-\log\left(\frac{\lambda_j^{(2)}}{\lambda_j^{(1)}}\right)\mathrm{I}\right)v_i = \left(\log\mu_i-\log\left(\frac{\lambda_j^{(2)}}{\lambda_j^{(1)}}\right)\right)v_i\,,\label{eq:pseudo_a_5}
\end{eqnarray}
yields
\begin{align}
&\frac{\sqrt{\lambda_j^{(2)}}}{2}\mathbf{W}_1^{1/2}\sum_{i=1}^N\left(\log(\mathbf{W}_1^{-1/2}\mathbf{W}_2\mathbf{W}_1^{-1/2})-\log\left(\frac{\lambda_j^{(2)}}{\lambda_j^{(1)}}\right)\mathrm{I}\right)\alpha_{ji}v_i\nonumber\\
=&\,\frac{\sqrt{\lambda_j^{(2)}}}{2}\mathbf{W}_1^{1/2}\sum_{i=1}^N\left(\log\mu_i-\log\left(\frac{\lambda_j^{(2)}}{\lambda_j^{(1)}}\right)\right)\alpha_{ji}v_i\label{eq to control 00}.
\end{align}
%
Therefore, the squared $L^2$ norm of the first term in the right hand side of \eqref{eq:pseudo_a_2} becomes
\begin{align}
&\left\| \frac{\sqrt{\lambda_j^{(1)}\lambda_j^{(2)}}}{2}\left(\log(\mathbf{W}_2\mathbf{W}_1^{-1})-\log\left(\frac{\lambda_j^{(2)}}{\lambda_j^{(1)}}\right)\mathrm{I}\right)\psi_j^{(1)}\right\|^2_2\nonumber\\
=&\left\|\frac{\sqrt{\lambda_j^{(2)}}}{2}\mathbf{W}_1^{1/2}\sum_{i=1}^N\left(\log\mu_i-\log\left(\frac{\lambda_j^{(2)}}{\lambda_j^{(1)}}\right)\right)\alpha_{ji}v_i\right\|^2_2
\leq \frac{\lambda_j^{(2)}}{4}\left\|\mathbf{W}_1^{1/2}\right\|^2\left\|\sum_{i=1}^N\left(\log\mu_i-\log\left(\frac{\lambda_j^{(2)}}{\lambda_j^{(1)}}\right)\right)\alpha_{ji}v_i\right\|^2_2\nonumber\\
=&\frac{\lambda_j^{(2)}}{4}  \sum_{i=1}^N\alpha_{ji}^2 
\left(\log\mu_i-\log\left(\frac{\lambda_j^{(2)}}{\lambda_j^{(1)}}\right)\right)^2\, ,\label{eq to control 1}
\end{align}
where we use the fact that $\{v_j\}$ form an orthonormal basis and that the operator norm $\left\|\mathbf{W}_1^{1/2}\right\|^2=1$, since $\mathbf{W}_1$ is normalized.

Next, as in \eqref{eq:pseudo_s_eq}, we set
\begin{align}
\mathbf{E}_{\mathbf{S}}\psi_j^{(1)}=
\lambda_j^{(1)}\left((\mathbf{W}_2\mathbf{W}_1^{-1})^{1/2}+\sqrt{\frac{\lambda_j^{(2)}}{\lambda_j^{(1)}}}\mathrm{I}\right)^{-1}\mathbf{E}\psi_j^{(1)} 
\underset{\eqref{eq:pseudo_s_sqrt_eq}}{=}-\lambda_j^{(1)}\left((\mathbf{W}_2\mathbf{W}_1^{-1})^{1/2}-\sqrt{\frac{\lambda_j^{(2)}}{\lambda_j^{(1)}}}\mathrm{I}\right)\psi_j^{(1)}
\end{align}
for some $\mathbf{E}$ with a sufficiently small norm.
Since $f(x)=\log(x)\sqrt{x}$ is analytic over an open set that contains the spectrum of $\mathbf{W}_1\mathbf{W}_2^{-1}$, by the same argument as that for \eqref{eq to control 00}, we have 
\begin{align}
\log\left(\mathbf{W}_2\mathbf{W}_1^{-1}\right)\mathbf{E}_{\mathbf{S}}\psi_j^{(1)}&\,=-\lambda_j^{(1)}\log\left(\mathbf{W}_2\mathbf{W}_1^{-1}\right)\left((\mathbf{W}_2\mathbf{W}_1^{-1})^{1/2}-\sqrt{\frac{\lambda_j^{(2)}}{\lambda_j^{(1)}}}\mathrm{I}\right)\psi_j^{(1)}\nonumber\\
&\,=-\sqrt{\lambda_j^{(1)}}  \mathbf{W}_1^{1/2}\sum_{i=1}^N\alpha_{ji}\log(\mu_i)\left(\sqrt{\mu_i}-\sqrt{\frac{\lambda_j^{(2)}}{\lambda_j^{(1)}}}\right)v_i\,,\nonumber
\end{align}
and hence the squared $L^2$ norm of the second term in the right hand side of \eqref{eq:pseudo_a_2} becomes
\begin{align}
\left\|\log\left(\mathbf{W}_2\mathbf{W}_1^{-1}\right)\mathbf{E}_{\mathbf{S}}\psi_j^{(1)}\right\|_2^2 \leq \lambda_j^{(1)} \sum_{i=1}^N  
\alpha^2_{ji}(\log \mu_i)^2\left(\sqrt{\mu_i}-\sqrt{\frac{\lambda_j^{(2)}}{\lambda_j^{(1)}}}\right)^2\,,\label{eq to control 2}
\end{align}
where we again use the fact that $\{v_j\}$ form an orthonormal basis and that the operator norm $\left\|\mathbf{W}_1^{1/2}\right\|^2=1$.
To finish the proof, we control $\alpha_{ji}$ and the relationship between $\mu_i$ and $\frac{\lambda_j^{(2)}}{\lambda_j^{(1)}}$ in \eqref{eq to control 1} and \eqref{eq to control 2} using matrix perturbation theory.
By a direct expansion, we have
\begin{align}
\mathbf{W}_1^{-1/2}\mathbf{W}_2\mathbf{W}_1^{-1/2}=\mathbf{U}_1 \mathbf{L}_2 \mathbf{L}_1^{-1} \mathbf{U}_1^\top+\epsilon \mathbf{B}\,,
\end{align}
where $\mathbf{B}=\mathbf{U}_1 \mathbf{L}_1^{-1/2} \mathbf{U}_1^\top(\mathbf{A}\mathbf{L}_2\mathbf{U}_1^\top+\mathbf{U}_1\mathbf{L}_2\mathbf{A}^\top)\mathbf{U}_1\mathbf{L}_1^{-1/2}\mathbf{U}_1^\top+\epsilon \mathbf{U}_1\mathbf{L}_1^{-1/2}\mathbf{U}_1^\top \mathbf{A} \mathbf{L}_2 \mathbf{A}^\top \mathbf{U}_1 \mathbf{L}_1^{-1/2}\mathbf{U}_1^\top$. 
To simplify the notation, we assume that the diagonal entries of $\mathbf{L}_2\mathbf{L}_1^{-1}=\text{diag}\left(\frac{\lambda_1^{(2)}}{\lambda_1^{(1)}},\ldots,\frac{\lambda_N^{(2)}}{\lambda_N^{(1)}}\right)=\text{diag}\left(\ell_1,\ldots,\ell_N\right)$ are all distinct, or the following argument could be carried out with eigenprojections. Thus, by a standard perturbation argument, when $\epsilon$ is sufficiently small, we have 
\begin{equation}
v_i=\psi_i^{(1)}+\epsilon\sum_{k\neq i}\frac{\langle \psi_k^{(1)}, \mathbf{B}\psi_i^{(1)}\rangle}{\lambda_i^{(2)}/\lambda_i^{(1)}-\lambda_k^{(2)}/\lambda_k^{(1)}}\psi_k^{(1)}+O(\epsilon^2),\ \ 
\mu_i=\frac{\lambda_i^{(2)}}{\lambda_i^{(1)}}+\epsilon \langle \psi_i^{(1)}, \mathbf{B}\psi_i^{(1)}\rangle +O(\epsilon^2)\label{relationship vi and psii}
\end{equation}
for each $i=1,\ldots,N$.
Note that
\[
I_N=\mathbf{U}_2^\top\mathbf{U}_2=(\mathbf{U}_1+\epsilon\mathbf{A})^\top(\mathbf{U}_1+\epsilon\mathbf{A})=I_N+\epsilon(\mathbf{A}^\top\mathbf{U}_1+\mathbf{U}_1^\top\mathbf{A}) +\epsilon^2 \mathbf{A}^\top \mathbf{A}\,,
\]
so we have
\begin{equation}
\mathbf{A}^\top\mathbf{U}_1=-\mathbf{U}_1^\top\mathbf{A} -\epsilon \mathbf{A}^\top \mathbf{A}\,.\label{AU=-UA-eps AA}
\end{equation}
Thus, by a direct expansion with the definition of $\mathbf{B}$, we have 
\begin{align}
\langle \psi_k^{(1)}, \mathbf{B}\psi_i^{(1)}\rangle
=&\,(\lambda_i^{(1)}\lambda_k^{(1)})^{-1/2}e_k^\top \mathbf{U}_1^\top(\mathbf{A}\mathbf{L}_2\mathbf{U}_1^\top+\mathbf{U}_1\mathbf{L}_2\mathbf{A}^\top)\mathbf{U}_1e_i +O(\epsilon)\nonumber\\
=&\, (\lambda_i^{(1)}\lambda_k^{(1)})^{-1/2} (\lambda_i^{(2)}e_k^\top \mathbf{U}_1^\top\mathbf{A} e_i + \lambda_k^{(2)}e_k^\top \mathbf{A}^\top\mathbf{U}_1 e_i)
+O(\epsilon)\nonumber\\
=&\,(\lambda_i^{(1)}\lambda_k^{(1)})^{-1/2}\left((\lambda_i^{(2)}-\lambda_k^{(2)})e_k^\top\mathbf{U}_1^\top\mathbf{A}e_i-\epsilon \lambda_k^{(2)} e_k\mathbf{A}^\top\mathbf{A}e_i\right)+(\lambda_i^{(1)}\lambda_k^{(1)})^{-1/2}O(\epsilon)\nonumber\\
=&\,(\lambda_i^{(1)}\lambda_k^{(1)})^{-1/2}(\lambda_i^{(2)}-\lambda_k^{(2)})e_k^\top\mathbf{U}_1^\top\mathbf{A}e_i+O(\epsilon)\,,\nonumber
\end{align}
where the equality before last 
comes from \eqref{AU=-UA-eps AA}, the last equality is due to $|e_k^\top \mathbf{A}^\top \mathbf{A}e_i|\leq 1$, since $\|A\|=1$, and the constant in this derivation depends on $(\lambda_i^{(1)}\lambda_k^{(1)})^{-1/2}$.
Thus, for the $j$-th eigenpair we are concerned with, we have $\mu_j=\frac{\lambda_j^{(2)}}{\lambda_j^{(1)}}+O(\epsilon^2)$, since $\langle \psi_j^{(1)}, \mathbf{B}\psi_j^{(1)}\rangle=O(\epsilon)$.
For the eigenvector, when $k\neq i$,
we have
\begin{align}
\left|\frac{\langle \psi_k^{(1)}, \mathbf{B}\psi_i^{(1)}\rangle}{\lambda_i^{(2)}/\lambda_i^{(1)}-\lambda_k^{(2)}/\lambda_k^{(1)}}\right|
\leq &\,\frac{(\lambda_i^{(1)}\lambda_k^{(1)})^{-1/2}\left|\lambda_i^{(2)}-\lambda_k^{(2)}\right|}{\left|\lambda_i^{(2)}/\lambda_i^{(1)}-\lambda_k^{(2)}/\lambda_k^{(1)}\right|}\left|e_k^\top \mathbf{U}_1^\top\mathbf{A}e_i\right|+O(\epsilon)\nonumber\\
\leq &\,\frac{1}{\gamma_i} \frac{\left|(\lambda_i^{(2)}-\lambda_k^{(2)})\right|}{\sqrt{\lambda_i^{(1)}\lambda_k^{(1)}}} \left|e_k^\top \mathbf{U}_1^\top\mathbf{A}e_i\right|+O(\epsilon)
\nonumber\,,
\end{align}
where the constant depends on $\frac{1}{\sqrt{\lambda_i^{(1)}\lambda_k^{(1)}} \gamma_i}$.

By \eqref{relationship vi and psii} we have for $j\neq i$:
\begin{align}
	\alpha_{ji} &= \langle \psi_j^{(1)}, v_i \rangle = \langle \psi_j^{(1)}, \psi_i^{(1)}+\epsilon\sum_{k\neq i}\frac{\langle \psi_k^{(1)}, \mathbf{B}\psi_i^{(1)}\rangle}{\lambda_i^{(2)}/\lambda_i^{(1)}-\lambda_k^{(2)}/\lambda_k^{(1)}}\psi_k^{(1)}+O(\epsilon^2) \rangle\nonumber \\
	&= \epsilon\sum_{k\neq i}\frac{\langle \psi_k^{(1)}, \mathbf{B}\psi_i^{(1)}\rangle}{\lambda_i^{(2)}/\lambda_i^{(1)}-\lambda_k^{(2)}/\lambda_k^{(1)}}\langle \psi_j^{(1)}, \psi_k^{(1)}\rangle +O(\epsilon^2)\nonumber \\
	&= \epsilon \frac{\langle \psi_j^{(1)}, \mathbf{B}\psi_i^{(1)}\rangle}{\lambda_i^{(2)}/\lambda_i^{(1)}-\lambda_j^{(2)}/\lambda_j^{(1)}} +O(\epsilon^2)\, .\nonumber
\end{align}
Combining this with the inequality above, we have for $j\neq i$:
\begin{align}
	\left| \alpha_{ji} \right| \leq  \epsilon
	\frac{1}{\gamma_i} \frac{\left|(\lambda_i^{(2)}-\lambda_j^{(2)})\right|}{\sqrt{\lambda_i^{(1)}\lambda_j^{(1)}}}
	\left|e_j^\top\mathbf{U}_1^\top\mathbf{A}e_i \right| + O(\epsilon^2)\, ,
\end{align}
where the constant depends on $\frac{1}{\sqrt{\lambda_i^{(1)}\lambda_j^{(1)}} \gamma_i}$.

We thus have for \eqref{eq to control 1}:
\begin{align}
 &\frac{\lambda_j^{(2)}}{4}\sum_{i=1}^N\alpha_{ji}^2 
 \left(\log\mu_i-\log\left(\frac{\lambda_j^{(2)}}{\lambda_j^{(1)}}\right)\right)^2\nonumber\\
=&\,\frac{\lambda_j^{(2)}}{4}\alpha_{jj}^2 
\left(\log\mu_j-\log\left(\frac{\lambda_j^{(2)}}{\lambda_j^{(1)}}\right)\right)^2+ \frac{\lambda_j^{(2)}}{4}\sum_{i\neq j}\alpha_{ji}^2 
\left(\log\mu_i-\log\left(\frac{\lambda_j^{(2)}}{\lambda_j^{(1)}}\right)\right)^2
=O(\epsilon^2)\,,\label{final bound part 1}
\end{align}
where the constant depends on $\frac{c(\ln c)^2}{\min_i\{\gamma_i^2\lambda_i^{(1)}\}}$.

This is because the first term is $O(\epsilon^2)$ by
\begin{align}
	\alpha_{jj}^2 \left(\log\mu_j-\log\left(\frac{\lambda_j^{(2)}}{\lambda_j^{(1)}}\right)\right)^2 = &\, \alpha_{jj}^2 \left(\epsilon \langle \psi_j^{(1)}, \mathbf{B}\psi_j^{(1)}\rangle \frac{\lambda_j^{(1)}}{\lambda_j^{(2)}}\right)^2 + O(\epsilon^2) \\
	= &\, \epsilon^2 \alpha_{jj}^2 \left( \frac{\lambda_j^{(1)}}{\lambda_j^{(2)}} \right)^2 (\lambda_j^{(1)}\lambda_j^{(1)})^{-1} \left((\lambda_j^{(2)}-\lambda_j^{(2)})e_j^\top\mathbf{U}_1^\top\mathbf{A}e_j+O(\epsilon)\right) ^2 + O(\epsilon^2)
\end{align}
using Taylor expansion $\log (x+h)=\log(x)+h/x+O(h^2)$ for $x>0$ and sufficiently small $h$.
In addition, noting that $\left(\log\mu_i-\log\left(\frac{\lambda_j^{(2)}}{\lambda_j^{(1)}}\right)\right)^2\leq 5(\ln c)^2$ we have that the second term is bounded by:
\begin{align}
	\frac{\lambda_j^{(2)}}{4}\sum_{i\neq j}\alpha_{ji}^2 \left(\log\mu_i-\log\left(\frac{\lambda_j^{(2)}}{\lambda_j^{(1)}}\right)\right)^2 & \leq \frac{\lambda_j^{(2)}}{4}5(\ln c)^2 \sum_{i\neq j}\alpha_{ji}^2 \nonumber \\
	& \leq \frac{5\lambda_j^{(2)}}{4}(\ln c)^2 \sum_{i\neq j} \epsilon^2 
	\frac{1}{\gamma_i^2} \frac{(\lambda_i^{(2)}-\lambda_j^{(2)})^2}{\lambda_i^{(1)}\lambda_j^{(1)}}
	\left|e_j^\top\mathbf{U}_1^\top\mathbf{A}e_i \right|^2 \nonumber\\
	& \leq \frac{5}{4}\epsilon^2(\ln c)^2\frac{\lambda_j^{(2)}}{\lambda_j^{(1)}} \sum_{i\neq j} 
	\frac{1}{\gamma_i^2} \frac{(\lambda_i^{(2)}-\lambda_j^{(2)})^2}{\lambda_i^{(1)}}\left|e_j^\top\mathbf{U}_1^\top\mathbf{A}e_i \right|^2 \, .
\end{align}
Continuing with a few coarse steps:
\begin{align}
    \frac{\lambda_j^{(2)}}{4}\sum_{i\neq j}\alpha_{ji}^2 \left(\log\mu_i-\log\left(\frac{\lambda_j^{(2)}}{\lambda_j^{(1)}}\right)\right)^2 & \leq \frac{5}{4}\epsilon^2\frac{c(\ln c)^2}{\min_i\{\gamma_i^2\lambda_i^{(1)}\}} \sum_{i\neq j} 
	\left|e_j^\top\mathbf{U}_1^\top\mathbf{A}e_i \right|^2\nonumber\\
	& \leq \frac{5}{2}\epsilon^2\frac{c(\ln c)^2}{\min_i\{\gamma_i^2\lambda_i^{(1)}\}}
\end{align}
where $\left|\lambda_i^{(2)}-\lambda_j^{(2)}\right|^2\leq 1$ due to the normalization of $\mathbf{W}_1$ and $\mathbf{W}_2$ and $\sum_{i\neq j} \left|e_j^\top\mathbf{U}_1^\top\mathbf{A}e_i \right|^2 \leq 2 \| \mathbf{A} e_j \|^2 \leq 2$ due to $\| \mathbf{A}\|=1$.

%
Similarly, it can be shown for \eqref{eq to control 2} that
\[
\sum_{i=1}^N  \alpha^2_{ji}(\log \mu_i)^2\left(\sqrt{\mu_i}-\sqrt{\frac{\lambda_j^{(2)}}{\lambda_j^{(1)}}}\right)^2=O(\epsilon^2)\,,
\]
and the proof is concluded.
\end{proof}

\begin{remark}
Note that the implied constant $\frac{\sqrt{c}\ln c}{\min_i\left(\gamma_i\sqrt{\lambda_i^{(1)}}\right)}$ might be large, narrowing the scope of Theorem \ref{prop:pseudo_apart}. Particularly in our context, the matrix $\mathbf{W}_1$ (and $\mathbf{W}_2$) tends to be close to low rank, for which $\min_i\left(\gamma_i\sqrt{\lambda_i^{(1)}}\right)$ is small.
\end{remark}
\begin{remark}
Empirically, we observe that $\psi_j^{(1)}=\sum_i \alpha_{ji} v_i \simeq \sum_{i \sim j} \alpha_{ji} v_i$, i.e., only a small number of expansion coefficients $\alpha_{ji}$ are non-negligible, for which $\lambda_i^{(1)}$ is close to $\lambda_j^{(1)}$. 
Therefore, in practice, the implied constant depends on $1/\min_{i \sim j}\left(\gamma_i\sqrt{\lambda_i^{(1)}}\right)$.
Since we are usually interested in principal components $\psi_j^{(1)}$ (i.e., with large $\lambda_j^{(1)}$), the implied constant is typically sufficiently large. 
\end{remark}
}

\section{Conclusions\label{sec:conc}}


In this work, we introduce a new multi-resolution analysis of temporal high-dimensional data with an underlying time-varying manifold structure. Our analysis is based on the definition of two new composite operators that represent the relation of two aligned datasets jointly sampled from two diffeomorphic manifolds in terms of their spectral components. Specifically, we showed that these operators not only recover but also distinguish different types of common spectral components of the underlying manifolds and that each operator emphasizes different properties. 
One operator was shown to emphasize common components that are similarly expressed in the two manifolds, and the other operator was shown to emphasize the common components that are expressed with significantly different eigenvalues.
In the context of spatiotemporal data analysis, the application of the new operators is analogous to low-pass and high-pass filters. Therefore, by applying them in multiple resolutions, we devise a wavelet-like analysis framework.
We demonstrated this framework on a dynamical system describing a transitory double-gyre flow, showing that such a framework can be used for the analysis of non-stationary multivariate time-series.

In addition to spatiotemporal analysis, we showed that the new composite operators may be useful for multimodal data analysis as well. Specifically, we showed application to remote sensing, demonstrating the recovery of meaningful properties expressed by  different sensing modalities.

In the future, we plan to extend the definition of the operators $\mathbf{S}$ and $\mathbf{F}$ from two to more time frames (datasets). 
In addition, since our analysis results in a large number of vectors representing the common components at different scales and time-points, we plan to develop compact representations of these components, which may lead to improved, more conclusive results for highly non-stationary time-series.

Finally, we remark that in our model, we represent each sample by an undirected weighted graph, and then, analyze the temporal sequence of graphs. Another interesting future work would be to investigate our Riemannian composite operators in the context of graph neural networks (GNNs) and graph convolutional networks (GCNs) \cite{scarselli2008graph,kipf2016semi,bronstein2017geometric}.

\bibliographystyle{abbrv}
\bibliography{papers}

\clearpage
\begin{appendix}
\section{Additional Statements\label{app:add_state}}

The following equivalent forms of operators $\mathbf{S}$ and $\mathbf{F}$ are used in proofs of the theorems.

\begin{proposition}[Equivalent Forms of the Operators \textbf{S} and \textbf{F}]\label{prop_app:sa_equiv_forms}
        We have
        \[
        \mathbf{S} = \mathbf{W}_1^{1/2}\left(\mathbf{W}_1^{-1/2}\mathbf{W}_2\mathbf{W}_1^{-1/2}\right)^{1/2}\mathbf{W}_1^{1/2} = \left(\mathbf{W}_2\mathbf{W}_1^{-1}\right)^{1/2}\mathbf{W}_1
        \]
    and
        \[
        \mathbf{F} = \mathbf{S}^{1/2}\log\left(\mathbf{S}^{-1/2}\mathbf{W}_1\mathbf{S}^{-1/2}\right)\mathbf{S}^{1/2} = \log\left(\mathbf{W}_1\mathbf{S}^{-1}\right)\mathbf{S}\,.
        \]
\end{proposition}
\begin{proof}
To show the claim for $\mathbf{S}$, define the following:
    \begin{align}
    \mathbf{M} = \mathbf{W}_1^{-1/2}\mathbf{W}_2\mathbf{W}_1^{-1/2}\ \ \mbox{ and } \ \ 
    \tilde{\mathbf{M}} =  \mathbf{W}_2\mathbf{W}_1^{-1}.\label{eq_app:s_equiv_2}
    \end{align}
    Since $\mathbf{W}_1$ and $\mathbf{W}_2$ are positive definite, the matrix $\mathbf{M}$ is positive definite, and hence $\mathbf{M}$ and $\tilde{\mathbf{M}}$ are similar via $\mathbf{M}=\mathbf{W}_1^{-1/2}\tilde{\mathbf{M}}\mathbf{W}_1^{1/2}$.
    Denote the eigenvalue and eigenvector matrices of $\mathbf{M}$ by $\mathbf{\Lambda}^{(\mathbf{M})}$ and $\mathbf{V}^{(\mathbf{M})}$ respectively.
    Therefore, 
    the eigenvalue matrix of $\tilde{\mathbf{M}}$ is $\mathbf{\Lambda}^{(\tilde{\mathbf{M}})}=\mathbf{\Lambda}^{(\mathbf{M})}$, and the right and left eigenvectors are $\mathbf{V}_R^{(\tilde{\mathbf{M}})}=\mathbf{W}_1^{1/2}\mathbf{V}^{(\mathbf{M})}$ and $\mathbf{V}_L^{(\tilde{\mathbf{M}})}=\mathbf{W}_1^{-1/2}\mathbf{V}^{(\mathbf{M})}$ respectively.
    Thus, we have
    \begin{align}
    \mathbf{M}^{1/2} =&\, \mathbf{V}^{(\mathbf{M})}\left(\mathbf{\Lambda}^{(\mathbf{M})}\right)^{1/2}\left(\mathbf{V}^{(\mathbf{M})}\right)^T\label{eq_app:s_equiv_sqrt1}\\
    \tilde{\mathbf{M}}^{1/2}  = &\, \mathbf{W}_1^{1/2}\mathbf{V}^{(\mathbf{M})}\left(\mathbf{\Lambda}^{(\mathbf{M})}\right)^{1/2}\left(\mathbf{V}^{(\mathbf{M})}\right)^T\mathbf{W}_1^{-1/2}\nonumber
    \end{align}
    and hence $\tilde{\mathbf{M}}^{1/2}=\mathbf{W}_1^{1/2}\mathbf{M}^{1/2}\mathbf{W}_1^{-1/2}$.
    As a result, we have
        \[
        \mathbf{S}=\mathbf{W}_1^{1/2}\mathbf{M}^{1/2}\mathbf{W}_1^{1/2}=\mathbf{W}_1^{1/2}\mathbf{M}^{1/2}\mathbf{W}_1^{-1/2}\mathbf{W}_1=\tilde{\mathbf{M}}^{1/2}\mathbf{W}_1\,,
        \]
        which shows the claim.
    
    The proof for the claim for the operator $\mathbf{F}$ is similar. Define 
    \begin{align}
    \mathbf{N} := \mathbf{S}^{-1/2}\mathbf{W}_1\mathbf{S}^{-1/2}\ \ \mbox{ and } 
    \ \
    \tilde{\mathbf{N}} := \mathbf{W}_1\mathbf{S}^{-1}.\label{eq_app:a_equiv_2}
    \end{align}
    Since $\mathbf{S}$ is positive definite, $\mathbf{N}$ and $\tilde{\mathbf{N}}$ are similar and their eigenvalues and eigenvectors are related by $\mathbf{\Lambda}^{(\tilde{\mathbf{N}})}=\mathbf{\Lambda}^{(\mathbf{N})}$, $\mathbf{V}_R^{(\tilde{\mathbf{N}})}=\mathbf{S}^{1/2}\mathbf{V}^{(\mathbf{N})}$ and $\mathbf{V}_L^{(\tilde{\mathbf{N}})}=\mathbf{S}^{-1/2}\mathbf{V}^{(\mathbf{N})}$.
    The matrix logarithm of $\mathbf{N}$ and $\tilde{\mathbf{N}}$ can then be expressed by:
    \begin{align}
    \log(\mathbf{N}) = & \,\mathbf{V}^{(\mathbf{N})}\log\left(\mathbf{\Lambda}^{(\mathbf{N})}\right)\left(\mathbf{V}^{(\mathbf{N})}\right)^T\label{eq_app:a_equiv_sqrt1}\\
    \log(\tilde{\mathbf{N}}) = & \,\mathbf{S}^{1/2}\mathbf{V}^{(\mathbf{N})}\log\left(\mathbf{\Lambda}^{(\mathbf{N})}\right)\left(\mathbf{V}^{(\mathbf{N})}\right)^T\mathbf{S}^{-1/2},\nonumber
    \end{align}
    Based on these expressions the relationship between the logarithm of the two matrices is $\log(\tilde{\mathbf{N}})=\mathbf{S}^{1/2}\log(\mathbf{N})\mathbf{S}^{-1/2}$.
    The proof can now be concluded by 
\[
\mathbf{F}=\mathbf{S}^{1/2}\log(\mathbf{N})\mathbf{S}^{1/2}=\mathbf{S}^{1/2}\log(\mathbf{N})\mathbf{S}^{-1/2}\mathbf{S}=\log(\tilde{\mathbf{N}})\mathbf{S}\,.
\]
\end{proof}

Next, for completeness, we explicitly show that the equivalence between definitions 2 and 3 of the pseudo-spectrum in Definition \ref{def:pseudo} includes a shared pseudo-eigenvector.

\begin{proposition}\label{prop_app:pseudo_explicit}
Consider $\mathbf{M}\in\mathbb{R}^{N \times N}$ and a small $\epsilon>0$.
If $v\in\mathbb{R}^N$ with $\|v\|_2=1$ s.t. $\|(\mathbf{M}-\lambda \mathbf{I})v\|_2 \le \epsilon$ for $\lambda \in \mathbb{R}$, then there exists $\mathbf{E}\in\mathbb{R}^{N \times N}$ with $\|\mathbf{E}\|\le \epsilon$ s.t. $(\mathbf{M} + \mathbf{E})v = \lambda v$.
\end{proposition}

\begin{proof}
Define the following rank one operator:
\[
    \mathbf{B}u = -\langle u , v \rangle (\mathbf{M} - \lambda \mathbf{I})v,
\]
where $u\in\mathbf{R}^N$. Then, we have that $\|\mathbf{B}\|\le \epsilon$ and $(\mathbf{M} + \mathbf{B})v = \lambda v$.
\end{proof}

\section{3-Dimensional Tori Example\label{sub:4d_torus}}
In this subsection we consider datasets of samples from two 3-dimensional tori. 
Using these datasets we demonstrate that the operators $\mathbf{S}$ and $\mathbf{F}$ indeed recover the similarly expressed common components and the differently expressed common components, respectively.
In addition, in Subsection \ref{subsub:3d_tori_unique}, we demonstrate that these operators still recover the common components, even when modality-specific unique structures exist.

\subsection{With Common Components Only\label{subsub:3d_tori_common}}
Consider two 3-dimensional tori in two observation spaces, denoted by $\mathcal{O}_1$ and $\mathcal{O}_2$. 
Both tori are obtained by sampling the product of three $\mathcal{S}^1$ manifolds that differ in scaling, which are embedded into two ambient spaces:
\begin{eqnarray}
\mathcal{O}_1 & = f^{(1)}_1\left(\mathcal{S}^{1}\right)\times f^{(2)}_1\left(\mathcal{S}^{1}\right)\times\mathcal{S}^{1} & \mapsto\mathcal{X}_1\label{eq:3d_tori_common_prob_form1}\\
\mathcal{O}_2 & = f^{(1)}_2\left(\mathcal{S}^{1}\right)\times f^{(2)}_2\left(\mathcal{S}^{1}\right)\times\mathcal{S}^{1} & \mapsto\mathcal{X}_2\label{eq:3d_tori_common_prob_form2}
\end{eqnarray}
where $f^{(k)}_\ell:\mathcal{S}^{1}\rightarrow\mathcal{S}^1$, $\nabla f^{(k)}_\ell\vert_x=\alpha^{(k)}_\ell$ $\forall x\in\mathcal{S}^{1}$ and $\alpha^{(1)}_1=\alpha^{(2)}_2$, $\alpha^{(2)}_1=\alpha^{(1)}_2$.
We assume that the embedding into the ambient space preserves the order of the $\mathcal{S}^1$ manifolds and therefore, these scale differences can be represented by switching two of the main angles in the parameterization of the two tori.
The samples in the ambient spaces can be explicitly described by the following embedding in 4D:
\begin{eqnarray}
\mathcal{X}_1\ni\mathbf{x}_1[i] = \begin{Bmatrix*}[l]
x_\ell[i] & = & (\tilde{R}+(R+r\cos(\theta_2[i]))\cos(\theta_1[i]))\cos(\theta_3[i])\\ 
y_\ell[i] & = & (\tilde{R}+(R+r\cos(\theta_2[i]))\cos(\theta_1[i]))\sin(\theta_3[i])\\
z_\ell[i] & = & (R+r\cos(\theta_2[i]))\sin(\theta_1[i])\\
w_\ell[i] & = & r\sin(\theta_2[i])
\end{Bmatrix*}\label{eq:4d_param1}\\
\mathcal{X}_2\ni\mathbf{x}_2[i] = \begin{Bmatrix*}[l]
x_\ell[i] & = & (\tilde{R}+(R+r\cos(\theta_1[i]))\cos(\theta_2[i]))\cos(\theta_3[i])\\ 
y_\ell[i] & = & (\tilde{R}+(R+r\cos(\theta_1[i]))\cos(\theta_2[i]))\sin(\theta_3[i])\\
z_\ell[i] & = & (R+r\cos(\theta_1[i]))\sin(\theta_2[i])\\
w_\ell[i] & = & r\sin(\theta_1[i])
\end{Bmatrix*}\label{eq:4d_param2}
\end{eqnarray}
where $\theta_1[i],\theta_2[i],\theta_3[i]\in [0,2\pi]$, $r=2$, $R=7$ and $\tilde{R}=15$.
In this setting, the radii $R$ and $r$ are related to the scale parameters $\alpha_1^{(1)}=\alpha_2^{(2)}$ and $\alpha_1^{(2)}=\alpha_2^{(1)}$ that define the diffeomorphisms $f^{(k)}_\ell$ in \eqref{eq:3d_tori_common_prob_form1} and \eqref{eq:3d_tori_common_prob_form2}.

We sample $N=2000$ points from each torus, $\{\mathbf{x}_\ell[i]\}_{i=1}^N$, $\mathbf{x}_\ell[i]\in\mathbb{R}^4$, $\ell=1,2$, with point correspondence between the two tori, which is obtained through correspondence of the samples of $\theta_1$, $\theta_2$ and $\theta_3$. 

The data is visualized in Figure \ref{fig:Tori_data_3d_common_theta} where projections of the two tori are colored according to the $3$ different angles, $\cos(\theta_1)$ in (a), $\cos(\theta_2)$ in (b) and $\cos(\theta_3)$ in (c).
This figure presents projections of the tori to the following 3-dimensional spaces: `xyz', `xyw' and `yzw'.
\begin{figure}[bhtp!]
\centering
\subfloat[]{\includegraphics[width=0.35\textwidth]{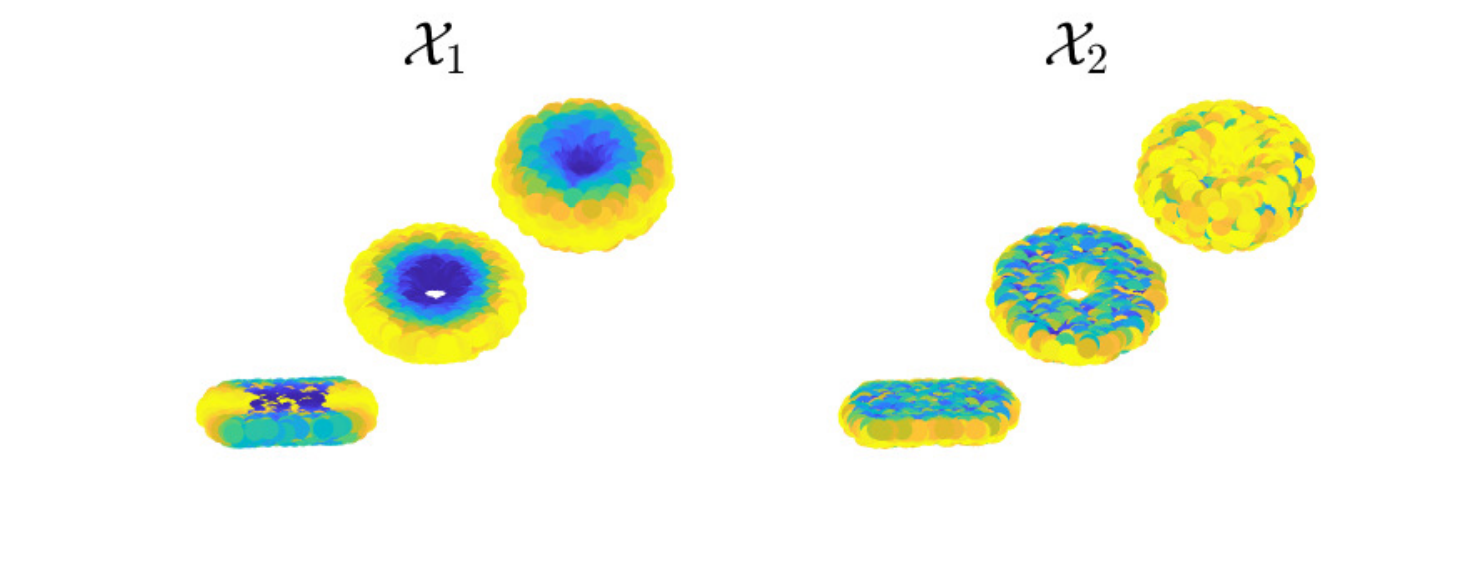}}
\subfloat[]{\includegraphics[width=0.35\textwidth]{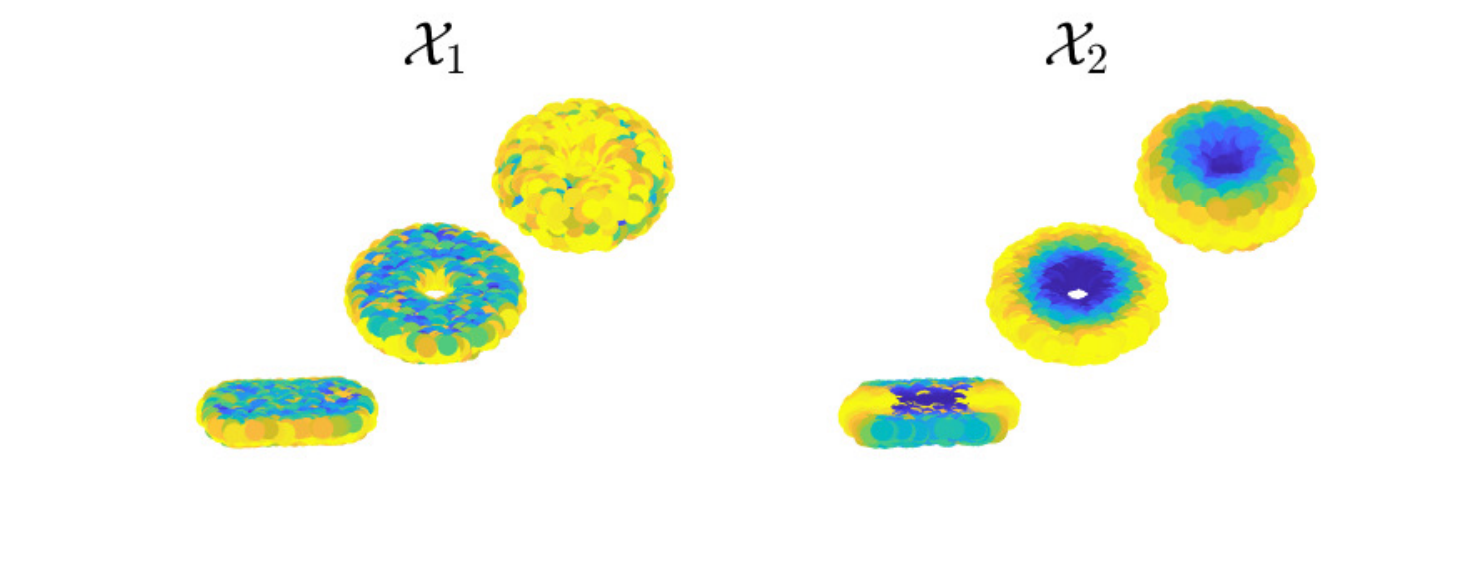}}
\subfloat[]{\includegraphics[width=0.35\textwidth]{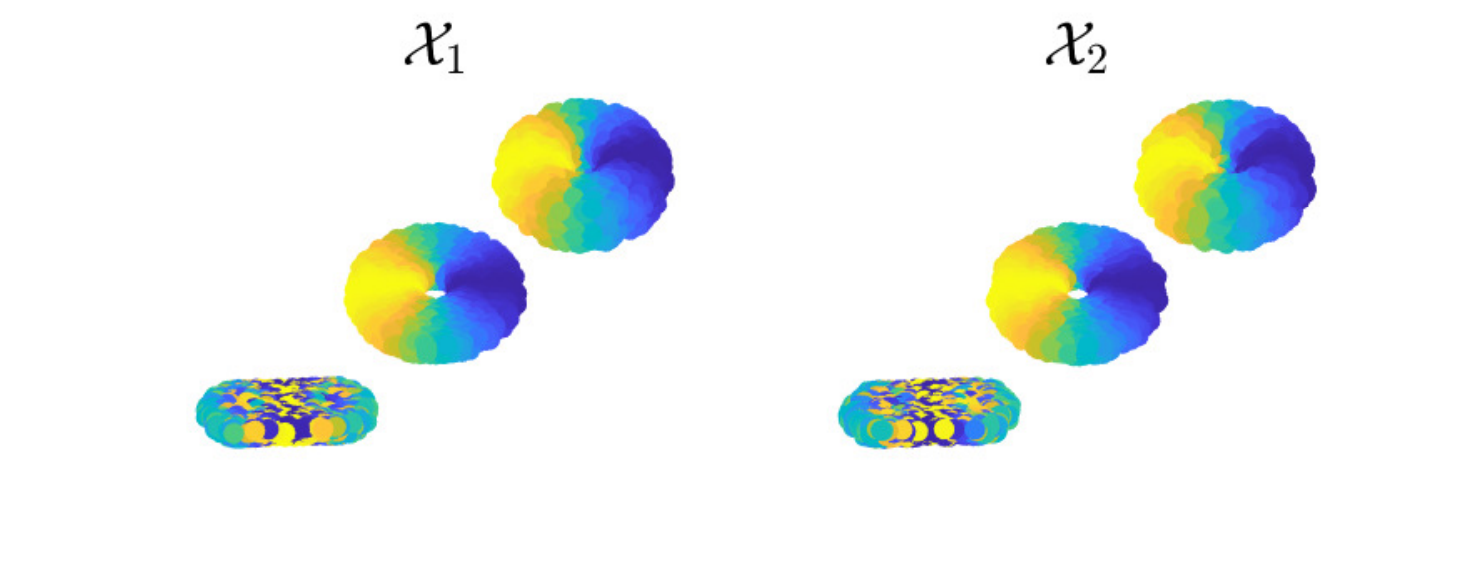}}
\caption[Two 3D tori with angles $\theta_1$, $\theta_2$, $\theta_3$]{Two 3D tori with angles $\theta_1$, $\theta_2$, $\theta_3$. Angles $\theta_1$ and $\theta_2$ are switched between the tori and angles $\theta_3$ is similar. 
Both tori are colored according to (a) $\cos(\theta_1)$, (b) $\cos(\theta_2)$ and (c) $\cos(\theta_3)$.\label{fig:Tori_data_3d_common_theta}}
\end{figure}
In this example, the two modalities, represented by $\mathcal{X}_1$ and $\mathcal{X}_2$, includes only common components, some of which are differently expressed and some are similarly expressed. 
Specifically, $\theta_3$ is related to the similarly expressed common components, as depicted by Figure \ref{fig:Tori_data_3d_common_theta} (c), and $\theta_1$ and $\theta_2$ are related to the differently expressed common components, as depicted by Figure \ref{fig:Tori_data_3d_common_theta} (a) and (b).

We apply Algorithm \ref{alg:SA_implementation} to the samples of the two 3-dimensional tori with $\sigma_\ell=\mathrm{median}(d(\mathbf{x}_\ell[i],\mathbf{x}_\ell[j]))$, $i,j=1,...,N$, $\ell=1,2$, and compute the eigenvectors of the kernels $\mathbf{W}_1$, $\mathbf{W}_2$, as well as the operators $\mathbf{S}$ and $\mathbf{F}$ and their eigenvectors.
These eigenvectors compose the embedding of each operator.
The eigenvectors are denoted in the following figures by $\psi^{(1)}_n$, $\psi^{(2)}_n$, $\psi^{(\mathbf{S})}_n$ and $\psi^{(\mathbf{F})}_n$, respectively, and are ordered according to a decreasing magnitude of their corresponding eigenvalues, denoted by $\lambda^{(1)}_n$, $\lambda^{(2)}_n$, $\lambda^{(\mathbf{S})}_n$ and $\lambda^{(\mathbf{F})}_n$, respectively. 

To demonstrate that the operator $\mathbf{S}$ indeed emphasizes the similarly expressed common components and that $\mathbf{F}$ emphasizes the differently expressed common components, we compare their eigenvectors with the eigenvectors of $\mathbf{W}_1$ and $\mathbf{W}_2$ and with the angles $\theta_1$, $\theta_2$ and $\theta_3$ parameterizing the tori.
Note that the sign of the eigenvalues of $\mathbf{F}$ is meaningful and provides information on the source of the dominant difference components, as a demonstrated by Theorem \ref{theo:A_eigs}. 
Therefore, in order to clearly distinguish between the largest positive and smallest negative eigenvalues, an eigenvector of $\mathbf{F}$ that correspond to the $k$th largest (positive) eigenvalue is denoted by $\psi^{(\mathbf{F})}_k$ and an eigenvector that corresponds to the $\ell$th smallest negative eigenvalue is denoted by $\psi^{\mathbf{(A)}}_{-\ell\ \mathrm{mod}\ N}$, throughout this section.

Figure \ref{fig:Tori_3d_all_common_SA} presents the two tori, top and bottom in Figure \ref{fig:Tori_3d_all_common_SA} (a) and in Figure \ref{fig:Tori_3d_all_common_SA} (b), colored according to these eigenvectors and the torus angles.
The vector that each torus was colored by is denoted in the title of each plot. 
In Figure \ref{fig:Tori_3d_all_common_SA} (a), the tori are colored according to the two leading eigenvectors of $\mathbf{S}$, the leading eigenvector of $\mathbf{W}_1$ (top) and $\mathbf{W}_2$ (bottom) and the cosine of $\theta_3$, from left to right.
In Figure \ref{fig:Tori_3d_all_common_SA} (b), the tori are colored according to two leading eigenvectors of $\mathbf{F}$ (corresponding to positive and negative eigenvalues), eigenvector number $7$ of $\mathbf{W}_1$ (top) and $\mathbf{W}_2$ (bottom), the cosine of $\theta_1$ and the cosine of $\theta_2$, from left to right.
\begin{figure}[bhtp!]
\centering
\subfloat[]{\includegraphics[width=0.8\textwidth]{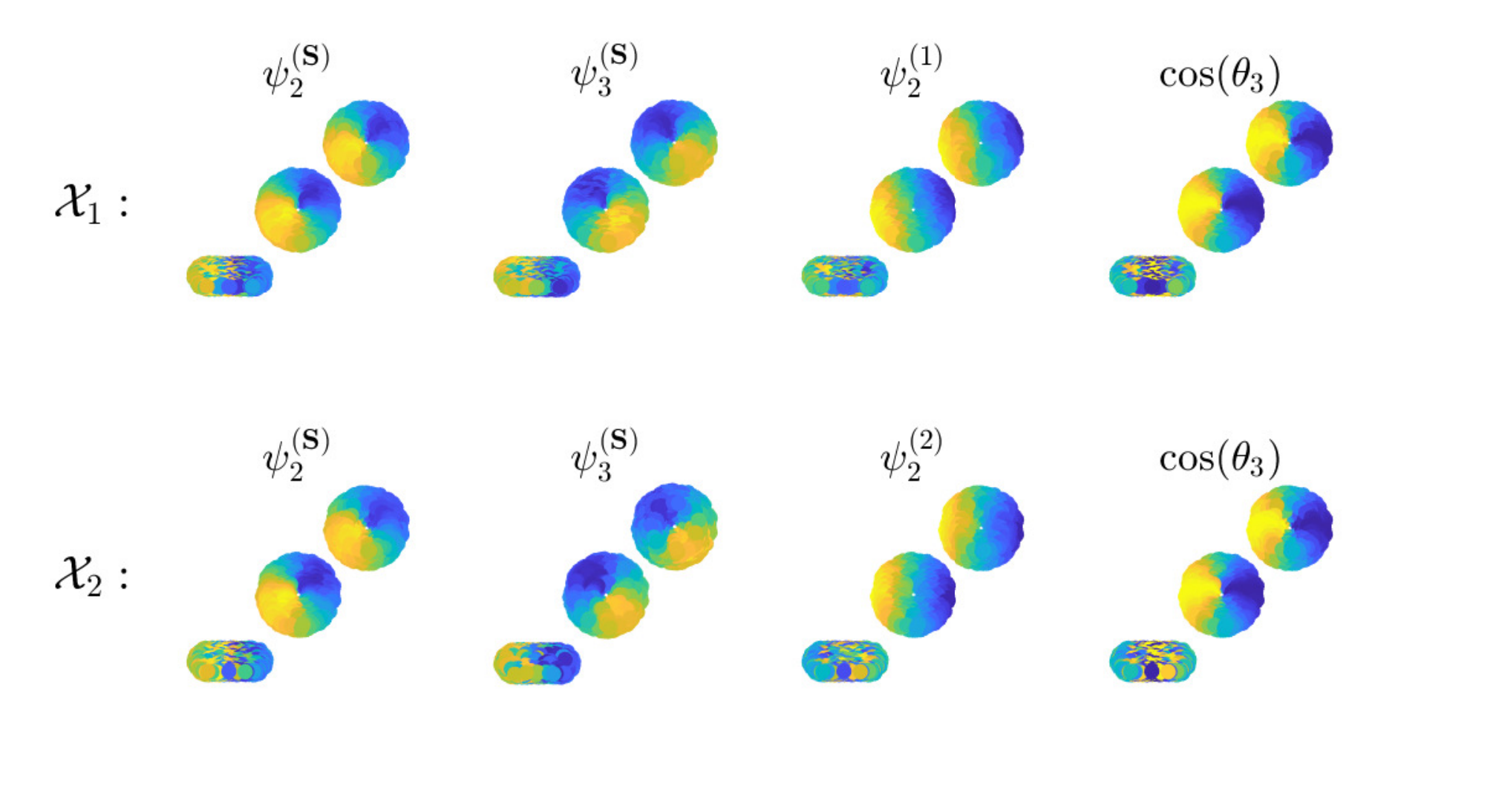}}

\subfloat[]{\includegraphics[width=1\textwidth]{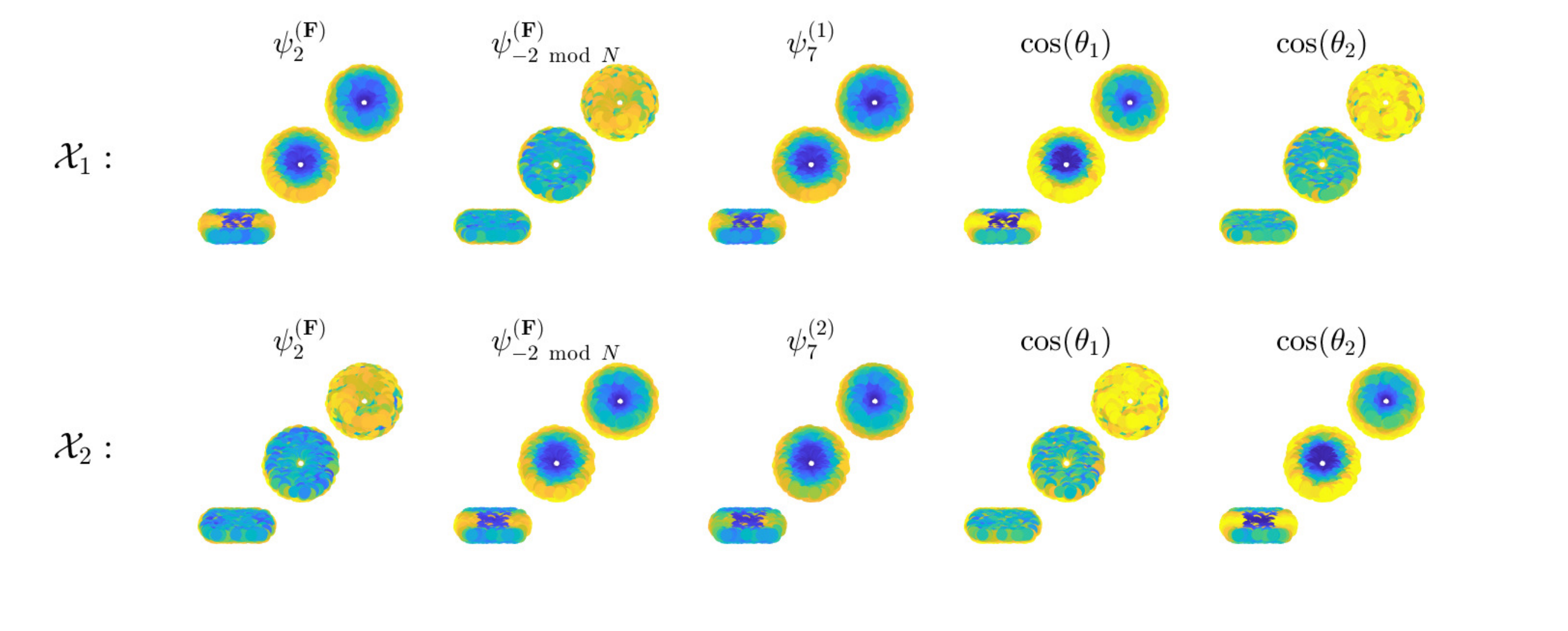}}
\caption{The two tori colored according to eigenvectors of the operators and kernels and according to the torus angles.
(a) Eigenvectors and angle that are captured by the operator $\mathbf{S}$, (b) eigenvectors and angles that are captured by the operator $\mathbf{F}$.
\label{fig:Tori_3d_all_common_SA}}
\end{figure}

Figure \ref{fig:Tori_3d_all_common_SA} (a) depicts that the operator $\mathbf{S}$ captures the similarly expressed common component, i.e. $\theta_3$, the angle that is related to the $\mathcal{S}^1$ manifold that does not undergo any transformation between the two observation spaces.
This is indicated by similarity of the torus colors (up to some rotation), when comparing the coloring according to the eigenvectors of $\mathbf{S}$ with the coloring according to $\cos(\theta_3)$.
In addition, the torus colors are also very similar when colored according to the second eigenvectors of $\mathbf{W}_1$ and $\mathbf{W}_2$. 
This implies that $\theta_3$ is highly expressed by both kernels $\mathbf{W}_1$ and $\mathbf{W}_2$, meaning that it is a dominant component in both tori.

Note that the first eigenvectors of $\mathbf{W}_1$ and $\mathbf{W}_2$ (and in this example also of $\mathbf{S}$) are related to the point density on each torus, which is of less interest in this example, and were omitted therefor.

In contrast, Figure \ref{fig:Tori_3d_all_common_SA} (b) depicts that the operator $\mathbf{F}$ captures the differently expressed common components, i.e. $\theta_1$ and $\theta_2$, the angles that are related to $\mathcal{S}^1$ manifolds that undergo transformations between the two observation spaces.
This is indicated by the similarity in the torus colors, when comparing the coloring according to the eigenvectors of $\mathbf{F}$ with the coloring according to $\cos(\theta_1)$ and $\cos(\theta_2)$.
Note that this figure also nicely demonstrates the connection between the sign of the eigenvalues of $\mathbf{F}$ and the source of the dominant component.
For example, the angle $\theta_1$ is a more dominant component in torus $\mathcal{X}_1$, i.e. its corresponding $\mathcal{S}^1$ manifold has a larger radius compared with $\mathcal{X}_2$, and this angle is captured by an eigenvector of $\mathbf{F}$ that corresponds to a positive eigenvalue.
In contrast, the eigenvector that corresponds to a negative eigenvalue of $\mathbf{F}$ captures the angle $\theta_2$, which is a more dominant component in torus $\mathcal{X}_2$.

Note that in kernels $\mathbf{W}_1$ and $\mathbf{W}_2$ the angles $\theta_1$ and $\theta_2$ are expressed only in less dominant eigenvectors (corresponding to much smaller eigenvalues).
Specifically, $\theta_1$ and $\theta_2$ are first captured by $\mathbf{W}_1$ only in eigenvector number $6$ (and $7$ as presented in the figure) and eigenvector number $17$, respectively, and their location in $\mathbf{W}_2$ is similar but reversed.
Therefore, this demonstrates that the operator $\mathbf{F}$ significantly emphasized these differently expressed common components.
We note in addition, that the eigenvectors of $\mathbf{F}$ that correspond to the first (largest) positive and negative eigenvalues were highly related to $\theta_1$ and $\theta_2$ as well, specifically, they captured $\sin(\theta_1)$ and $\sin(\theta_2)$, and were omitted for brevity.


\subsection{With Unique Components\label{subsub:3d_tori_unique}}
Consider now a slightly different setting of the two $3$D tori, in which the third $\mathcal{S}^1$ manifold undergoes two different diffeomorphisms, which lead to a unique structure in each observation space.
The formulation for this setting is given by:
\begin{eqnarray}
\mathcal{O}_1 & = f^{(1)}_1\left(\mathcal{S}^{1}\right)\times f^{(2)}_1\left(\mathcal{S}^{1}\right)\times\mathcal{N}_1 & \mapsto\mathcal{X}_1\\
\mathcal{O}_2 & = f^{(1)}_2\left(\mathcal{S}^{1}\right)\times f^{(2)}_2\left(\mathcal{S}^{1}\right)\times\mathcal{N}_2 & \mapsto\mathcal{X}_2
\end{eqnarray}
where $\mathcal{N}_1$ and $\mathcal{N}_2$ denote two different unique structures, $f^{(k)}_\ell:\mathcal{S}^{1}\rightarrow\mathcal{S}^1$, $\nabla f^{(k)}_\ell\vert_x=\alpha^{(k)}_\ell$ $\forall x\in\mathcal{S}^{1}$ and $\alpha^{(1)}_1=\alpha^{(2)}_2$, $\alpha^{(2)}_1=\alpha^{(1)}_2$.
We assume that the embedding into the ambient spaces express these unique structures as different permutations of the angel corresponding to this $\mathcal{S}^1$ manifold.
In addition, similarly to Subsection \ref{subsub:3d_tori_common}, we assume that the scale differences between the first two manifolds can be represented by switching two of the main angles in the parameterization of the two tori.
In this case, the samples in the ambient spaces can be  explicitly described by the following embedding in 4D:
\begin{eqnarray}
\mathcal{X}_1\ni\mathbf{x}_1[i] = \begin{Bmatrix*}[l]
x_\ell[i] & = & (\tilde{R}+(R+r\cos(\theta_2[i]))\cos(\theta_1[i]))\cos(\theta_3[i])\\ 
y_\ell[i] & = & (\tilde{R}+(R+r\cos(\theta_2[i]))\cos(\theta_1[i]))\sin(\theta_3[i])\\
z_\ell[i] & = & (R+r\cos(\theta_2[i]))\sin(\theta_1[i])\\
w_\ell[i] & = & r\sin(\theta_2[i])
\end{Bmatrix*}\\
\mathcal{X}_2\ni\mathbf{x}_2[i] = \begin{Bmatrix*}[l]
x_\ell[i] & = & (\tilde{R}+(R+r\cos(\theta_1[i]))\cos(\theta_2[i]))\cos(\theta_4[i])\\ 
y_\ell[i] & = & (\tilde{R}+(R+r\cos(\theta_1[i]))\cos(\theta_2[i]))\sin(\theta_4[i])\\
z_\ell[i] & = & (R+r\cos(\theta_1[i]))\sin(\theta_2[i])\\
w_\ell[i] & = & r\sin(\theta_1[i])
\end{Bmatrix*}\\
\end{eqnarray}
where $\theta_1[i],\theta_2[i],\theta_3[i],\theta_4[i]\in [0,2\pi]$, $r=2$, $R=7$ and $\tilde{R}=15$.
Note that the difference between the parameterization in this setting compared with the setting in Subsection \ref{subsub:3d_tori_common} is the angle $\theta_4$, which differs from $\theta_3$.
These two angles are visualized in Figure \ref{fig:Tori_data_3d_unique_theta}, which presents projections of the tori to the following 3-dimensional spaces: `xyz', `xyw' and `yzw'.
\begin{figure}[bhtp!]
\centering
\subfloat[]{\includegraphics[width=0.5\textwidth]{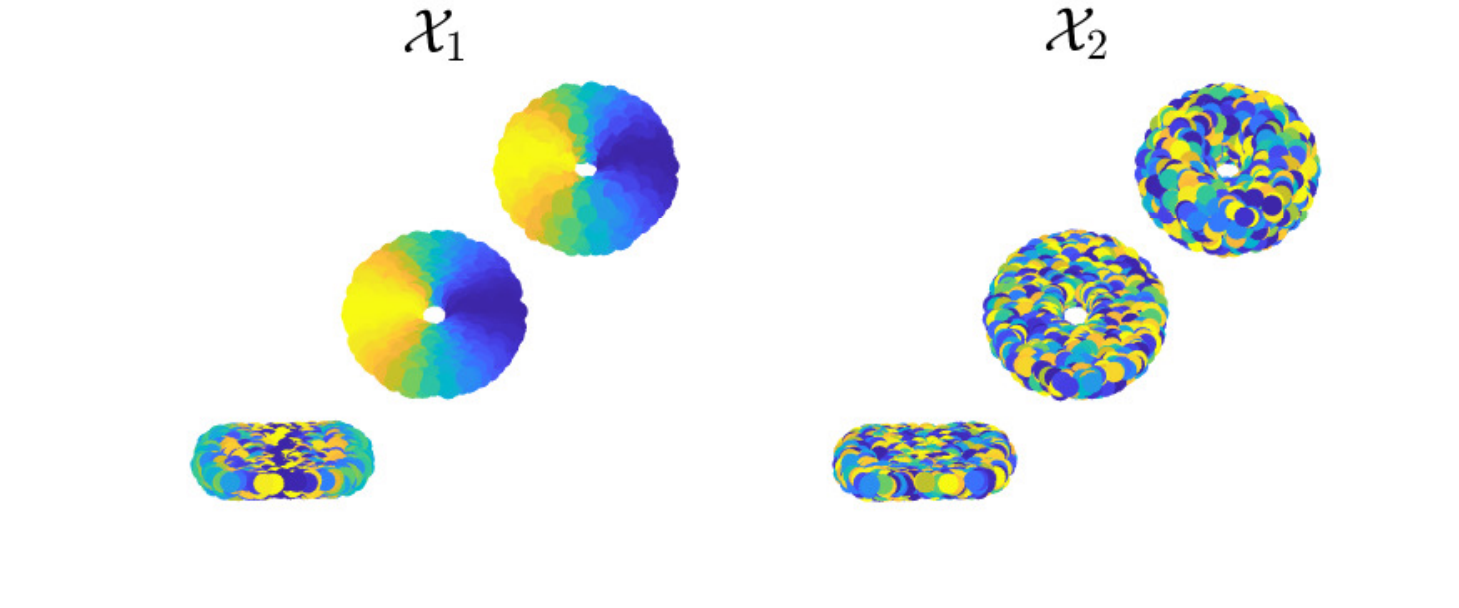}}
\subfloat[]{\includegraphics[width=0.5\textwidth]{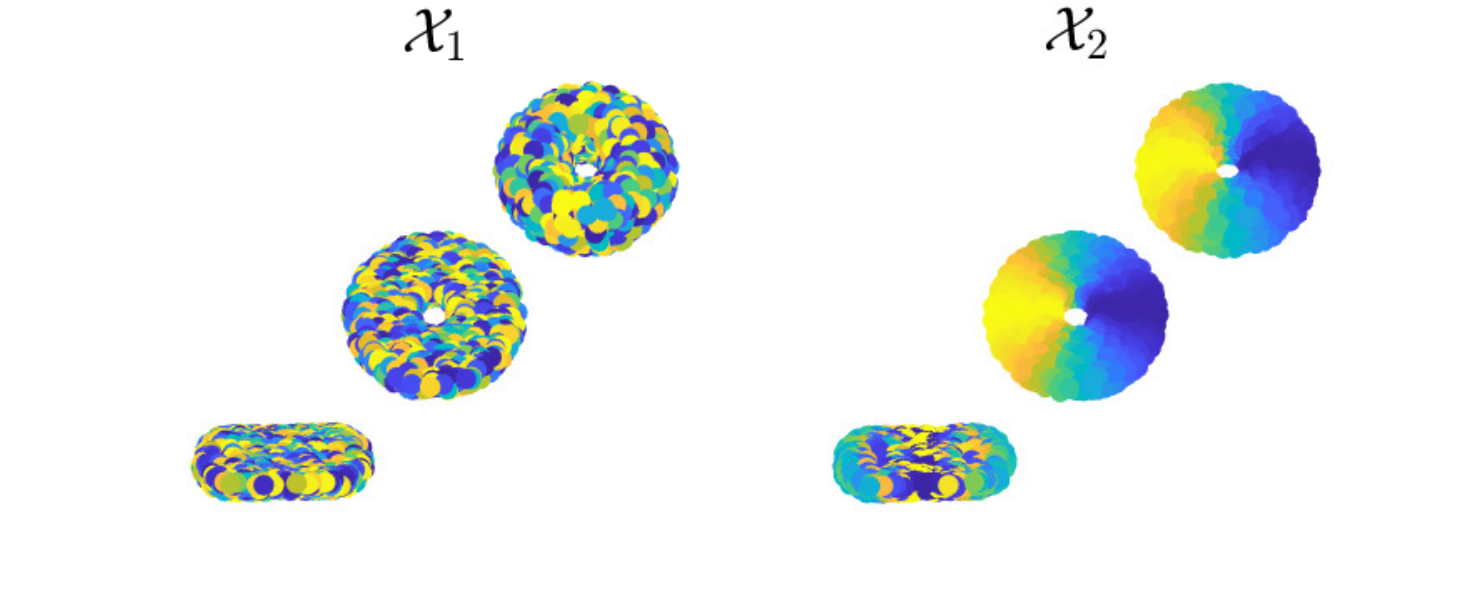}}
\caption{Two 3D tori colored according to (a) $\cos(\theta_3)$ and (b) $\cos(\theta_4)$. Angles $\theta_3$ and $\theta_4$ are unique to each torus.\label{fig:Tori_data_3d_unique_theta}}
\end{figure}
This figure demonstrate that $\theta_3$ is unique to $\mathcal{X}_1$ and that $\theta_4$ is unique to $\mathcal{X}_2$, since coloring $\mathcal{X}_2$ according to $\theta_3$ and $\mathcal{X}_1$ according to $\theta_4$ leads to random coloring, as presented in the right plot of Figure \ref{fig:Tori_data_3d_unique_theta} (a) and the left plot of Figure \ref{fig:Tori_data_3d_unique_theta} (b), respectively.

In summary, the two modalities, represented by $\mathcal{X}_1$ and $\mathcal{X}_2$, include common components that are differently expressed, which are related to $\theta_1$ and $\theta_2$ (similarly to the previous setting), and a unique component, which is related to $\theta_3$ and $\theta_4$.

We sample $N=2000$ points from each torus similarly to Subsection \ref{subsub:3d_tori_common} and apply Algorithm \ref{alg:SA_implementation} with $\sigma_\ell=\mathrm{median}(d(\mathbf{x}_\ell[i],\mathbf{x}_\ell[j]))$, $i,j=1,\dots,N$, $\ell=1,2$.
We then compute the eigenvectors of the operators $\mathbf{S}$ and $\mathbf{F}$.

Figure \ref{fig:Tori_3d_unique_SA} presents the two tori (top and bottom) colored according to the two leading eigenvectors of $\mathbf{F}$ (corresponding to positive and negative eigenvalues), the leading eigenvector of $\mathbf{S}$, the sine of $\theta_1$ and the sine of $\theta_2$, from left to right.
Note that this figure presents the sine of the angles compared with a different eigenvector of $\mathbf{F}$ as a complementary example to Figure \ref{fig:Tori_3d_all_common_SA}, since the two leading eigenvectors of $\mathbf{F}$ (with either negative or positive eigenvalues) capture the sine and cosine of the differently expressed angles.
\begin{figure}[bhtp!]
\centering
\includegraphics[width=1\textwidth]{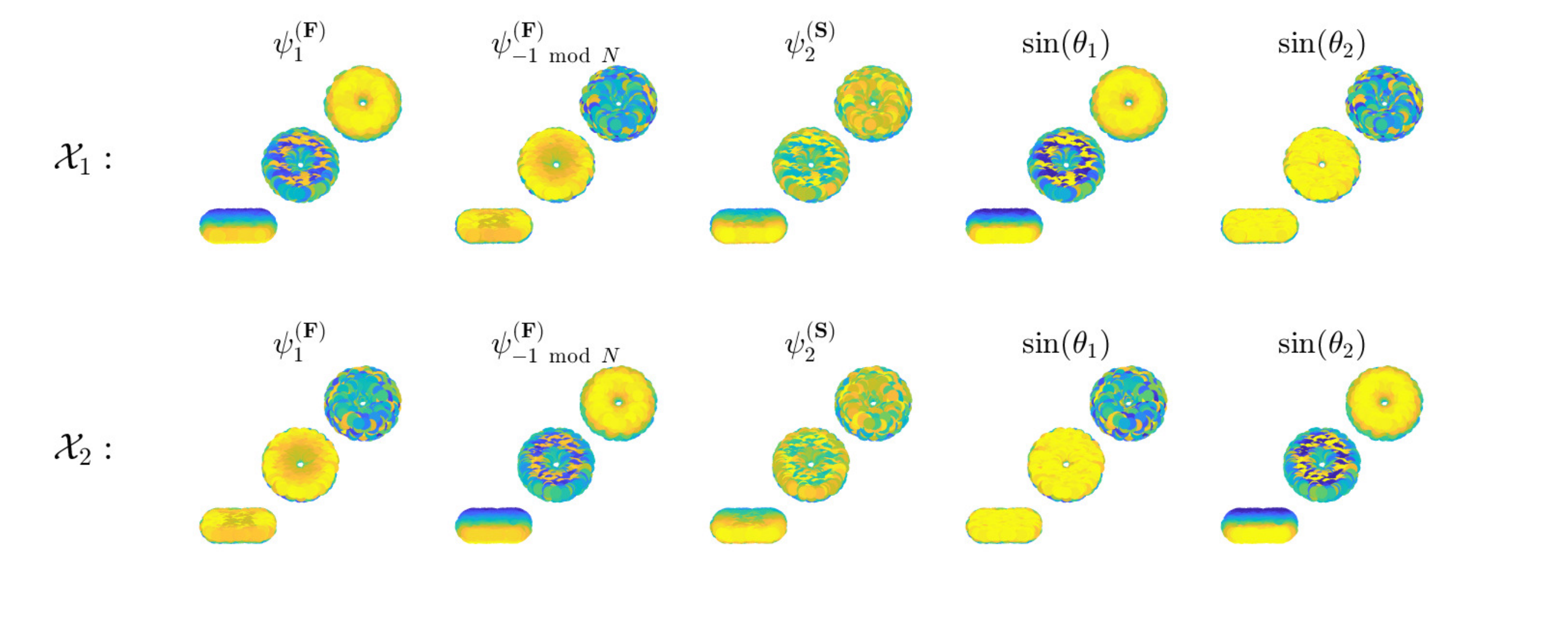} 
\caption{The two tori (top and bottom) colored according to the eigenvectors of $\mathbf{F}$ and $\mathbf{S}$ and according to $\sin(\theta_1)$ and $\sin(\theta_2)$.\label{fig:Tori_3d_unique_SA}}
\end{figure}

Figure \ref{fig:Tori_3d_unique_SA} depicts that the operator $\mathbf{F}$ captures the differently expressed common components, i.e. $\theta_1$ and $\theta_2$.
This is indicated by the similarity in the torus colors, when comparing the coloring according to the eigenvectors of $\mathbf{F}$ with the coloring according to $\sin(\theta_1)$ and $\sin(\theta_2)$.
Note the connection between the sign of the eigenvalues of $\mathbf{F}$ and the source of the dominant component, as depicted by the two left-most and two right-most plots in this figure.

In addition, in this setting, the operator $\mathbf{S}$ recovers the same common components, $\theta_1$ and $\theta_2$, since there are no similarly expressed common components.
However, the coloring of the middle plot in Figure \ref{fig:Tori_3d_unique_SA} depicts that the leading eigenvector of $\mathbf{S}$ is not identical to either one of the angles.
This is due to rotations of eigenvector subspaces of $\mathbf{S}$, which stem from the symmetry of the two tori problem.
Due to the symmetry, the angles $\theta_1$ and $\theta_2$ are expressed by the eigenvectors of both kernels, $\mathbf{W}_1$ and $\mathbf{W}_2$, but in a different complementary order.
For example, if $\theta_1$ appears in eigenvector number $6$ of $\mathbf{W}_1$ and eigenvector number $18$ of $\mathbf{W}_2$, then $\theta_2$ will appear in eigenvector number $18$ of $\mathbf{W}_1$ and eigenvector number $6$ of $\mathbf{W}_2$, and with similar complementary eigenvalues.
Since according to Theorem \ref{theo:S_eigs} the eigenvalues of $\mathbf{S}$ are given by $\sqrt{\lambda^{(1)}\lambda^{(2)}}$, where $\lambda^{(1)}$ and $\lambda^{(2)}$ correspond to common eigenvectors, the eigenvector of $\mathbf{S}$ that captures $\theta_1$ and the eigenvector of $\mathbf{S}$ that captures $\theta_2$ have the same eigenvalue, with a multiplicity of $2$.
This may lead to rotation of the eigenvectors in the eigenspace of eigenvalue $\sqrt{\lambda^{(1)}\lambda^{(2)}}$.
As a result, the computed eigenvectors correspond to a combination of $\theta_1$ and $\theta_2$, as depicted by the middle plot in Figure \ref{fig:Tori_3d_unique_SA}.
Moreover, note that this explains why the middle-top and middle-bottom plots appear to be colored similarly.

Finally, note that in this example, the leading eigenvectors of $\mathbf{W}_1$ and $\mathbf{W}_2$ capture angles $\theta_3$ and $\theta_4$, respectively, and that angles $\theta_1$ and $\theta_2$ first appear only in eigenvectors number $6$ and $18$ of $\mathbf{W}_1$, respectively, and similarly for $\mathbf{W}_2$ but in reversed order.
This demonstrates that even in the presence of unique structures in each torus, the operators $\mathbf{S}$ and $\mathbf{F}$ successfully recovers the common components, and that these common components are significantly enhanced compared with the other unique components.

\end{appendix}

\end{document}